\documentclass{article} % For LaTeX2e
\usepackage{iclr2027_conference,times}

\usepackage{amsmath,amsfonts,bm}

\def\eqref#1{equation~\ref{#1}}
\def\1{\bm{1}}

\DeclareMathAlphabet{\mathsfit}{\encodingdefault}{\sfdefault}{m}{sl}
\SetMathAlphabet{\mathsfit}{bold}{\encodingdefault}{\sfdefault}{bx}{n}

\def\gG{{\mathcal{G}}}

\newcommand{\E}{\mathbb{E}}

\newcommand{\R}{\mathbb{R}}

\newcommand{\Var}{\mathrm{Var}}

\newcommand{\Cov}{\mathrm{Cov}}
\DeclareMathOperator*{\argmin}{arg\,min}

\DeclareMathOperator{\Tr}{Tr}

\newcommand{\iid}{\stackrel{\mathrm{iid}}{\sim}}

\newcommand{\norm}[1]{\left\lVert #1 \right\rVert}

\usepackage{amsmath,amssymb,amsthm,amsfonts}
\newtheorem{theorem}{Theorem}
\newtheorem{proposition}{Proposition}
\newtheorem{lemma}{Lemma}
\newtheorem{corollary}{Corollary}
\theoremstyle{definition}

\newtheorem{assumption}{Assumption}
\newtheorem{remark}{Remark}

\usepackage[skip=2pt]{caption}
\usepackage{float}
\usepackage{amssymb,amsthm,mathtools}
\usepackage{enumitem}
\usepackage{mathrsfs}
\usepackage{bbm}
\usepackage{hyperref}
\usepackage{url}
\usepackage{graphicx}
\usepackage{xcolor}
\usepackage{setspace}
\usepackage{array}
\usepackage{booktabs}
\usepackage{multirow}
\usepackage{algorithm}
\usepackage{algorithmic}
\usepackage{comment}
\usepackage{tablefootnote}
\usepackage[compact]{titlesec}
\setdisplayskipstretch{1}
\AtBeginDocument{
  \setlength{\abovedisplayskip}{4pt}
  \setlength{\belowdisplayskip}{4pt}
  \setlength{\abovedisplayshortskip}{0pt}
  \setlength{\belowdisplayshortskip}{0pt}
}

\title{Transfer Calibrated Prediction Powered Inference}

\author{Aditya T. Vadlamani~\textsuperscript{1}, Jae Ho Chang~\textsuperscript{2}\thanks{This work was done while the author was a student at The Ohio State University}~~, \\ \textbf{Srinivasan Parthasarathy~\textsuperscript{1}, Subhadeep Paul~\textsuperscript{1}}\\
\textsuperscript{1}The Ohio State University\quad\textsuperscript{2}Yale University
}

\iclrpreprintcopy
\begin{document}

\maketitle

\begin{abstract}
Prediction-powered inference (PPI) and its power-tuned extension (PPI++) improve confidence intervals by combining a small gold-standard labeled sample with a large AI model's predictions. Its efficiency gain relies on low residual variance, which may not hold if the predictor is pre-trained on a different source domain. We propose Transfer Calibrated Prediction-Powered Inference (TC-PPI), adapting the source-domain predictor to the target domain using gold-standard samples through cross-fitting. This approach supports various adaptation methods, such as sparse linear calibration, LoRA, and fine-tuning. Our jointly tuned cross-fit estimator, Joint-TC-Cross-PPI++, maintains unbiasedness and is simultaneously at least as efficient as classical inference, PPI, and PPI++, thereby protecting against negative transfer. We provide high-dimensional MSE bounds for calibration and show empirical improvements over baseline methods across various real-world applications.
\end{abstract}

\vspace{-0.1in}
\section{Introduction}
\vspace{-0.05in}

Reliable statistical inference is often difficult in domains with scarce amounts of ``gold-standard'' (i.e., labeled) data. In such settings, one may have a small set of labeled data $(X_i, Y_i)$ along with a much larger set of $N$ unlabeled data points, $\tilde X_j$, and a pretrained black-box AI model to produce predictions. With just the gold samples, we can perform valid inference, but this can lead to larger errors or more conservative intervals, due to the small sample size. Conversely, relying on the unlabeled data and the model's predictions with them is invalid due to potential bias or inaccurate predictions. A popular framework for performing reliable inference using both types of data is the Prediction-Powered Inference (PPI) framework~\citep{angelopoulos2023prediction,angelopoulos2023ppi++,zrnic2024cross,boyeau2024autoeval,cowen-breen2026multiplepredictionpowered,mani2026no,nath2025recidivism} and closely related semi-supervised inference literature~\citep{zhang2022high}. PPI uses the gold-standard data to rectify the model's predictions on the larger unlabeled dataset, enabling more efficient inference while maintaining validity. For mean estimation, the PPI estimator is
\[
\hat{\theta}^{\text{PP}} = \frac{1}{N}\sum_{i=1}^N f(\tilde{X}_i) + \frac{1}{n} \sum_i^{n}\{Y_i - f(X_i)\}.
\]
%Since $X_i$ and $\tilde{X}_i$ are assumed to be i.i.d. samples from the same distribution, the above estimator is clearly unbiased irrespective of the pre-trained AI model $f$. We note that the variance of the above estimator is the sum of the variances of the two components above, as their covariance is 0 due to independence. Then the
%variance reduction of PPI estimator is driven by the observation that  variance of the target-domain residuals $Y_i-f(X_i)$ is smaller than the variance of $Y_i$ itself, while the variance of the first part is small due to $N$ being large. Thus, the better the predictor on the \emph{target} domain and the smaller the residuals, the larger the benefits.
If $X_i$ and $\tilde{X}_i$ are i.i.d., $\hat{\theta}^{\text{PP}}$ is unbiased regardless of $f$, and has smaller variance than the classical estimator whenever $f$ predicts $Y$ reasonably well, since $N\gg n$ and $Y_i-f(X_i)$ has smaller variance than $Y_i$. This benefit can deteriorate in the presence of domain shifts.

%The PPI estimator reduces variance because the variance of target-domain residuals $Y_i - f(X_i)$ is smaller than that of $Y_i$ for a reasonable predictor $f$, and the first part's variance is small because $N$ is large. Thus, better predictors in the target domain lead to smaller residuals, which in turn lead to greater benefits.
%This observation also reveals a weakness of vanilla PPI in the presence of domain shift. In many scientific settings, the available predictor is not trained on the target population. Instead, one may have access to a source-domain predictor $f_s$ trained on related but different data. In most cases $f_s$ will be a pre-trained model such as a Large Language Model or a protein Language Model. The PPI estimator remains valid after rectifying with target gold data, but its efficiency can deteriorate sharply because the residual variance $\Var_t(Y-f_s(X))$ may be large if $f_s$ is a poor predictor in the target domain.
% This observation reveals a limitation of vanilla PPI in domain-shift scenarios. 
Often, the available predictor is a source-domain model $f_s$ trained on data distinct from the target population. While the PPI estimator remains valid when adjusted with target gold data, its efficiency can decline significantly if the residual variance $\Var_t(Y-f_s(X))$ is large.
%due to $f_s$ being a poor predictor for the target domain.
%To address this problem, we develop a formal framework called \emph{Transfer Calibrated Prediction-Powered Inference} (TC-PPI), which proposes a transfer-learning approach to calibrate the predictor using target-domain data. Transfer learning is an umbrella term for methods that adapt an ML or AI model trained in a different source domain or a broad corpus of data to a new target domain \citep{zhuang2020comprehensive,tripuraneni2020theory,day2017survey}. We show that a version which uses cross-fitting and joint power-tuning provably does not underperform the original PPI.
To redress, we introduce the \textit{Transfer Calibrated Prediction-Powered Inference} (TC-PPI) framework. TC-PPI employs transfer-learning~\citep{zhuang2020comprehensive,tripuraneni2020theory,day2017survey} to calibrate the source predictor using target-domain data. 

 \noindent {\bf Core Methodology.} We develop in stages, beginning with TC-PPI and its power-tuned version~\citep{angelopoulos2023ppi++}, TC-PPI++, which reduces PPI residual variance via transfer calibration. Next, we introduce a joint power-tuning step that keeps the source predictor and transfer correction separate, ensuring robustness against ineffective corrections. We prove a ``no-harm'' guarantee (Theorem~\ref{thm:joint-tc-cross}), showing that the Joint TC estimator with an optimal tuning parameter is at least as effective as classical methods. Finally, we integrate joint tuning with cross-fitting, utilizing all gold observations for rectification. This yields our Joint TC--Cross-PPI++ estimator, which significantly improves performance, as using the full gold sample is crucial for effective rectification.

\noindent {\bf Key Differentiators.} Unlike recent studies~\citep{bastani2021predicting,li2022transfer,chang2024heterogeneous}, which explore transfer learning in high-dimensional linear regression and require individual-level data from the source domain, our approach learns the sparse linear correction term using Lasso~\citep{tibshirani1996regression} {\it with features available only in the target domain}. While our methodology and theory allow any suitable transfer or fine-tuning procedure (e.g., last-layer fine-tuning, LoRA), sparse linear calibration lets us {\it obtain explicit high-dimensional convergence rates}.
Applying transfer learning or fine-tuning in PPI is challenging: using the same gold data for calibration and rectification leads to bias, and naive sample splitting reduces performance.  Our cross-fitting formulation (Section~\ref{sec:tcppi}), addresses these issues. Our joint power-tuning method (Section~\ref{sec:joint-tc-ppipp}), {\it helps mitigate negative transfer} \citep{Zhang_2023}. We prove that our estimators are {\it minimax optimal} (Section~\ref{sec:minimax}). Our framework effectively works across various estimation targets (e.g., mean, OLS, GLM, Bradley-Terry coefficients) and correction methods, {\it demonstrating its versatility across diverse domains} (Section~\ref{sec:empirical}).

The original PPI framework \citep{angelopoulos2023prediction} provides valid confidence intervals when a predictor is already available. Cross-prediction-powered inference \citep{zrnic2024cross} instead learns the predictor directly from the labeled sample via cross-fitting, but this requires enough target-domain data to train a good model from scratch. 
Instead, our proposal {\it adapts a pretrained model using the available gold data through a cross-fitting step}. PPI++ \citep{angelopoulos2023ppi++} addresses poor predictors with a power-tuning parameter, while \citep{mani2026no} notes that near-oracle conditions may not always hold (i.e., {\it no free lunch}). Our joint power-tuning approach aims for a {\it cheap lunch}, improving a predictor with the same gold data while ensuring no harm with high probability.
%Our proposal instead adapts a pretrained model through the same cross-fitting step, using only the gold data already available. PPI++ \citep{angelopoulos2023ppi++} separately addresses a poor predictor via a power-tuning parameter with a no-worse-than-classical guarantee; \citep{mani2026no} shows this is ``not a free lunch'' since a near-oracle tuning parameter requires conditions that may not hold. Our joint power-tuning formulation can be seen as an attempt at a ``cheap lunch'': improving an available predictor with the same gold data while still guaranteeing no harm with high probability.
%A separate line of work addresses predictor heterogeneity structurally rather than through calibration. StratPPI \citep{fisch2024stratified} partitions the population into strata of more homogeneous predictor quality (Appendix~\ref{app:additional_results}). Multi-Source PPI \citep{li2026multisource} aggregates pseudo-labeled sources by minimizing asymptotic confidence-region volume. PPI by Mixture of Experts \citep{gu2026mixture} learns a mixture over predictors with a best-expert guarantee. These exploit structure across strata, sources, or experts, whereas transfer calibration improves a single source predictor's target-domain alignment from the same gold-standard data.
Other approaches address predictor heterogeneity structurally. StratPPI partitions the population into strata of homogeneous predictor quality (Appendix~\ref{app:additional_results}). Multi-Source PPI \citep{li2026multisource} aggregates pseudo-labeled sources by minimizing asymptotic confidence-region volume, while PPI by Mixture of Experts~\citep{gu2026mixture} learns a mixture of predictors with a best-expert guarantee. They utilize structure across strata, sources, or experts, whereas {\it our transfer calibration improves a single-source predictor's target-domain alignment using the {\bf same} gold-standard data}.
%\section{PPI under domain shift}

\section{Transfer calibration and PPI/PPI++ rectification}\label{sec:tcppi}

\noindent {\bf PPI under Domain Shift.} Let $P_t$ denote the target-domain joint distribution of $(X,Y)$ and $P_s$ denote the source-domain joint distribution. In the target domain, we observe two datasets,
\begin{align*}
\mathcal D_t^{\mathrm{gold}} = \{(X_i,Y_i): i=1,\dots,n\} \iid P_t,\qquad 
\mathcal D_t^{\mathrm{pred}} = \{\widetilde X_j: j=1,\dots,N\} \iid P_{t,X},
\end{align*}
with $N\gg n$. We assume that a predictor $f_s:\mathcal X\to\mathcal Y$ has been learned from a vast source corpus: $\mathcal D_s = \{(X_i^s,Y_i^s): i=1,\dots,n_s\} \iid P_s$. 
We do not observe individual-level data from this corpus. 
We allow this predictor to be misspecified for the response in the target domain $P_t$. The inferential target is target-domain parameter $\theta_t^*$. We use the mean, $\E_t[Y] \coloneqq \int ydP_t(x,y)$ to develop key ideas, with additional results on linear regression, generalized linear models, and low-dimensional M-estimation in the Appendix.
%We treat the source predictor $f_s$ as a black-box machine-learning model, such as a tree-based model, deep neural network, or a large language model. In what follows, we condition on the source data $\mathcal D_s$ and consequently treat $f_s$ as fixed.
We treat the source predictor $f_s$ as a black-box machine-learning model (e.g., tree-based, deep neural network, or large language model). We condition on the source data $\mathcal D_s$ and consequently treat $f_s$ as fixed.
%We assume the source domain sample size is so large that the difference between an estimated $\hat{f}_s$ and the true $f_s$ is negligible, and we do not differentiate between them. 
Let $
f_t(x)\coloneqq \E_t[Y\mid X=x]$
denote the target regression function and define the oracle target-domain correction
\begin{equation}\label{eq:general-delta}
\Delta^*(x)\coloneqq  f_t(x)-f_s(x).
\end{equation}
Thus the target mean function can  be written as $
f_t(x)=f_s(x)+\Delta^*(x).$
The role of transfer learning in our framework is to estimate $\Delta^*$, or, more generally, to construct a useful target-adapted prediction function, from limited gold-standard target data. As stated before, we consider the domain difference $\Delta^*$ to be a simple function or a low-dimensional function that scientists and domain researchers can easily train with limited data and limited resources.  A transfer-calibration algorithm $\mathcal A_{\rm cal}$ applied to a subset of the gold data produces a learned correction $\widehat\Delta$ and hence
%\begin{equation}\label{eq:fhat}
$\widehat f_t(x)=f_s(x)+\widehat\Delta(x).$
%\end{equation}
The calibration algorithm may be sparse linear regression (lasso) on the residuals $Y_i-f_s(X_i)$ either using the raw features or using embedding from an open-weight AI model, parameter-efficient fine-tuning methods on deep neural networks such as adapter \citep{houlsby2019parameter} or LoRA \citep{hu2022lora}, or another low-complexity transfer procedure. If a fine-tuning procedure directly produces a predictor $\hat{f}_t$, we define $\widehat\Delta=\widehat f_t-f_s$ after fitting. Thus, the additive notation does not restrict the architecture or parameterization of the transfer calibration. Define the target noise
$\varepsilon=Y-f_t(X).$
Then, by definition of $f_t(X)$, we have $
\E_t[\varepsilon\mid X]=0,$ and 
$\E_t[\varepsilon]=0.$
While some of our theoretical results rely on a sparse linear correction $\Delta^*$, the validity results below require no structural assumptions on $\Delta^*$, such as linearity or sparsity. Such assumptions enter only when we seek explicit rates for learning the transfer correction.
%\section{Transfer calibration and PPI/PPI++ rectification}\label{sec:tcppi}
%\subsection{Exact unbiasedness and conditional variance}
%\subsection{Cross-fitting} \label{sec:tc-cross-ppi}
%However, since $n_1 + n_2=n$, the rectification sample in TC-PPI is smaller, and  hence the asymptotic variance of vanilla PPI is $V_{\mathrm{PPI}}/n$, while that of TC-PPI is $V_{\mathrm{TC}}/n_2$. This may appear to be a tradeoff that requires careful allocation of gold samples to learning transfer correction and rectification, but we now show that cross-fitting solves this problem entirely. The cross-fitting enables us to use the full gold sample for rectification and almost the full sample for learning correction.
%A simple split-sample construction where transfer calibration is trained on a part of the data, ${\cal D}_{t,1}^{\rm gold}$ of size $n_1$, and PPI rectification is obtained on the other part of the sample ${\cal D}_{t,1}^{\rm gold}$  of size $n_2$ is given in the Appendix. Proposition \ref{prop:unbiased} shows the estimator is conditionally unbiased and gives an expression for the conditional variance $\Var\bigl(\widehat\theta_{\mathrm{TC}}\mid \mathcal D_s,{\cal D}_{t,1}^{\rm gold}\bigr)$, while Theorem \ref{thm:clt} provided asymptotic limiting distribution. While this construction has the advantage that the conditional variance of the estimator cleanly decomposes due to independence, it wastes gold samples from the rectification stage, which ultimately leads to underperformance.

\noindent {\bf Calibration and Rectification.} A naive construction, where transfer calibration is trained on one part of the data and PPI rectification is obtained on the other, is given in the Appendix, along with proofs that such an estimator is conditionally unbiased and its asymptotic limiting distribution, along with a derivation of the conditional variance.  While the conditional variance cleanly decomposes in this construction (independent), it wastes gold samples during rectification, reducing performance. We use cross-fitting \citep{zrnic2024cross,mani2026no} to recover all labels for rectification. 

We partition the labeled gold sample into $K$ equal folds $I_1,\ldots,I_K$, with fixed $K$ and $|I_k|=n/K$. For each fold $k$, we fit a transfer correction using combined gold data from the other folds,
\[
 \widehat\Delta^{(-k)}
 =\mathcal A_{\rm cal}\!\left(
 f_s,\{(X_i,Y_i):i\notin I_k\}\right),
 \qquad
 \widehat f_t^{(-k)}(x)=f_s(x)+\widehat\Delta^{(-k)}(x).
\]
For $i\in I_k$, we define the out-of-fold prediction as
$\widehat f_i^{\rm oof}=\widehat f_t^{(-k)}(X_i)$ and the averaged unlabeled prediction as $
 \overline f_t(\widetilde X_j)
 =\frac1K\sum_{k=1}^K
 \widehat f_t^{(-k)}(\widetilde X_j).$
\begin{comment}
For mean estimation, the transfer-calibrated cross-prediction estimator is
\begin{equation}\label{eq:tc-cross-mean}
 \widehat\mu_{\rm TC\text{-}Cross}
 =\frac1N\sum_{j=1}^N\overline f_t(\widetilde X_j)
 +\frac1n\sum_{i=1}^n
 \{Y_i-\widehat f_i^{\rm oof}\}.
\end{equation}
\end{comment}
The new cross-fitted estimator of mean is
\begin{equation}\label{eq:tc-cross-fold-decomp}
\widehat\mu_{\rm TC\text{-}Cross}
=
\frac1K\sum_{k=1}^K
\left[
\frac1N\sum_{j=1}^N
\widehat f_t^{(-k)}(\widetilde X_j)
+
\frac K n\sum_{i\in I_k}
\{Y_i-\widehat f_t^{(-k)}(X_i)\}
\right].
\end{equation}
%Clearly, the averaged unlabeled predictor $\overline f_t$ and the out-of-fold predictions are not independent since the observation $i$ has been used to train one of the predictors in the average. However, the above expression shows that the estimator is an average of fold-specific PPI estimators in which the same predictor $\widehat f_t^{(-k)}$ is used for the unlabeled prediction term and for rectification on its corresponding held-out fold. Because observation $i\in I_k$ is excluded from the data used to fit $\widehat f_t^{(-k)}$, its rectification term remains out of sample. 

The averaged unlabeled predictor $\overline f_t$ and the out-of-fold predictions are not independent, as observation $i$ was used to train one of the predictors in the average. However, the estimator averages fold-specific PPI estimators, where the same predictor $\widehat f_t^{(-k)}$ is used for both the unlabeled prediction and rectification on its held-out fold, and each held-out observation is excluded from fitting the predictor used in its own rectification term. 

\begin{proposition}[Unbiasedness]
\label{prop:tc-cross-unbiased}
Suppose the folds are fixed independently of the observations and have equal
size $|I_k|=n/K$. Under Assumption~\ref{ass:model},
$
\E\!\left[
\widehat\mu_{\rm TC\text{-}Cross}
\mid\mathcal D_s
\right]
=
\theta_t^*.$
\end{proposition}
%Unlike the split-sample TC-PPI estimator, the cross-fitted estimator does not admit a simple exact conditional variance decomposition after conditioning on a single calibration sample. The fold-specific fitted predictors have overlapping training samples and share the same unlabeled sample, so cross-fold covariance terms are generally nonzero. Theorem~\ref{thm:tc-cross-mean} handles this dependence through an asymptotic linear representation. All terms generated by estimation of the fold-specific calibration functions are collected in a remainder that is $o_p(n^{-1/2})$ under prediction consistency.
The cross-fitted estimator differs from the split-sample TC-PPI estimator as it lacks a straightforward conditional variance decomposition after using a single calibration sample. The overlap in training samples and shared unlabeled samples among fold-specific fitted predictors results in nonzero cross-fold covariance terms. Theorem~\ref{thm:tc-cross-mean} addresses this dependence with an asymptotic linear representation, and the estimation errors from fold-specific calibration functions are collected in a remainder that is $o_p(n^{-1/2})$, if the transfer calibration is consistent. Define $\mathcal D_{-k}\coloneqq \{(X_i,Y_i):i\notin I_k\}$ and
\[
 e_k(x)\coloneqq \widehat f_t^{(-k)}(x)-f_t(x),
 \qquad
 a_n\coloneqq \max_{1\le k\le K}\E_t[e_k(X)^2\mid \mathcal D_s,\mathcal D_{-k}].
\]
Then $ e_k(x)$ is the error that the predictor trained on all folds except for $k$ fold makes in approximating $f_t(x)$. The quantity $a_n$ is the maximum expected squared error over all folds.

\begin{theorem}[TC--Cross-PPI for the target mean]
\label{thm:tc-cross-mean}
Suppose $K$ is fixed, the folds have equal size $|I_k|=n/K$, $n/N\to\rho\in[0,\infty)$, and, conditionally on the source data, $Y, f_t(X), \hat{f}_t(X)$ have finite $(2+\eta)$ moments for some $\eta>0$. Assume $a_n=o_p(1)$. Then
\[
\widehat\mu_{\rm TC\text{-}Cross}-\theta_t^*
=
\frac1n\sum_{i=1}^n\varepsilon_i
+
\frac1N\sum_{j=1}^N
\{f_t(\widetilde X_j)-\E_t f_t(X)\}
+
o_p(n^{-1/2}),
\]
and consequently, $
 \sqrt n\bigl(\widehat\mu_{\rm TC\text{-}Cross}-\theta_t^*\bigr)
 \overset{D}{\longrightarrow}
 N\!\left(0,\Var_t(\varepsilon)+\rho\Var_t\{f_t(X)\}\right).$
%More generally, if for a deterministic sequence $r_n\to0$, it holds that $ \max_{1\le k\le K}\E_t[e_k(X)^2\mid\mathcal D_s,\mathcal D_{-k}]=O_p(r_n),$ then $\sqrt n\,R_n=O_p(\sqrt{r_n}).$
\end{theorem}

\noindent{\bf Sparse linear calibration: high-dimensional lasso rates.} The above theorem is agnostic to the form of the transfer correction and instead assumes the conditional mean squared error of the correction is $o_p(1)$. To obtain an explicit rate, we now specialize to a sparse linear correction as follows,
\begin{equation}\label{eq:sparse-linear-calibration-main}
\Delta^*(x)=x^\top\delta^*,\quad\delta^*\in\R^p, \quad \|\delta^*\|_0\le s,
\end{equation}
for some $s>0$. 
We assume the features are high-dimensional, i.e., $p \gg n$. 
%On the calibration split, define the pseudo-response $U_i=Y_i-f_s(X_i),\quad i\in\mathcal I_1.$ Then $\E_t[U_i\mid X_i]=X_i^\top\delta^*.$
We estimate $\delta^*$ by the lasso on the calibration data for each fold as,
\[
\widehat\delta^{(-k)}\in\arg\min_{\delta\in\R^p}
\left\{
\frac{1}{2n_{-k}}\sum_{i\notin I_k}
\bigl(Y_i-f_s(X_i)-X_i^\top\delta\bigr)^2
+\lambda\|\delta\|_1
\right\},
\qquad n_{-k}=\frac{K-1}{K}n.
\]

Corollary~\ref{cor:tc-cross-lasso} specializes Theorem \ref{thm:tc-cross-mean} to the case of sparse linear calibration estimated through lasso.

\begin{corollary}[Sparse calibration rate for TC--Cross-PPI]
\label{cor:tc-cross-lasso}
Consider the sparse linear calibration model and assume the lasso related conditions in Appendix~\ref{app:theory} hold on every training complement.
Then, uniformly over all folds (fixed), 
$
 (\widehat\delta^{(-k)}-\delta^*)^\top\Sigma_t
 (\widehat\delta^{(-k)}-\delta^*)
 =
 O_p\!\left(
 \frac{s\log p}{(K-1)n/K}
 \right).$
Consequently, if $
\frac{s\log p}{(K-1)n/K}\to0,$,
then Theorem~\ref{thm:tc-cross-mean} holds. For fixed $K$, it is 
equivalent to $s\log p/n\to0$.
\end{corollary}

\begin{remark}
    [Comparison with vanilla PPI]\label{cor:compare}
Using all $n$ gold observations, the vanilla PPI estimator based on the unadapted source predictor $f_s$ has $\sqrt n$-scale asymptotic variance of $
V_{\mathrm{PPI}}
=
\rho\Var_t\{f_s(X)\}
+\Var_t(\varepsilon)
+\Var_t\{\Delta^*(X)\}.$
Under Theorem~\ref{thm:tc-cross-mean}, TC--Cross-PPI also uses all $n$ gold observations and has $\sqrt n$-scale asymptotic variance
$
V_{\mathrm{TC-Cross}}
=
\rho\Var_t\{f_t(X)\}+\Var_t(\varepsilon).$
Thus
\[
V_{\mathrm{PPI}}-V_{\mathrm{TC-Cross}}
=
\Var_t\{\Delta^*(X)\}
+\rho\left[\Var_t\{f_s(X)\}-\Var_t\{f_t(X)\}\right].
\]
In the regime where $N\gg n$, we have $\rho=0$, and consistent transfer calibration removes the source-to-target drift variance $\Var_t\{\Delta^*(X)\}$, making TC-Cross-PPI more efficient than vanilla PPI. %Under sparse linear calibration and mean-zero $X$, $\Var_t\{\Delta^*(X)\}=(\delta^*)^\top\Sigma_t\delta^*$, while Corollary~\ref{cor:tc-cross-lasso} gives calibration error of order $O_p(s\log p/n)$ for the cross-fitted procedure.
\end{remark}

\noindent{\bf Transfer Calibrated Extension for PPI++.}
%Next we extend the Transfer Calibrated estimator to the PPI++ approach of \citep{angelopoulos2023ppi++}. \citeauthor{angelopoulos2023ppi++} consider the possibility that the AI model might be a poor predictor of the response in the gold sample, which can cause PPI to perform even worse than classical inference. They show that one can improve the inference by interpolating between classical and prediction-powered inference with a scalar tuning parameter $\lambda$. The same idea is natural here because even the calibrated predictor may be imperfect.
\citep{angelopoulos2023ppi++} note that if the AI model poorly predicts responses in the gold sample, PPI may underperform compared to classical inference. They show that inference can be improved by interpolating between classical and prediction-powered approaches using a tuning parameter $\lambda$. This concept applies here as even a calibrated predictor may be imperfect. We define the $\lambda$-weighted transfer-calibrated estimator as:
\begin{equation}\label{eq:tcppi-lambda}
 \widehat\mu_{{\rm TC\text{-}Cross},\lambda}
 =\bar Y_n+\lambda\left\{
 \frac1N\sum_{j=1}^N\overline f_t(\widetilde X_j)
 -\frac1n\sum_{k=1}^K\sum_{i\in I_k}\widehat f_t^{(-k)}(X_i)
 \right\}.
\end{equation}
This estimator equals classical inference at $\lambda=0$ and ordinary TC--Cross-PPI at $\lambda=1$.

\begin{theorem}[TC--Cross-PPI++]
\label{thm:tc-cross-ppipp}
Suppose the fold and moment conditions of Theorem \ref{thm:tc-cross-mean} hold. Assume there exists a
square-integrable function $f_\dagger$ such that, uniformly over the fixed
number of folds,
$
\E_t\!\left[
\{\widehat f_t^{(-k)}(X)-f_\dagger(X)\}^2
\mid \mathcal D_s,\mathcal D_{-k}
\right]
=o_p(1).$
Then, for every fixed $\lambda$, we have 
\[
\sqrt n\bigl(
\widehat\mu_{{\rm TC\text{-}Cross},\lambda}
-\theta_t^*
\bigr)
\overset{D}{\longrightarrow}
N\!\left(0,\Var_t(Y)
-
2\lambda\Cov_t\{Y,f_\dagger(X)\}
+
(1+\rho)\lambda^2\Var_t\{f_\dagger(X)\}\right).
\]
If $\Var_t\{f_\dagger(X)\}>0$, the oracle limiting minimizer is
\begin{equation}\label{eq:tc-cross-lambda-general}
\lambda_\dagger^*
=
\frac{\Cov_t\{Y,f_\dagger(X)\}}
{(1+\rho)\Var_t\{f_\dagger(X)\}},
\end{equation}
and the corresponding minimized variance is $
\Var_t(Y)
-
\frac{
\Cov_t\{Y,f_\dagger(X)\}^2
}{
(1+\rho)\Var_t\{f_\dagger(X)\}
}.$
If in addition, the transfer-calibrated predictor is consistent for the oracle target
predictor, $
f_\dagger(X)=f_t(X),$ and
then $
\lambda_t^*
=
\frac{1}{1+\rho},$
with the minimized variance being $
\Var_t(Y)
-
\frac{\Var_t\{f_t(X)\}}{1+\rho}.$
If a feasible estimator $\widehat\lambda$ satisfies
$\widehat\lambda-\lambda_\dagger^*=o_p(1)$, using
$\widehat\lambda$ in the estimator gives the same asymptotic limiting distribution.
\end{theorem}

In the above theorem, we let the transfer-corrected predictors from the folds converge to a function $f_{\dagger}(X)$, instead of the true target domain predictor function $f_t(X)$. This removes the assumption that, post-transfer learning we are able to correctly approximate the true target domain conditional expectation of $Y$ given $X$.  The oracle parameter $\lambda^*_{\dagger}$ in \eqref{eq:tc-cross-lambda-general} depends on unknown
target-domain moments. 
\begin{comment}
Let
$
\overline f_{\rm oof}
=
\frac1n\sum_{i=1}^n\widehat f_i^{\rm oof},$ and
$
\overline Y_n
=
\frac1n\sum_{i=1}^nY_i,$
and define
\[
\widehat{\Cov}_{\rm oof}(Y,\widehat f)
=
\frac1{n-1}\sum_{i=1}^n
(Y_i-\overline Y_n)
(\widehat f_i^{\rm oof}-\overline f_{\rm oof}),
\qquad
\widehat{\Var}_{\rm oof}(\widehat f)
=
\frac1{n-1}\sum_{i=1}^n
(\widehat f_i^{\rm oof}-\overline f_{\rm oof})^2.
\]
\end{comment}
A plug-in estimator for $\lambda^*_{\dagger}$ is
%\begin{equation}\label{eq:tc-cross-lambda-hat}
$\widehat\lambda
=
\frac{
\widehat{\Cov}_{\rm oof}(Y,\widehat f)
}{
(1+n/N)
\widehat{\Var}_{\rm oof}(\widehat f)
},$
%\end{equation}
when the denominator is positive, with $\widehat\lambda=0$ otherwise.
One may also clip $\widehat\lambda$ to $[0,1]$ for stability. Under the foldwise prediction-consistency and moment conditions of
Theorem~\ref{thm:tc-cross-ppipp},
$
\widehat\lambda-\lambda_\dagger^*=o_p(1).$
\vspace{-0.05in}
\section{Joint power tuning: negative transfer protection}
\label{sec:joint-tc-ppipp}
\vspace{-0.05in}

The TC--Cross-PPI++ estimator in \eqref{eq:tcppi-lambda} controls the influence of the transfer-calibrated predictor, ensuring its variance does not exceed that of the classical-only variance. However, as target samples are limited, transfer calibration can be weak or unstable. %Therefore, we keep the source predictor separate for inference instead of fully replacing it with the calibrated predictor.
%The TC--Cross-PPI++ estimator in \eqref{eq:tcppi-lambda} controls how strongly we use the transfer-calibrated predictor and guarantees that the estimator variance is no larger than the classical-only variance. However, even after cross-fitting, transfer calibration to the source predictor is estimated from a limited target sample and can be weak or unstable. 
We thus also retain the original source predictor instead of replacing it with the calibrated predictor. %We therefore introduce separate power parameters for the source prediction and the learned transfer correction and optimize them jointly. This construction makes ordinary PPI++ and TC--Cross-PPI++ nested subfamilies of a single estimator.
For $i\in I_k$, define $
\widehat\Delta_i^{\rm oof}
\coloneqq 
\widehat f_t^{(-k)}(X_i)-f_s(X_i)$ and $
\widehat H_i^{\rm oof}
\coloneqq 
\begin{pmatrix}
 f_s(X_i)\\
 \widehat\Delta_i^{\rm oof}
\end{pmatrix},$
and the averaged unlabeled vector
\[
\widetilde H_N^{\rm CF}
\coloneqq 
\begin{pmatrix}
\displaystyle
\frac1N\sum_{j=1}^N f_s(\widetilde X_j) & \qquad 
\displaystyle
\frac1{KN}\sum_{k=1}^K\sum_{j=1}^N
\left\{
\widehat f_t^{(-k)}(\widetilde X_j)
-
f_s(\widetilde X_j)
\right\}
\end{pmatrix}^T.
\]
For
$\boldsymbol\lambda=(\lambda_s,\lambda_\Delta)^\top\in\mathbb R^2$,
define the Joint TC-Cross-PPI++ estimator as
\begin{equation}\label{eq:joint-tc-cross}
\widehat\mu_{\mathrm{JTC\text{-}Cross},\boldsymbol\lambda}
=
\overline Y_n
+
\boldsymbol\lambda^\top
\left\{
\widetilde H_N^{\rm CF}
-
\frac1n\sum_{i=1}^n\widehat H_i^{\rm oof}
\right\}.
\end{equation}
%The important special cases are,
\begin{align*}
\textbf{\underline{Special~Cases:} }
(0,0)^\top: \text{classical inference}, \quad
(1,0)^\top&: \text{Source PPI using }f_s, \quad
(1,1)^\top: \text{TC--Cross-PPI},\\
(\lambda,0)^\top: \text{Source PPI++ using }f_s, & \quad
(\lambda,\lambda)^\top: \text{TC--Cross-PPI++}.
\end{align*}

%Thus optimizing over the vector $\boldsymbol\lambda$ gives the procedure the option to ignore the transfer correction, use it fully, or combine it with the original source predictor in a way that neither scalar PPI++ nor scalar TC--Cross-PPI++ can represent. Proposition \ref{prop:joint-cross-unbiased} in the Appendix shows the estimator is unbiased for every fixed $\boldsymbol \lambda$.
Optimizing over the vector $\boldsymbol\lambda$ allows one to ignore the transfer correction, use it fully, or combine it with the original source predictor in ways that neither vanilla PPI++ nor TC--Cross-PPI++ can represent. Proposition~ \ref{prop:joint-cross-unbiased} (see Appendix) shows the estimator is unbiased for any fixed $\boldsymbol \lambda$.

\begin{comment}

The unbiasedness above does not require the fold-specific fitted predictors to be independent. For each fold, the same fold-specific predictor is applied to the held-out gold observations and to the unlabeled observations, and its two population means cancel after conditioning on the corresponding training complement. The different fold-specific terms are generally dependent because their training samples overlap and because they share the same unlabeled sample. This dependence is handled in the asymptotic theory below through the cross-fitting remainder, as before.
\end{comment}

\begin{theorem}[Joint TC--Cross-PPI++ and asymptotic no-harm]
\label{thm:joint-tc-cross}
Suppose the fold and moment conditions of Theorem \ref{thm:tc-cross-mean} hold. Assume there exists a
square-integrable function $\Delta^\dagger$ such that $
\max_{1\le k\le K}
\E_t\!\left[
\{\widehat\Delta^{(-k)}(X)-\Delta^\dagger(X)\}^2
\mid \mathcal D_s,\mathcal D_{-k}
\right]
=O_p(r_n),$ with $r_n=o(1)$.
Define $
H_\dagger(x)
\coloneqq 
\begin{pmatrix}
 f_s(x) & \, 
 \Delta^\dagger(x)
\end{pmatrix}^\top,$ and $
f_\dagger(x)
\coloneqq 
f_s(x)+\Delta^\dagger(x).$
Then for every fixed
$\boldsymbol\lambda\in\mathbb R^2$,
\[
\sqrt n\left(
\widehat\mu_{\mathrm{JTC\text{-}Cross},\boldsymbol\lambda}
-\theta_t^*
\right)
\overset{D}{\longrightarrow}
N\!\left(
0,
\Var_t\{Y-\boldsymbol\lambda^\top H_\dagger(X)\}
+
\rho\Var_t\{\boldsymbol\lambda^\top H_\dagger(X)\}
\right).
\]
Further, let, $
\Sigma_\dagger
\coloneqq 
\Var_t\{H_\dagger(X)\},$ and $
c_\dagger
\coloneqq 
\Cov_t\{H_\dagger(X),Y\}.$
If $\Sigma_\dagger$ is nonsingular, the unique oracle limiting weight is
%\begin{equation}\label{eq:joint-cross-lambda}
$\boldsymbol\lambda_{\rm Jcross}^*
=
\frac1{1+\rho}
\Sigma_\dagger^{-1}c_\dagger.$
%\end{equation}
If $\Sigma_\dagger$ is singular, a minimum-norm oracle solution is obtained by
replacing $\Sigma_\dagger^{-1}$ with its Moore--Penrose pseudoinverse
$\Sigma_\dagger^\dagger$. Now, let $
V_{\rm Jcross}^{\dagger,*}
\coloneqq 
\inf_{\boldsymbol\lambda\in\mathbb R^2}
V_{\rm Jcross}^\dagger(\boldsymbol\lambda).$
Then the minimized variance is given by,
%\begin{equation}\label{eq:joint-cross-min-var}
$V_{\rm Jcross}^{\dagger,*}
=
\Var_t(Y)
-
\frac{
c_\dagger^\top
\Sigma_\dagger^\dagger
c_\dagger
}{
1+\rho
}.$
%\end{equation}
Denote the optimized source PPI++ variance by $
V_{s,++}^{\rm cross}$
and the optimized scalar TC--Cross-PPI++ variance associated with the
limiting transfer predictor $f_\dagger$ by
$
V_{\dagger,++}^{\rm cross}$.
Because the joint class contains source-PPI++,
TC--Cross-PPI++,
and classical,
\begin{equation}\label{eq:joint-cross-dominance}
V_{\rm Jcross}^{\dagger,*}
\le
\min\left\{
V_{s,++}^{\rm cross},
V_{\dagger,++}^{\rm cross},
\Var_t(Y)
\right\}.
\end{equation}
\end{theorem}
The asymptotic no-harm property relative to source PPI++ does not require the transfer learner to be correctly specified. It suffices for fold-specific estimators to converge to a stable, square-integrable limiting function $\Delta^\dagger$. 
%Thus the asymptotic no-harm property relative to source PPI++ does not require the transfer learner to be correctly specified. It is enough that the fold-specific transfer estimators converge to a stable square-integrable limiting function $\Delta^\dagger$. If the transfer correction is unhelpful, the joint optimizer can asymptotically assign it no weight.
If in addition, the transfer learner is consistent for the oracle correction,
so that $
\Delta^\dagger=\Delta^*,$
then
%\begin{equation}\label{eq:joint-cross-oracle-equality}
$V_{\rm Jcross}^{\dagger,*}
=
\Var_t(Y)
-
\frac{\Var_t\{f_t(X)\}}{1+\rho}.
$
%This equals the oracle TC--Cross-PPI++ variance in Theorem~\ref{thm:tc-cross-ppipp}. Hence convergence to a stable transfer function is sufficient for the no-harm safeguard. However, if we want our estimator to be strictly more efficient than classical inference, then we need consistency for the oracle correction.
This equals the oracle TC--Cross-PPI++ variance in Theorem~\ref{thm:tc-cross-ppipp}, and is smaller than the variance of classical inference. While convergence to a stable transfer function ensures the no-harm safeguard, consistency for the oracle correction guarantees greater efficiency than classical inference. 
Ordinary source PPI++ has limiting variance
$
\Var_t(Y)
-
\frac{
\Cov_t\{Y,f_s(X)\}^2
}{
(1+\rho)\Var_t\{f_s(X)\}
}.
$
Since $f_t(X)=\E_t[Y\mid X]$, we have, $
\Cov_t\{Y,f_s(X)\}
=
\Cov_t\{f_t(X),f_s(X)\},$
and the covariance-variance inequality gives,
$
\frac{
\Cov_t\{Y,f_s(X)\}^2
}{
\Var_t\{f_s(X)\}
}
\le
\Var_t\{f_t(X)\}.$
Therefore, TC--Cross-PPI++, with oracle power tuning parameters and consistent transfer correction, is guaranteed to be no less efficient than source
PPI++. Theorem~\ref{thm:joint-tc-cross} strengthens this comparison by retaining source PPI++ as a nested subfamily even when the transfer learner converges to a misspecified direction.

If a feasible weight estimator satisfies $
\widehat{\boldsymbol\lambda}
-
\boldsymbol\lambda_{\rm Jcross}^*
=o_p(1),$
substitution of $\widehat{\boldsymbol\lambda}$ provides the same asymptotic limiting distribution. We can estimate $\boldsymbol\lambda$ using the plug-in estimator (with a small ridge regularization when needed), $\widehat{\boldsymbol\lambda}
=
\frac1{1+n/N}
(\widehat\Sigma_H+\tau I_2)^{-1}
\widehat c_H,$ with $
\tau\ge0
$. Estimating the weights $\boldsymbol\lambda$ using the above estimator incurs an additional inflation of $O(n^{-1}+r_n + \tau^2)$ to the $\sqrt{n}$ scale variance (Proposition \ref{prop:joint-weight-rate}, Appendix \ref{app:theory}).
The empirical estimators for the asymptotic variances in the above theorems, along with results on their consistency and coverage rates of the corresponding Wald intervals, are provided in Appendix \ref{app:variance-estimation}. Unlike \cite{zrnic2024cross}, we avoid expensive bootstrap to estimate standard errors and obtain confidence intervals. Instead, we show that despite using cross-fitting, we have the correct asymptotic coverage rate in Lemma \ref{lem:mean-var-consistency}.
%The joint formulation is valuable as its safeguard is independent of the transfer model's form or accuracy, including Lasso, last-layer fine-tuning, adapters, LoRA, or a misspecified procedure. 
%If transfer calibration is ineffective, the optimizer can shrink the correction, returning to standard PPI++; otherwise, it leverages additional target-domain information.
%When the source predictor and transfer correction provide complementary insights, the joint estimator can outperform either scalar PPI++ estimators.
%The cross-fitted formulation makes the comparison with ordinary PPI++ at the same gold-label budget.
\vspace{-2mm}
\section{Nonasymptotic rates and Minimax Optimality}
\vspace{-0.1in}
\label{sec:minimax}
In this section, we specialize to the sparse linear calibration model and study the finite-sample cost of estimating the transfer correction. 
For balanced $K$-fold cross-fitting, each correction is estimated on $
n_{\rm tr}\coloneqq \frac{K-1}{K}n$ observations, while all $n$ gold observations contribute to rectification.
Define $
r_n
\coloneqq 
\max_{1\le k\le K}
\E\!\left[
\E_t\{e_k(X)^2\mid\mathcal D_s,\mathcal D_{-k}\}
\mid\mathcal D_s
\right].$
Under sparse linear calibration with lasso, $
r_n
\lesssim
\frac{\sigma_\epsilon^2s\log p}{n_{\rm tr}}$ (Lemma \ref{lem:exp-lasso}, Appendix),
up to constants determined by the design and noise conditions (Appendix~\ref{app:theory}). Now, we can state the non-asymptotic MSE bound theorem as follows.

\begin{theorem}[Nonasymptotic MSE bound]
\label{thm:tc-cross-nonasymptotic}
For fixed, balanced folds, and finite second moments, the TC-Cross PPI satisfies
\begin{equation}\label{eq:tc-cross-mse-upper}
\operatorname{MSE}\!\left(
\widehat\mu_{\rm TC\text{-}Cross}
\mid\mathcal D_s
\right)
\le
\frac{\sigma_\epsilon^2}{n}
+
\frac{2V_t}{N}
+
\left(
\frac2N+\frac K n
\right)r_n.
\end{equation}
Further consider the TC--Cross-PPI++ estimator with $
\lambda_0
\coloneqq 
\frac{1}{1+n/N}$ as the tuning parameter. Then
\begin{align}
\operatorname{MSE}\!\left(
\widehat\mu_{{\rm TC\text{-}Cross},\lambda_0}
\mid\mathcal D_s
\right)
\leq
\frac{\sigma_\epsilon^2}{n}
+
\frac{V_t}{n+N}
 + K
\frac{N}{n(n+N)}\,r_n.
\label{eq:tc-cross-ppipp-mse-upper}
\end{align}
\end{theorem}

The difference between the two bounds clarifies the role of power tuning. Power tuning is not only a safeguard against a poor predictor. When the target predictor is consistently learned, the finite-sample choice $\lambda_0=N/(n+N)$ changes the oracle signal term from $V_t/N$ to $V_t/(n+N)$ and cancels the leading interaction between the oracle prediction fluctuation and the calibration error. Consequently, the $O(\sqrt{r_n}/N)$ term appearing
in \ref{eq:tc-cross-mse-upper} disappears and the calibration cost in
\ref{eq:tc-cross-ppipp-mse-upper} is proportional to $r_n$. We next give a matching information-theoretic lower bound. Consider the
Gaussian sparse-transfer subclass
${\cal G}_s(V_s,V_\delta,\sigma_\epsilon^2)$ defined in Appendix~\ref{app:minimax},
and let $
\psi_{n,p,s}(V_\delta,\sigma_\epsilon^2)
\coloneqq 
\min\left\{
V_\delta,\,
\frac{\sigma_\epsilon^2(s-1)
\log\!\{e(p-1)/(s-1)\}}
{n-1}
\right\}.$
\begin{theorem}[Minimax lower bound with an unknown sparse transfer direction]
\label{thm:minimax-main}
Assume $n\ge3$, $N\ge1$, and $2\le s\le p/2$. There exists a universal constant $c>0$ such that
\begin{align}
&\inf_{\widehat\theta\in\mathfrak T_{n,N}}
\sup_{F\in\ {\cal G}_s(V_s,V_\delta,\sigma_\epsilon^2)}
\E_F\bigl[(\widehat\theta-\theta)^2\bigr]
~\ge~
\frac{\sigma_\epsilon^2}{n}+\frac{V_t}{n+N}
+
c\,\frac{N}{n(n+N)}
\psi_{n,p,s}(V_\delta,\sigma_\epsilon^2)
.
\label{eq:minimax-main}
\end{align}
\end{theorem}

\begin{comment}
The first term $R_{\rm or}$ is the prediction-powered floor if the centered
target prediction direction were known. The second term $R_{\rm cal}$ is the additional
price of learning the unknown sparse transfer direction. In the estimable sparse regime
$\frac{\sigma_\epsilon^2 s\log(ep/s)}{n}
\lesssim V_\delta$,
Theorem~\ref{thm:minimax-main} gives, up to constants,
$
R_{\rm cal}
\gtrsim
\frac{\sigma_\epsilon^2 N}{n(n+N)}
\frac{s\log(ep/s)}{n}.
$
\end{comment}

\begin{remark}[Minimax optimality]
\label{cor:tc-cross-ppipp-minimax}
%Fix $K$ and assume $r_n\lesssim \sigma_\epsilon^2 s\log p/n_{\rm tr}$.
In the estimable regime above and on the Gaussian class, both the upper bound \ref{eq:tc-cross-ppipp-mse-upper} and the lower bound \ref{eq:minimax-main} have the oracle rate $\sigma_\epsilon^2/n+V_t/(n+N)$ and a calibration cost of order $
\frac{N}{n(n+N)}\cdot\frac{s}{n}.$
The two bounds differ only by a multiplicative factor independent of $n$, $N$, $p$, or $s$, and by the logarithmic factor whenever $s\ll p$; the lasso-$r_n$ upper bound carries $\log p$, while the lower bound carries $\log(ep/s)$. TC--Cross-PPI++ at $\lambda_0=N/(n+N)$ is thus minimax optimal.
\end{remark} 

\vspace{-0.1in}
\section{Empirical Experiments and Results}\label{sec:empirical}
\vspace{-0.1in}
\begin{figure}[htb!]
  \centering
  \includegraphics[width=\textwidth]{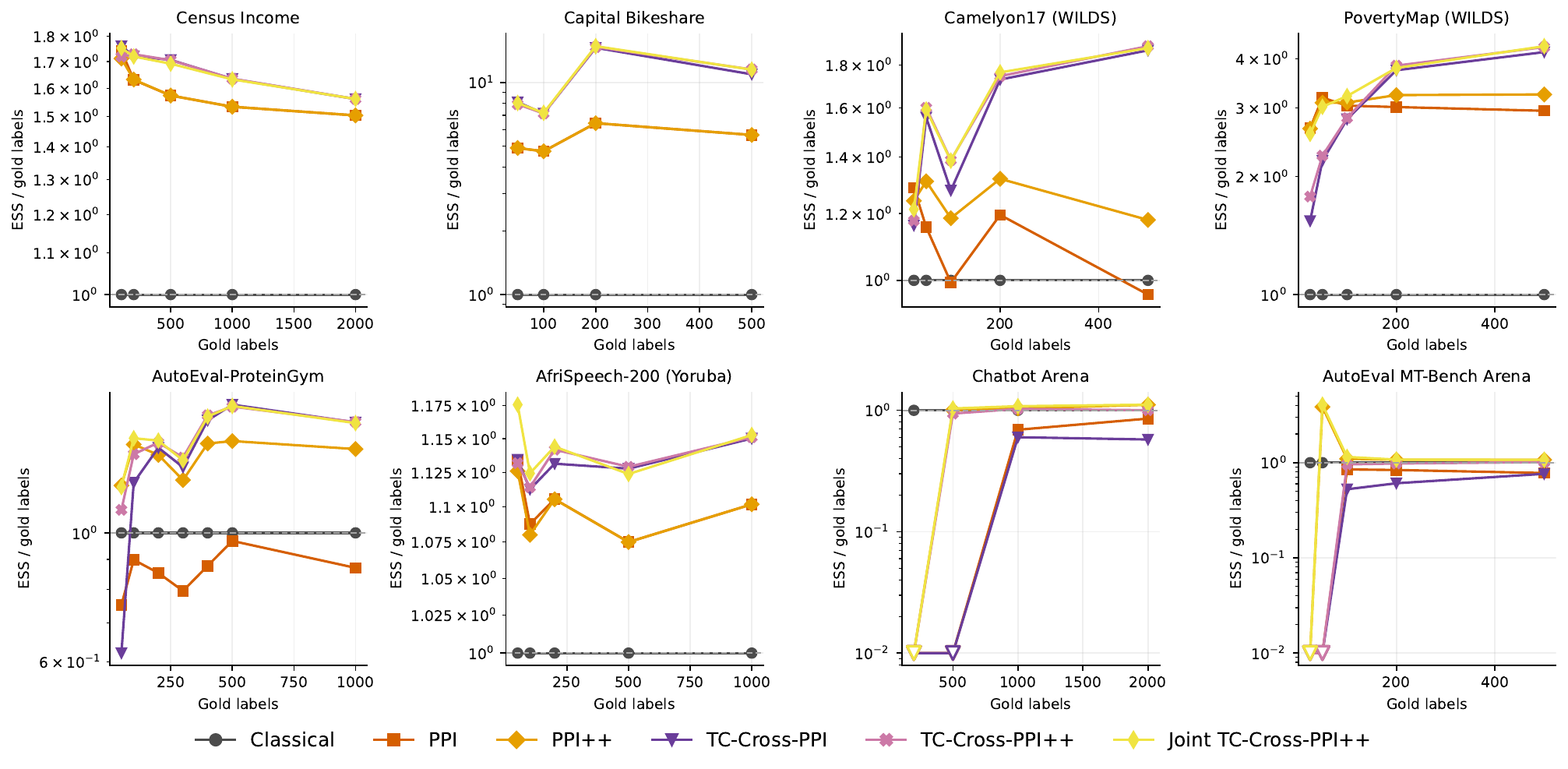}
  \caption{Empirical effective sample size (ESS) divided by the number of gold labels, on a logarithmic $y$-axis. The dashed line at $1$ is classical inference. Values above one are MSE \textbf{improvements} over the gold-only estimator. Values below $0.01$ are drawn at $0.01$ as open downward triangles.}
  \vspace{-0.1in}
  \label{fig:empirical-ess}
\end{figure}
We evaluate transfer-calibrated PPI variants across a wide range of datasets. We use all the target-domain observations for evaluation, with $n$ sampled points serving as ``gold labels'' and the rest as prediction-rich unlabeled samples. The ground truth is the parameter estimated using the entire target-domain data. We set the nominal coverage to $95\%$, except for the two AutoEval replications which match the original coverage of $90\%$ (Appendix~\ref{app:experimental_setup}).
In addition to reporting classical inference, PPI, and PPI++, we report metrics for the cross-fitted and joint-power tuned variants of TC-PPI
\footnote{Split TC-PPI(++) performs poorly due to half the data not being used for rectification and is omitted.}. We include some results with other PPI variants in Appendix~\ref{app:other_ppi_results}. Metrics reported include MSE, interval width, and Effective Sample Size (ESS) defined as $\operatorname{ESS}\coloneqq n\cdot({\operatorname{MSE}_{\rm classical}}/{\operatorname{MSE}_{\rm method}})$. 
% See Figure~\ref{fig:empirical-ess} for the relative ESS and Figure~\ref{fig:empirical-width} for confidence interval widths.
%We report coverage and interval width results in Appendix~\ref{app:additional_results}.
LoRA-based calibration for protein and LLM-arena applications is in Appendix~\ref{app:experimental_setup}~and~\ref{app:additional_results}.
\begin{figure}[hb!]
  \centering
  \includegraphics[width=\textwidth]{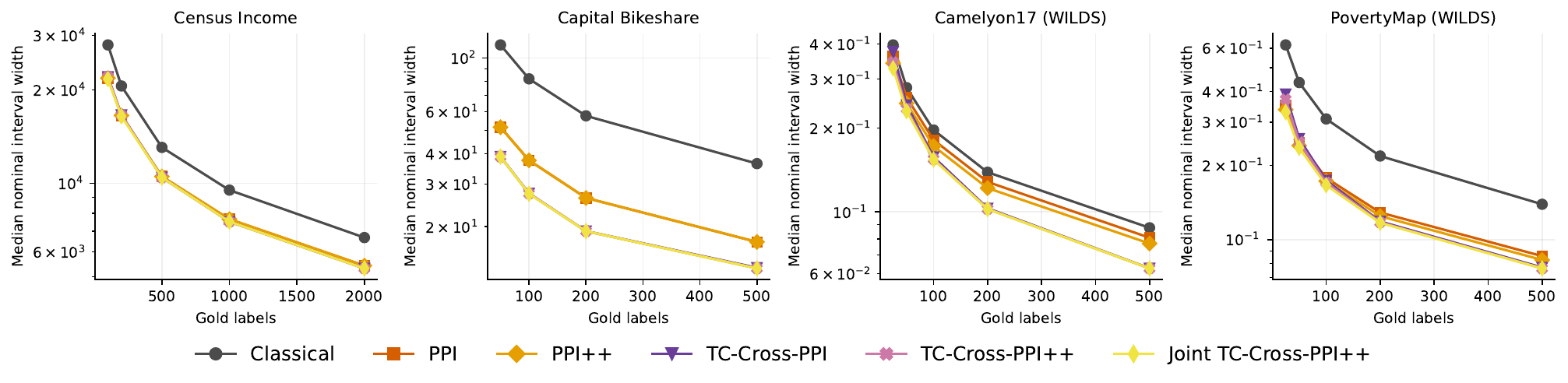}
    \caption{Median (over trials) width (in estimand's natural units (e.g., dollars for Census)) of nominal $95\%$ confidence intervals, on a logarithmic $y$-axis. Remaining datasets are in Appendix~\ref{app:additional_results}.}
  \label{fig:empirical-width}
\end{figure}
\vspace{-6mm}

\noindent\textbf{Census income.}
For mean income, the source XGBoost predictions are informative for the target year as PPI/PPI++ and the cross-fitted TC-PPI methods all substantially improve on classical inference, with relative MSE of $0.58$--$0.61$ at $n=200$ and ESS gains of $63$--$73\%$. TC--Cross-PPI++ and Joint TC--Cross-PPI++ are tied for best. Sparse calibration changes the residual variance by only a few percent (Table~\ref{tab:residual-variance} in Appendix~\ref{app:additional_results}), so most of this gain comes from the source predictor itself.

\noindent\textbf{Temporal transfer of bicycle demand (Bikeshare).}
We train a gradient-boosted source predictor on 2011 data and transfer it to 2012. The inferential target is the mean 2012 hourly rental demand, $234.67$. The source predictions correlate strongly with target outcomes ($0.953$) but their mean is only $150.27$, and the error varies with calendar and weather covariates rather than being a constant offset. Transfer calibration roughly halves the residual variance (Table~\ref{tab:residual-variance}), producing some of the largest gains we observe. At $n=100$, relative MSE ranges from $0.21$ for PPI++ down to $0.14$ for the cross-fitted methods, corresponding to ESS gains of $370$--$620\%$. Joint TC--Cross-PPI++ is best.

% \noindent\textbf{AfriSpeech-200 (Yoruba).} We estimate the mean word error rate (WER) for Yoruba-accented English speech (AfriSpeech-200~\citep{olatunji2023afrispeech}) using \texttt{wav2vec2-base-960h}~\citep{baevski2020wav2vec20frameworkselfsupervised} as the source ASR model. Gold labels come only from AfriSpeech-200's \texttt{train} split, while \texttt{dev} and \texttt{test} remain unlabeled. Since the splits may not contain the same mix of speakers, this departs slightly from the i.i.d. sampling we use for the other datasets. The gains are modest but consistent across all $n$. TC--Cross-PPI gains are between $9.2\%$ ESS at $n=1000$ and $20.3\%$ at $n=500$, and coverage stays within $0.91$--$0.97$ at every $n$ and method. 
\noindent\textbf{AfriSpeech-200 (Yoruba).} We estimate the mean word error rate (WER) for Yoruba-accented English speech~\citep{olatunji2023afrispeech} using \texttt{wav2vec2-base-960h}~\citep{baevski2020wav2vec20frameworkselfsupervised} as the source ASR model. The gains are modest but consistent across all $n$: every method improves on classical inference at every budget, TC--Cross-PPI's ESS gains range from $11.3\%$ at $n=100$ to $15.0\%$ at $n=1000$, and coverage stays within $0.91$--$0.98$ at every $n$ and method.

\noindent\textbf{Distribution shift in the wild.}
The two WILDS benchmarks~\citep{koh2021wilds} resemble Bikeshare: the source predictions are correlated with the target, but their errors differ across target sub-populations (hospitals, slides, countries) in ways the correction can learn. Camelyon17~\citep{bandi2018detection} has histopathology patches from hospitals with different staining protocols; we transfer a DenseNet-121 to two held-out hospitals to estimate tumor prevalence ($0.500$). PovertyMap~\citep{yeh2020using} pairs satellite imagery with survey-based wealth measurements; we transfer an 8-band ResNet18 to ten held-out countries to estimate mean wealth ($-0.049$). Both sources degrade out of domain. Camelyon17's accuracy falls from $0.966$ in distribution to $0.787$ and $0.673$ on the two held-out hospitals, PovertyMap's correlation falls from $0.868$ to $0.830$, and their mean predictions are $0.339$ and $0.115$ against true values of $0.500$ and $-0.049$. On PovertyMap, cross-fitted methods reach relative MSEs of $0.23$--$0.27$ for $n\ge200$, versus $0.31$--$0.34$ for PPI/PPI++; on Camelyon17 they reach $0.53$--$0.58$ versus $0.76$--$1.04$, with PPI trailing classical at several budgets. Coverage stays between $0.88$ and $0.99$.

% \subsection{Protein-Modeling Applications}
% \noindent\textbf{ProteinGym.}
% We analyze the \texttt{BLAT\_ECOLX\_Firnberg\_2014} assay with ESM-2 8M (\texttt{esm2\_t6\_8M\_UR50D}) as the source predictor. The raw source score is poorly aligned with experimental fitness, and vanilla PPI pins $\lambda$ to one, so it is highly inefficient here, with relative MSE above $20\times$ classical at every $n$. Power tuning fixes this. PPI++ reaches relative MSE of $0.79$--$0.82$ across $n=200$--$500$, and the cross-fitted methods improve further still, to $0.65$--$0.77$, or ESS gains of up to ${\sim}53\%$.
% %, PPI++ is $0.815$ ($+22.7\%$) and TC--Cross-PPI++/Joint are $0.773$/$0.766$ ($+29.4\%$/$+30.6\%$). 
% For model ranking we target $\mu_m=\mathbb E[Ys_m(X)]$, where $s_m$ is model $m$'s standardized zero-shot score. With three candidate ESM models, TC--Cross-PPI++ reaches mean rank correlation $1.000$ at $n=100$, against $0.917$ for PPI++ and $0.250$ for vanilla PPI.

\begin{table}[htb!]
  \centering
  \caption{Representative small-label results for different datasets and the inferential target in parentheses. MSE and interval width are normalized by classical inference at the same label budget.}
  \label{tab:empirical-main}
  \resizebox{\textwidth}{!}{%
    \begin{tabular}{lccccc}
\toprule
& \multicolumn{2}{c}{\textit{Tabular (Mean)}}
& \multicolumn{1}{c}{\textit{Speech (WER)}}
& \multicolumn{1}{c}{\begin{tabular}[c]{@{}c@{}}
\textit{Protein Language Models} \\
\textit{(Mean Mutational Fitness)}
\end{tabular}}
&
\\
\cmidrule(lr){2-3}
\cmidrule(lr){4-4}
\cmidrule(lr){5-5}
Method
& \begin{tabular}[c]{@{}c@{}}Census Income\\$n=200$\end{tabular}
& \begin{tabular}[c]{@{}c@{}}Capital Bikeshare\\$n=100$\end{tabular}
& \begin{tabular}[c]{@{}c@{}}AfriSpeech-200\\(Yoruba), $n=500$\end{tabular}
& \begin{tabular}[c]{@{}c@{}}AutoEval-ProteinGym\\$n=500$\end{tabular}
&
\\
\midrule
Classical
& (1.000, 1.000) & (1.000, 1.000) & (1.000, 1.000) & (1.000, 1.000) &
\\
Joint TC-Cross-PPI++
& (0.581, \textbf{0.798})
& (\textbf{0.139}, \textbf{0.334})
& (0.890, \textbf{0.946})
& (\textbf{0.605}, \textbf{0.806})
&
\\
PPI++
& (0.613, 0.803)
& (0.211, 0.460)
& (0.930, 0.963)
& (0.695, 0.850)
&
\\
TC-Cross-PPI++
& (\textbf{0.578}, 0.806)
& (0.140, 0.335)
& (\textbf{0.886}, 0.950)
& (0.607, 0.809)
&
\\

\midrule
& \multicolumn{2}{c}{\textit{WILDS (Mean)}}
& \multicolumn{2}{c}{\textit{LLM Arena (Bradley--Terry)}}
&
\\
\cmidrule(lr){2-3}
\cmidrule(lr){4-5}
Method
& \begin{tabular}[c]{@{}c@{}}Camelyon17 (WILDS)\\$n=100$\end{tabular}
& \begin{tabular}[c]{@{}c@{}}PovertyMap (WILDS)\\$n=200$\end{tabular}
& \begin{tabular}[c]{@{}c@{}}Chatbot Arena\\$n=500$\end{tabular}
& \begin{tabular}[c]{@{}c@{}}AutoEval MT-Bench Arena\\$n=500$\end{tabular}
&
\\
\midrule
Classical
& (1.000, 1.000) & (1.000, 1.000) & (1.000, 1.000) & (1.000, 1.000) &
\\
Joint TC-Cross-PPI++
& (0.721, \textbf{0.784})
& (0.264, \textbf{0.535})
& (\textbf{0.965}, \textbf{0.946})
& (0.933, \textbf{0.959})
&
\\
PPI++
& (0.844, 0.875)
& (0.310, 0.569)
& (0.996, 0.962)
& (\textbf{0.932}, 0.965)
&
\\
TC-Cross-PPI++
& (\textbf{0.720}, 0.788)
& (\textbf{0.260}, 0.538)
& (1.063, 1.009)
& (0.987, 0.976)
&
\\
\bottomrule
\end{tabular}%
  }
\end{table}

\noindent\textbf{AutoEval-ProteinGym (SPG1 replication).}
We reproduce AutoEval's~\citep{boyeau2024autoeval} released SPG1 assay comparison, in which seven candidate zero-shot models (see Appendix~\ref{app:experimental_setup}) provide source/annotator scores. PPI++ has relative MSE of $0.70$--$0.74$ across $n=200$--$500$. The cross-fitted methods improve further as $n$ grows, reaching an MSE of $0.61$, or ESS gains of up to $65\%$; Joint TC--Cross-PPI++ is consistently the best performing. We also study annotator-quality sensitivity across eight annotators (the seven candidates plus VESPA) and repeat it for the standard ProteinGym BLAT assay with three ESM-2 annotators (see Appendix~\ref{app:additional_results}).

\noindent\textbf{Chatbot Arena.}
We keep the 20 most compared models and use Meta-Llama-3-8B-Instruct as the source judge. With few labeled points, the 19-parameter Bradley--Terry model is underdetermined, and every method except Classical suffers quasi-separation. Near separation, the power-weight estimate becomes ill-conditioned, which is a finite-sample effect, not a violation of the asymptotic no-harm guarantee. (Theorem~\ref{thm:joint-tc-cross}). From $n\geq 500$, PPI++ and Joint TC--Cross-PPI++ match or beat Classical, reaching $0.90$ relative MSE (ESS $+11\%$) at $n=2000$.

\noindent\textbf{AutoEval MT-Bench Arena.} We reproduce AutoEval's MT-Bench experiment, matching human ratings to GPT-4 pairwise judgments across six systems, calibrated with frozen MiniLM sentence embeddings. Like Chatbot Arena, the power-tuned methods help reliably. PPI++ and Joint TC--Cross-PPI++ have relative MSE of $0.910$ and $0.870$ at $n=100$, or ESS gains of $9.9\%$ and $14.9\%$, and $0.932$ and $0.933$ at $n=500$, both $+7.2\%$. The fixed-$\lambda$ PPI and TC--Cross-PPI stay worse than classical throughout, and TC--Cross-PPI++ becomes beneficial at $n=500$ ($0.988$, $+1.3\%$). 

% We compare against StratPPI~\citep{fisch2024stratified} in Appendix~\ref{app:additional_results} on the Arena datasets.

% \subsection{Summary}
% Coverage approaches nominal levels once the gold sample is adequate for the target dimension, although undercoverage remains at the smallest Census sample sizes and in one ProteinGym annotator-quality configuration. Figure~\ref{fig:representative-intervals} illustrates the corresponding interval geometry: cross-fitted calibration frequently moves the point estimate toward the full-data target while retaining intervals substantially shorter than classical inference, whereas split TC intervals are often wider because only half the labels are available for rectification. The cross-fitted results materially change the empirical conclusion: much of split TC--PPI's poor performance was a label-allocation artifact. TC--Cross-PPI recovers full-label efficiency in the Census analyses, improves on PPI++ for all three ProteinGym annotators, and yields especially large gains under temporal shift in bicycle demand. It is not uniformly superior to PPI++, because calibration must improve the estimand-relevant influence function, and high-dimensional Arena inference remains unstable at extremely small $n$. Overall, two-fold transfer-calibrated cross-fitting is the preferred default implementation, while the Arena judge-quality study remains explicitly exploratory.
\vspace{-0.1in}
\section{Concluding Remarks}
\vspace{-0.08in}
%In this work, we introduce Transfer-Calibrated Prediction Powered Inference (TC-PPI), which adapts a pretrained source predictor to a target domain using the same gold-standard data used for the rectification step in PPI. TC-PPI avoids needing to train a target-domain predictor from scratch without having to accept an efficiency loss due to a poorly-aligned source predictor. We discuss how applying transfer-learning within PPI is non-trivial since splitting the gold sample or reusing it for both calibration and rectification leads to issues with unbiasedness and efficiency. To address this we've introduced cross-fitting construction along with a joint-power-tuning construction to guard against negative transfer. 
We introduce Transfer-Calibrated Prediction Powered Inference (TC-PPI), a framework that adapts a pretrained source predictor to a target domain using gold-standard data from PPI’s rectification step. This approach avoids the inefficiencies of training a target predictor from scratch and  provably addresses source-predictor misalignment issues. We highlight the challenges of transfer learning in PPI, where using the gold sample for both calibration and rectification can introduce bias. To mitigate, we propose cross-fitting and a joint power-tuning approach to prevent negative transfer. We recommend that researchers always adapt a pre-trained AI/ML model to their data using our TC-PPI framework for potentially significant gains with asymptotically negligible downside risk.

\begin{comment}
Transfer calibration does not alter the validity mechanism of PPI. The target gold sample is still what guarantees valid inference. What transfer calibration changes is the quality of the predictor entering the PPI correction. We formulate the source-to-target discrepancy as the general function
\[
\Delta^*(x)=f_t(x)-f_s(x),
\]
and allow the learned correction $\widehat\Delta$ to come from sparse linear calibration, last-layer fine-tuning, adapters, LoRA, or another transfer procedure. For mean estimation, vanilla PPI pays the drift variance $\Var_t\{\Delta^*(X)\}$ in its rectification term, whereas TC-PPI replaces it by the variance of $\widehat\Delta(X)-\Delta^*(X)$. Under sparse linear calibration, this specializes from
\[
\Var_t(\varepsilon)+(\delta^*)^\top\Sigma_t\delta^*
\]
to
\[
\Var_t(\varepsilon)+O_p\!\left(\frac{s\log p}{n_1}\right).
\]
The cross-fitted version additionally recovers all gold labels for rectification and, under prediction consistency, has the same first-order limit as oracle PPI using the target-calibrated predictor. Joint power tuning further nests ordinary PPI++ and TC-PPI++ in a single class. Its oracle no-harm safeguard is model-agnostic and does not require the transfer correction to be consistently estimated: the transfer direction can be ignored when it is unhelpful. Consistency is needed only for the stronger oracle-efficiency conclusion. This supports Joint TC--Cross-PPI++ as the recommended formulation, with the scalar cross-fitted procedures serving as important special cases. 
\end{comment}

\newpage
\subsection*{AI use statement}

For quick initial prototyping of our framework code, we relied on Codex/Claude Code. Subsequent versions relied on manual corrections and interventions.  The authors wrote all proofs and the original draft manually. We used Grammarly for grammar corrections and, in a few cases, to improve the prose for clarity. We did not use AI tools for the proofs or references. We used Google's PAT to identify gaps, which we addressed in the current version of the draft.

We, the authors, take responsibility for the final content of this work.

% \subsection*{Ethics statement}

% (This section is \textbf{recommended} and does not count toward the page limit.)

% If authors feel that their paper submission raises questions regarding the Code
% of Ethics, they are encouraged to include a paragraph of Ethics Statement (at
% the end of the main text before references) to address potential concerns where
% appropriate. Topics include, but are not limited to, studies that involve human
% subjects, practices to data set releases, potentially harmful insights,
% methodologies and applications, potential conflicts of interest and sponsorship,
% discrimination/bias/fairness concerns, privacy and security issues, legal
% compliance, and research integrity issues (e.g., IRB, documentation, research
% ethics). This statement should not be more than 1 page.

\subsection*{Reproducibility statement}

In the spirit of reproducibility, proofs are available for the theoretical aspects of our work, and details of the experiments, including datasets and models, can be found in Appendix~\ref{app:experimental_setup}. The source code is available in the supplemental material.



\bibliography{iclr2027_conference}
\bibliographystyle{iclr2027_conference}

\newpage
\appendix
\section{Nonasymptotic and minimax results for sparse linear calibration}
\label{sec:mse-information-gain}

\subsection{Nonasymptotic MSE upper bounds}
Recall, cross-fitting uses all $n$ gold observations for rectification, while each fold-specific transfer correction is estimated on a complementary training sample of size $
n_{\rm tr}\coloneqq \frac{K-1}{K}n$, under balanced $K$-fold cross-fitting. Further recall that we condition on the source data $\mathcal D_s$ and treat the source predictor $f_s$ as fixed. The target domain model with sparse linear discrepancy assumption is,
\[
Y=f_s(X)+X^\top\delta^*+\epsilon,
\qquad
\E_t[\epsilon\mid X]=0,
\qquad
\|\delta^*\|_0\le s,
\]
and define
\[
f_t(X)\coloneqq  f_s(X)+X^\top\delta^*=\E_t[Y\mid X].
\]
We use the following short-hand notations,
\[
\sigma_\epsilon^2\coloneqq \Var_t(\epsilon),
\qquad
V_t\coloneqq \Var_t\{f_t(X)\},
\qquad
V_s\coloneqq \Var_t\{f_s(X)\}.
\]
For simplicity, assume the predictors are centered, i.e., $\E_t[X]=0$. If this does not hold,  then the quadratic forms below should be interpreted with the centered covariance matrix. The variance of the source-to-target drift term is
\[
R_s\coloneqq \Var_t\{X^\top\delta^*\}
=(\delta^*)^\top\Sigma_t\delta^*.
\]

For each fold $k$, let $\widehat\delta^{(-k)}$ denote the lasso calibration estimator fitted on the training complement $\mathcal D_{-k}$ and define
\[
e_k(x)
\coloneqq 
x^\top(\widehat\delta^{(-k)}-\delta^*),
\]
as the prediction error for the calibration estimator.
Let
\begin{equation}\label{eq:fold-cal-risk}
a_{n,k}
\coloneqq 
\E_t\!\left[
e_k(X)^2
\mid
\mathcal D_s,\mathcal D_{-k}
\right],
\qquad
r_n
\coloneqq 
\max_{1\le k\le K}
\E\!\left[
a_{n,k}
\mid
\mathcal D_s
\right].
\end{equation}
Then the quantity $r_n$ measures the maximum foldwise prediction risk of the transfer-calibration step. For the lasso based transfer calibration, this prediction risk from proposition \ref{thm:lasso-calibration-rig},
\[
r_n
\lesssim
\frac{s\log p}{n_{\rm tr}}
=
\frac{s\log p}{(K-1)n/K}.
\]

%\subsection{A nonasymptotic oracle inequality for TC--Cross-PPI}

Recall the decomposition from Theorem~\ref{thm:tc-cross-mean},
\[
\widehat\mu_{\rm TC\text{-}Cross}-\theta_t^*
=
\frac1n\sum_{i=1}^n\epsilon_i
+
\frac1N\sum_{j=1}^N
\{f_t(\widetilde X_j)-\E_t f_t(X)\}
+
R_n,
\]
where
\[
R_n
=
\frac1K\sum_{k=1}^K
\left[
\frac1N\sum_{j=1}^N e_k(\widetilde X_j)
-
\frac K n\sum_{i\in I_k}e_k(X_i)
\right].
\]
The first two terms form the oracle PPI estimator based on $f_t$, while $R_n$ contains the error from estimating the transfer correction.

The bound in Theorem \ref{thm:tc-cross-nonasymptotic} separates the oracle PPI variance (which requires knowing $f_t$),
\[
\frac{\sigma_\epsilon^2}{n}+\frac{V_t}{N},
\]
from the finite-sample cost of estimating the fold-specific transfer corrections. In particular, if
\begin{equation}\label{eq:mean-lasso-risk-cross}
r_n
\le
C_{\rm cal}
\frac{s\log p}{n_{\rm tr}},
\end{equation}
then
\begin{align}
\operatorname{MSE}\!\left(
\widehat\mu_{\rm TC\text{-}Cross}
\mid\mathcal D_s
\right)
\le{}&
\frac{\sigma_\epsilon^2}{n}
+
\frac{V_t}{N}
+
\frac{2\sqrt{C_{\rm cal}V_t}}{N}
\sqrt{\frac{s\log p}{n_{\rm tr}}} + 
C_{\rm cal}
\left(
\frac1N+\frac K n
\right)
\frac{s\log p}{n_{\rm tr}}.
\label{eq:tc-cross-lasso-mse}
\end{align}
For fixed $K$ and $N\gtrsim n$, since $n_{\rm tr}\asymp n$, the excess over the oracle MSE is therefore at most of order
\[
O\!\left(
\frac{\sqrt{s\log p}}{n^{3/2}}
+
\frac{s\log p}{n^2}
\right).
\]
When the unlabeled sample is sufficiently large, the first term becomes negligible and the contribution of transfer-calibration estimation to the MSE is of order $s\log p/n^2$.

\subsection{Proof of Theorem \ref{thm:tc-cross-nonasymptotic} }
\begin{proof}[Proof of Theorem~\ref{thm:tc-cross-nonasymptotic}]
Condition throughout on the source data $\mathcal D_s$. Define
\[
A_n
\coloneqq 
\frac1n\sum_{i=1}^n\epsilon_i,
\qquad
B_N
\coloneqq 
\frac1N\sum_{j=1}^N
\{f_t(\widetilde X_j)-\E_t f_t(X)\},
\]
and, for each fold $k$,
\[
Z_{n,k}
\coloneqq 
\frac1N\sum_{j=1}^N e_k(\widetilde X_j)
-
\frac K n\sum_{i\in I_k}e_k(X_i).
\]
Then
\[
R_n=\frac1K\sum_{k=1}^K Z_{n,k}.
\]
Therefore combining these components the error of the TC-Cross-PPI estimator can be written as,
\[
\widehat\mu_{\rm TC\text{-}Cross}-\theta_t^*
=
A_n+B_N+R_n.
\]

By $\E_t[\epsilon\mid X]=0$, we can write
\[
\E[A_n\mid\mathcal D_s]=0,
\qquad
\Var(A_n\mid\mathcal D_s)
=
\frac{\sigma_\epsilon^2}{n}.
\]
Similarly, looking at $B_N$ we have,
\[
\E[B_N\mid\mathcal D_s]=0,
\qquad
\Var(B_N\mid\mathcal D_s)
=
\frac{V_t}{N}.
\]
The gold and unlabeled samples are independent, so
\[
\E[A_nB_N\mid\mathcal D_s]=0.
\]
Hence the oracle part $A_n+B_N$ has the conditional MSE of
\[
\E[(A_n+B_N)^2\mid\mathcal D_s]
=
\frac{\sigma_\epsilon^2}{n}
+
\frac{V_t}{N}.
\]

We next control the cross-fitting remainder. Recall that for each fixed fold $k$, $
\mathcal D_{-k}
=
\{(X_i,Y_i):i\notin I_k\}.$ denotes the combined data from other folds.
Conditional on $\mathcal D_s$ and $\mathcal D_{-k}$, the function $e_k$ is
fixed, the observations in $I_k$ are independent of $e_k$. Further, the unlabeled
sample is independent of both the training complement and the held-out fold.
The two empirical averages defining $Z_{n,k}$ therefore have the same
conditional population mean, so
\[
\E[Z_{n,k}\mid\mathcal D_s,\mathcal D_{-k}]=0.
\]
Moreover,
\begin{align*}
\E[Z_{n,k}^2\mid\mathcal D_s,\mathcal D_{-k}]
&=
\frac1N
\Var_t\{e_k(X)\mid\mathcal D_s,\mathcal D_{-k}\}+
\frac K n
\Var_t\{e_k(X)\mid\mathcal D_s,\mathcal D_{-k}\}\\
&\le
\left(
\frac1N+\frac K n
\right)a_{n,k}.
\end{align*}
Note that the above calculation so far is made fold-wise. Since the $Z_{n,k}$ need not be independent across folds, we use
\[
R_n^2
=
\left(
\frac1K\sum_{k=1}^K Z_{n,k}
\right)^2
\le
\frac1K\sum_{k=1}^K Z_{n,k}^2.
\]
Taking iterated conditional expectations gives
\begin{align}
\E[R_n^2\mid\mathcal D_s]
&\le
\frac1K\sum_{k=1}^K
\E[Z_{n,k}^2\mid\mathcal D_s] \le
\left(
\frac1N+\frac K n
\right)
\frac1K\sum_{k=1}^K
\E[a_{n,k}\mid\mathcal D_s]\nonumber\\
&\le
\left(
\frac1N+\frac K n
\right)r_n.
\label{eq:proof-Rn-second-moment}
\end{align}
The above display then provides the conditional MSE of the remainder term. The only thing left to examine is the covariance between the oracle part and $R_n$. First,
\[
\E[A_nR_n\mid\mathcal D_s]=0.
\]
To see this, it is enough to show
$\E[A_nZ_{n,k}\mid\mathcal D_s]=0$ for every $k$. Conditional on
$\mathcal D_s$ and $\mathcal D_{-k}$, decompose
\[
A_n
=
\frac1n\sum_{i\notin I_k}\epsilon_i
+
\frac1n\sum_{i\in I_k}\epsilon_i.
\]
The first term is fixed under this conditioning and has zero covariance with
$Z_{n,k}$ because
$\E[Z_{n,k}\mid\mathcal D_s,\mathcal D_{-k}]=0$.
For the second term, the unlabeled part of $Z_{n,k}$ is independent of the
held-out observations, while
\[
\E_t[\epsilon\, e_k(X)\mid\mathcal D_s,\mathcal D_{-k}]
=
\E_t\!\left[
e_k(X)\E_t(\epsilon\mid X)
\mid\mathcal D_s,\mathcal D_{-k}
\right]
=0.
\]
Independence across distinct held-out observations then gives
\[
\E\!\left[
\left(
\frac1n\sum_{i\in I_k}\epsilon_i
\right)
Z_{n,k}
\middle|
\mathcal D_s,\mathcal D_{-k}
\right]
=0.
\]
Iterated expectation proves the claim.

The covariance of $B_N$ with the held-out component of $Z_{n,k}$ is zero,
because $B_N$ depends only on the unlabeled sample. Conditional on
$\mathcal D_s$ and $\mathcal D_{-k}$, the covariance with the unlabeled
component is
\[
\E[B_N Z_{n,k}\mid\mathcal D_s,\mathcal D_{-k}]
=
\frac1N
\Cov_t\{f_t(X),e_k(X)\mid\mathcal D_s,\mathcal D_{-k}\}.
\]
Therefore, by Cauchy--Schwarz,
\[
\left|
\E[B_N Z_{n,k}\mid\mathcal D_s,\mathcal D_{-k}]
\right|
\le
\frac{\sqrt{V_t}}{N}\sqrt{a_{n,k}}.
\]
Averaging over folds, applying iterated expectation, and using Jensen's
inequality gives
\begin{align}
\left|
\E[B_NR_n\mid\mathcal D_s]
\right|
&\le
\frac{\sqrt{V_t}}{KN}
\sum_{k=1}^K
\E[\sqrt{a_{n,k}}\mid\mathcal D_s]\nonumber\\
&\le
\frac{\sqrt{V_t}}{KN}
\sum_{k=1}^K
\sqrt{\E[a_{n,k}\mid\mathcal D_s]}\nonumber\\
&\le
\frac{\sqrt{V_t r_n}}{N}.
\label{eq:proof-BR-bound}
\end{align}

By Proposition~\ref{prop:tc-cross-unbiased},
$\widehat\mu_{\rm TC\text{-}Cross}$ is conditionally unbiased given
$\mathcal D_s$, so its conditional MSE is its conditional second moment around
$\theta_t^*$. Combining the preceding calculations,
\begin{align*}
&
\operatorname{MSE}\!\left(
\widehat\mu_{\rm TC\text{-}Cross}
\mid\mathcal D_s
\right)
-
\left(
\frac{\sigma_\epsilon^2}{n}
+
\frac{V_t}{N}
\right)\\
&\qquad=
2\E[(A_n+B_N)R_n\mid\mathcal D_s]
+
\E[R_n^2\mid\mathcal D_s]\\
&\qquad=
2\E[B_NR_n\mid\mathcal D_s]
+
\E[R_n^2\mid\mathcal D_s].
\end{align*}
Using \eqref{eq:proof-Rn-second-moment} and
\eqref{eq:proof-BR-bound} yields
\[
\left|
\operatorname{MSE}\!\left(
\widehat\mu_{\rm TC\text{-}Cross}
\mid\mathcal D_s
\right)
-
\left(
\frac{\sigma_\epsilon^2}{n}
+
\frac{V_t}{N}
\right)
\right|
\le
\frac{2\sqrt{V_t r_n}}{N}
+
\left(
\frac1N+\frac K n
\right)r_n,
\]
The upper bound
\eqref{eq:tc-cross-mse-upper} follows immediately after applying the simple fact; $\sqrt{2ab}\le a+b$ for $a,b>0$.

We now establish the second claim of the theorem for the
TC--Cross-PPI++ estimator with the fixed power parameter
\[
\lambda_0
=
\frac{1}{1+n/N}
=
\frac{N}{n+N}.
\]

In addition to $B_N$, define the centered gold-sample average
\[
B_n^{\rm L}
\coloneqq
\frac1n\sum_{i=1}^n
\{f_t(X_i)-\E_t f_t(X)\}.
\]
Recall that for $i\in I_k$,
\[
\widehat f_t^{(-k)}(X_i)
=
f_t(X_i)+e_k(X_i),
\]
and similarly, for an unlabeled observation,
\[
\frac1K\sum_{k=1}^K
\widehat f_t^{(-k)}(\widetilde X_j)
=
f_t(\widetilde X_j)
+
\frac1K\sum_{k=1}^K e_k(\widetilde X_j).
\]
Therefore, for any fixed $\lambda$,
\begin{align}
\widehat\mu_{\rm TC\text{-}Cross,\lambda}
-\theta_t^*
&=
A_n
+
(1-\lambda)B_n^{\rm L}
+
\lambda B_N
+
\lambda R_n.
\label{eq:proof-ppipp-decomp}
\end{align}
The first three terms in
\eqref{eq:proof-ppipp-decomp} are components of the  error of the oracle
PPI++ estimator based on the knowledge of $f_t$, while $\lambda R_n$ is the contribution
from estimating the transfer correction.

For convenience, write
\[
C_{n,N}(\lambda)
\coloneqq
(1-\lambda)B_n^{\rm L}
+
\lambda B_N.
\]
Since the labeled and unlabeled samples are independent,
\[
\E[B_n^{\rm L}B_N\mid\mathcal D_s]=0.
\]
Moreover, because $\E_t[\epsilon\mid X]=0$,
\[
\E[A_nB_n^{\rm L}\mid\mathcal D_s]=0,
\qquad
\E[A_nB_N\mid\mathcal D_s]=0.
\]
Thus
\begin{align}
\E\!\left[
\{A_n+C_{n,N}(\lambda)\}^2
\mid\mathcal D_s
\right]
&=
\frac{\sigma_\epsilon^2}{n}
+
\left\{
\frac{(1-\lambda)^2}{n}
+
\frac{\lambda^2}{N}
\right\}V_t.
\label{eq:proof-oracle-ppipp-var}
\end{align}
At $\lambda=\lambda_0=N/(n+N)$,
\[
\frac{(1-\lambda_0)^2}{n}
+
\frac{\lambda_0^2}{N}
=
\frac{1}{n+N},
\]
and therefore
\begin{equation}
\E\!\left[
\{A_n+C_{n,N}(\lambda_0)\}^2
\mid\mathcal D_s
\right]
=
\frac{\sigma_\epsilon^2}{n}
+
\frac{V_t}{n+N}.
\label{eq:proof-oracle-ppipp-opt-var}
\end{equation}

It remains to control the interaction term between the oracle PPI++ part (the first 3 terms) 
and the calibration remainder (the last term). We already showed above that
\[
\E[A_nR_n\mid\mathcal D_s]=0.
\]
We next show that the choice $\lambda_0$ also makes
$C_{n,N}(\lambda_0)$ orthogonal to $R_n$.

Fix a fold $k$ and condition on $\mathcal D_s$ and
$\mathcal D_{-k}$. The part of $B_n^{\rm L}$ involving observations
outside $I_k$ is then fixed, and hence has zero expectation when
multiplied by $Z_{n,k}$ because
\[
\E[Z_{n,k}\mid\mathcal D_s,\mathcal D_{-k}]=0.
\]
For the held-out part, the unlabeled component of $Z_{n,k}$ is
independent of the observations in $I_k$, while independence across
held-out observations gives
\begin{align}
\E[
B_n^{\rm L} Z_{n,k}
\mid
\mathcal D_s,\mathcal D_{-k}]
&=
-\frac1n
\Cov_t\{
f_t(X),e_k(X)
\mid
\mathcal D_s,\mathcal D_{-k}
\}.
\label{eq:proof-BL-Z}
\end{align}
On the other hand, as derived above,
\begin{align}
\E[
B_N Z_{n,k}
\mid
\mathcal D_s,\mathcal D_{-k}]
&=
\frac1N
\Cov_t\{
f_t(X),e_k(X)
\mid
\mathcal D_s,\mathcal D_{-k}
\}.
\label{eq:proof-BN-Z}
\end{align}
Combining \eqref{eq:proof-BL-Z} and
\eqref{eq:proof-BN-Z}, for a general fixed $\lambda$,
\begin{align}
\E[
C_{n,N}(\lambda)Z_{n,k}
\mid
\mathcal D_s,\mathcal D_{-k}]
=
\left\{
\frac{\lambda}{N}
-
\frac{1-\lambda}{n}
\right\}
\Cov_t\{
f_t(X),e_k(X)
\mid
\mathcal D_s,\mathcal D_{-k}
\}.
\label{eq:proof-ppipp-cov-cancel}
\end{align}
Now setting $
\lambda=\lambda_0=\frac{N}{n+N},$ in the above equation,
we have
\[
\frac{\lambda_0}{N}
=
\frac{1-\lambda_0}{n}
=
\frac{1}{n+N},
\]
and consequently, the right-hand side of
\eqref{eq:proof-ppipp-cov-cancel} is identically zero. Averaging over
the folds and applying iterated expectation yields
\begin{equation}
\E[
C_{n,N}(\lambda_0)R_n
\mid
\mathcal D_s]
=
0.
\label{eq:proof-ppipp-cross-zero}
\end{equation}

Combining
\eqref{eq:proof-ppipp-decomp},
\eqref{eq:proof-oracle-ppipp-opt-var},
\eqref{eq:proof-ppipp-cross-zero}, and the previously established
identity
$\E[A_nR_n\mid\mathcal D_s]=0$, we obtain
\begin{align*}
\operatorname{MSE}\!\left(
\widehat\mu_{\rm TC\text{-}Cross,\lambda_0}
\mid\mathcal D_s
\right)
=
\frac{\sigma_\epsilon^2}{n}
+
\frac{V_t}{n+N}
+
\lambda_0^2
\E[R_n^2\mid\mathcal D_s].
\end{align*}
Applying \eqref{eq:proof-Rn-second-moment},
\begin{align*}
\operatorname{MSE}\!\left(
\widehat\mu_{\rm TC\text{-}Cross,\lambda_0}
\mid\mathcal D_s
\right)
&\le
\frac{\sigma_\epsilon^2}{n}
+
\frac{V_t}{n+N}
+
\lambda_0^2
\left(
\frac1N+\frac K n
\right)r_n
\\
&=
\frac{\sigma_\epsilon^2}{n}
+
\frac{V_t}{n+N}
+
\frac{N(n+KN)}{n(n+N)^2}\,r_n.
\end{align*}
Finally, since $K\ge1$,
\[
n+KN
\le
K(n+N),
\]
and hence
\[
\frac{N(n+KN)}{n(n+N)^2}
\le
\frac{KN}{n(n+N)}.
\]
Therefore,
\[
\operatorname{MSE}\!\left(
\widehat\mu_{\rm TC\text{-}Cross,\lambda_0}
\mid\mathcal D_s
\right)
\le
\frac{\sigma_\epsilon^2}{n}
+
\frac{V_t}{n+N}
+
K\frac{N}{n(n+N)}r_n,
\]
which proves \eqref{eq:tc-cross-ppipp-mse-upper} and completes the
proof.

\begin{comment}
The oracle procedure's gain over source PPI++ depends on the additional explanatory power gained by moving from $f_s$ to $f_t$. Under sparse linear calibration, a sufficient condition for the calibration error to be asymptotically negligible relative to this gain is
\begin{equation}\label{eq:feasible-tc-ppi-plus-plus-condition}
\frac{s\log p}{n_{\rm tr}}
\ll
V_t
-
\frac{\Cov_t\{Y,f_s(X)\}^2}{V_s}.
\end{equation}

\end{comment}

\end{proof}

\subsection{Minimax Lower Bound.}
\label{app:minimax}
Recall the settings:
Define the class ${\cal G}_s(V_s,V_\delta,\sigma_\epsilon^2)$ as follows.
For parameters
\[
\theta\in\R,\qquad
\mu_s\in\R,\qquad
\mu\in\R^p,\qquad
\delta\in\R^p,
\]
with
\begin{equation}\label{eq:sparse-shell}
\norm{\delta}_0\le s,
\qquad
\norm{\delta}_2^2=V_\delta,
\end{equation}
let
\begin{align}
S &= \mu_s+\xi,
& \xi&\sim N(0,V_s),\label{eq:Smodel}\\
X &= \mu+W,
& W&\sim N_p(0,I_p),\label{eq:Zmodel}\\
Y &= \theta+\xi+W^\top\delta+\varepsilon,
& \varepsilon&\sim N(0,\sigma_\epsilon^2),\label{eq:Ymodel}
\end{align}
where $\xi,W,\varepsilon$ are mutually independent. Thus $\mu_s=\E(S)$. The target regression on the observed features is
\begin{equation}\label{eq:target-regression-canonical}
f_t(s,x)
\coloneqq 
\E_t[Y\mid S=s,X=x]
=
\theta+(s-\mu_s)+(x-\mu)^\top\delta.
\end{equation}
Using the observed score $S$ itself as the available source prediction $f_s=S$, the transfer residual is
\begin{align}
    f_t(S,X)-S
    &=
    \theta-\mu_s-\mu^\top\delta+X^\top\delta\nonumber.
    \label{eq:canonical-affine-correction}
\end{align}
Moreover,
\[
\E[Y]=\theta,
\qquad
\Var\{f_t(S,X)\}
=V_s+V_\delta=:V_t,
\qquad
\Var(Y\mid S,X)=\sigma_\epsilon^2.
\]

With $n$ iid samples $\xi_i,\tilde\xi_i$ from $N(0,V_s)$, we observe $n$ independent labeled draws
\[
\{(S_i,X_i,Y_i):i=1,\ldots,n\}
\]
and, independently, $N$ unlabeled target draws
\[
\{(\widetilde S_j,\widetilde X_j):j=1,\ldots,N\}.
\]
Write
\[
M\coloneqq n+N.
\]
Let $\mathfrak T_{n,N}$ denote the class of \emph{all measurable estimators} of $\theta$ based on these observations.

Define
\begin{equation}\label{eq:psi-def}
\psi_{n,p,s}(V_\delta,\sigma_\epsilon^2)
\coloneqq 
\min\left\{
V_\delta,
\frac{\sigma_\epsilon^2(s-1)\log\!\left(\frac{e(p-1)}{s-1}\right)}{n-1}
\right\}.
\end{equation}

We first separate the information about the unknown means from the information about the sparse slope. Let
\[
\bar X_L=\frac1n\sum_{i=1}^n X_i,
\qquad
\bar X_U=\frac1N\sum_{j=1}^N\widetilde X_j,
\qquad
\bar X_M=\frac{n\bar X_L+N\bar X_U}{M},
\]
and define the labeled--unlabeled contrast
\[
D_X\coloneqq \bar X_L-\bar X_U.
\]
Then
\begin{equation}\label{eq:mean-contrast-dist}
\bar X_M\sim N_p\left(\mu,\frac1M I_p\right),
\qquad
D_X\sim N_p\left(0,\left(\frac1n+\frac1N\right)I_p\right),
\end{equation}
and $\bar X_M$ and $D_X$ are independent. The analogous statements hold for $S$.

Let $Q\in\R^{n\times(n-1)}$ have orthonormal columns satisfying
\[
Q^\top Q=I_{n-1},
\qquad
Q^\top\mathbf 1_n=0.
\]
Define the labeled residual vector
\[
U_L=(Y_1-S_1,\ldots,Y_n-S_n)^\top.
\]
From \eqref{eq:Ymodel},
\[
Y_i-S_i
=
\theta-\mu_s+(X_i-\mu)^\top\delta+\varepsilon_i.
\]
Hence
\begin{equation}\label{eq:centered-regression}
Q^\top U_L
=
(Q^\top X_L)\delta+Q^\top\varepsilon.
\end{equation}
Because the rows of $X_L$ are Gaussian with covariance $I_p$, the transformed design $Q^\top X_L$ has $n-1$ independent $N_p(0,I_p)$ rows, and $Q^\top\varepsilon\sim N_{n-1}(0,\sigma_\epsilon^2 I_{n-1})$. Moreover, the centered regression experiment \eqref{eq:centered-regression} is independent of the sample means and the labeled--unlabeled mean contrasts. Thus the unknown transfer slope $\delta$ is learned from an ordinary sparse Gaussian linear-regression experiment with effective sample size $n-1$.

The next lemma is the main bridge between sparse-regression difficulty and target-mean difficulty.

\begin{lemma}[Bayes decomposition]
\label{lem:bayes-decomp}
Let $\pi$ be any prior supported on sparse vectors satisfying $\norm{\delta}_2^2=V_\delta$ almost surely. Put independent Gaussian priors
\[
\theta\sim N(0,T^2),
\qquad
\mu_s\sim N(0,T^2),
\qquad
\mu\sim N_p(0,T^2I_p),
\]
independently of $\delta\sim\pi$. Let 
\begin{equation}
    R_B(T, \pi) = \inf_{\widehat{\theta}} \mathbb{E}_{\pi,\cal D} \left[ (\widehat{\theta}(\mathcal{D}) - \theta)^2 \right]
\end{equation}
be the Bayes risk for estimating $\theta$ under squared error from the full labeled and unlabeled experiment. Then
\begin{equation}\label{eq:bayes-limit}
\lim_{T\to\infty}R_B(T,\pi)
=
\frac{\sigma_\epsilon^2}{n}
+
\frac{V_s+V_\delta}{M}
+
\frac{N}{nM}\,R_\delta(\pi),
\end{equation}
where
\begin{equation}\label{eq:delta-bayes-risk}
R_\delta(\pi)
\coloneqq 
\inf_{\widehat\delta}
\E_\pi\norm{\widehat\delta-\delta}_2^2
\end{equation}
is the Bayes risk for estimating $\delta$ from the centered regression experiment \eqref{eq:centered-regression}.
\end{lemma}

\begin{proof}
By definition of the Bayes risk under squared error loss,
\begin{align*}
R_B(T, \pi) 
&= \inf_{\widehat{\theta}} \mathbb{E}_{\mathcal{D}} \left[ \mathbb{E}_{\theta \mid \mathcal{D}} \left[ (\widehat{\theta}(\mathcal{D}) - \theta)^2 \mid \mathcal{D} \right] \right].
\end{align*}
Because the outer expectation over $\mathcal{D}$ does not constrain the pointwise choice of estimate for each realized sample $\mathcal{D}$, the infimum passes inside the expectation:
\begin{equation*}
R_B(T, \pi) = \mathbb{E}_{\mathcal{D}} \left[ \inf_{c \in \mathbb{R}} \mathbb{E}_{\theta \mid \mathcal{D}} \left[ (c - \theta)^2 \mid \mathcal{D} \right] \right].
\end{equation*}
For any random variable $\theta \mid \mathcal{D}$ with finite second moment, the scalar constant $c$ minimizing the mean squared error is the conditional posterior mean:
\begin{equation*}
c^* = \mathbb{E}[\theta \mid \mathcal{D}].
\end{equation*}
Substituting this minimizer $\widehat{\theta}_{\mathrm{Bayes}}(\mathcal{D}) = \mathbb{E}[\theta \mid \mathcal{D}]$ into the inner expectation yields the conditional posterior variance:
\begin{equation*}
\inf_{c \in \mathbb{R}} \mathbb{E}_{\theta \mid \mathcal{D}} \left[ (c - \theta)^2 \mid \mathcal{D} \right] 
= \mathbb{E}_{\theta \mid \mathcal{D}} \left[ (\theta - \mathbb{E}[\theta \mid \mathcal{D}])^2 \mid \mathcal{D} \right] 
= \operatorname{Var}(\theta \mid \mathcal{D}).
\end{equation*}
Taking the remaining expectation over the marginal distribution of $\mathcal{D}$ gives:
\begin{equation}
R_B(T, \pi) = \mathbb{E}_{\mathcal{D}} \left[ \operatorname{Var}(\theta \mid \mathcal{D}) \right].
\end{equation}

With this observation, we now give the diffuse-prior calculation, which is the limit of the proper Gaussian-prior calculation as $T\to\infty$. Under a flat prior for $\mu_s$ and $\mu$, Gaussian mean--residual independence gives
\[
\mu_s\mid\mathcal D
\sim
N\left(\bar S_M,\frac{V_s}{M}\right),
\qquad
\mu\mid\mathcal D
\sim
N_p\left(\bar X_M,\frac1M I_p\right).
\]
The posterior for $\delta$ depends only on the centered regression data in \eqref{eq:centered-regression} and is independent of these mean posteriors and of $D_X$. Write
\[
m_\delta\coloneqq \E[\delta\mid\mathcal D_c],
\qquad
C_\delta\coloneqq \Var(\delta\mid\mathcal D_c),
\]
where $\mathcal D_c$ denotes the centered regression data.

Now let
\[
\bar S_L=\frac1n\sum_{i=1}^n S_i,
\quad \bar Y_L=\frac1n\sum_{i=1}^n Y_i.
\]
The labeled mean satisfies
\[
\overline Y_L
=
\theta
+
(\overline S_L-\mu_s)
+
(\overline X_L-\mu)^\top\delta
+
\overline\varepsilon,
\qquad
\overline\varepsilon
\sim
N\!\left(0,\frac{\sigma_\epsilon^2}{n}\right).
\]
Under the diffuse prior for $\theta$, conditional on
$(\mu_s,\mu,\delta)$ and the observed data,
\[
\theta
\mid
\mathcal D,\mu_s,\mu,\delta
\sim
N\!\left(
m_\theta(\mu_s,\mu,\delta),
\frac{\sigma_\epsilon^2}{n}
\right),
\]
where
\[
m_\theta(\mu_s,\mu,\delta)
=
\overline Y_L
-
(\overline S_L-\mu_s)
-
(\overline X_L-\mu)^\top\delta.
\]
Therefore, by the law of total variance,
\begin{align}
\Var(\theta\mid\mathcal D)
={}&
\E\!\left[
\Var(
\theta
\mid
\mathcal D,\mu_s,\mu,\delta
)
\mid\mathcal D
\right]
\nonumber\\
&+
\Var\!\left(
\E[
\theta
\mid
\mathcal D,\mu_s,\mu,\delta
]
\mid\mathcal D
\right).
\label{eq:total-var-theta}
\end{align}
The first term in \eqref{eq:total-var-theta} is simply
\[
\frac{\sigma_\epsilon^2}{n}.
\]

To evaluate the second term, write
\[
A_s\coloneqq \mu_s-\overline S_M,
\qquad
B_\mu\coloneqq \mu-\overline X_M,
\qquad
a_N\coloneqq \frac{N}{M}.
\]
Since
\[
\overline X_L-\mu
=
\frac{N}{M}D_X-(\mu-\overline X_M)
=
a_ND_X-B_\mu,
\]
the conditional posterior mean can be written as
\[
m_\theta(\mu_s,\mu,\delta)
=
C_{\mathcal D}
+
A_s
-
a_ND_X^\top\delta
+
B_\mu^\top\delta,
\]
where $C_{\mathcal D}$ depends only on the observed data and therefore does
not contribute to the conditional variance.

Under the diffuse-prior posterior,
\[
A_s\mid\mathcal D
\sim
N\!\left(0,\frac{V_s}{M}\right),
\qquad
B_\mu\mid\mathcal D
\sim
N_p\!\left(0,\frac1M I_p\right),
\]
and these variables are independent of each other and of
$\delta\mid\mathcal D_c$. Moreover, $D_X$ is fixed when conditioning on
$\mathcal D$. Hence all cross-covariance terms vanish and
\begin{align*}
\Var\!\left(
m_\theta(\mu_s,\mu,\delta)
\mid\mathcal D
\right)
={}&
\frac{V_s}{M}
+
\left(\frac{N}{M}\right)^2
D_X^\top
C_\delta
D_X\\
&+
\Var\!\left(
B_\mu^\top\delta
\mid\mathcal D
\right),
\end{align*}
where
\[
C_\delta
=
\Var(\delta\mid\mathcal D_c).
\]
Finally, using the conditional independence of $B_\mu$ and $\delta$,
\begin{align*}
\Var\!\left(
B_\mu^\top\delta
\mid\mathcal D
\right)
&=
\E\!\left[
(B_\mu^\top\delta)^2
\mid\mathcal D
\right]\\
&=
\E\!\left[
\delta^\top
\E(B_\mu B_\mu^\top\mid\mathcal D)
\delta
\mid\mathcal D
\right]\\
&=
\frac1M
\E\!\left[
\|\delta\|_2^2
\mid\mathcal D_c
\right].
\end{align*}
Combining these terms gives
\begin{align}
\Var(\theta\mid\mathcal D)
={}&
\frac{\sigma_\epsilon^2}{n}
+
\frac{V_s}{M}
+
\frac1M
\E\!\left[
\|\delta\|_2^2
\mid\mathcal D_c
\right]
\nonumber\\
&+
\left(\frac{N}{M}\right)^2
D_X^\top C_\delta D_X.
\label{eq:posterior-var-theta}
\end{align}

By assumption, $\norm{\delta}_2^2=V_\delta$ almost surely under $\pi$. Also $D_X$ is independent of $\mathcal D_c$ and
\[
\E[D_XD_X^\top]
=
\left(\frac1n+\frac1N\right)I_p
=
\frac{M}{nN}I_p.
\]
Averaging \eqref{eq:posterior-var-theta} therefore gives
\begin{align*}
\E\Var(\theta\mid\mathcal D)
&=
\frac{\sigma_\epsilon^2}{n}
+
\frac{V_s+V_\delta}{M}
+
\left(\frac{N}{M}\right)^2
\frac{M}{nN}\,
\E\Tr(C_\delta)\\
&=
\frac{\sigma_\epsilon^2}{n}
+
\frac{V_s+V_\delta}{M}
+
\frac{N}{nM}\E\Tr(C_\delta).
\end{align*}
For squared-error loss, the posterior mean is Bayes and its Bayes risk is the expected posterior variance. Likewise,
\[
\E\Tr(C_\delta)
=
\inf_{\widehat\delta}
\E_\pi\norm{\widehat\delta-\delta}_2^2
=R_\delta(\pi).
\]
The same identity is obtained as the limit of the proper Gaussian-prior Bayes risks as $T\to\infty$, proving \eqref{eq:bayes-limit}.
\end{proof}

We next construct a finite prior on an exact $V_\delta$ shell whose Bayes risk has the sparse-regression minimax order.

\begin{lemma}[Sparse-regression Bayes lower bound]
\label{lem:sparse-bayes}
Suppose $s\ge1$. Under the conditions of Theorem~\ref{thm:minimax-main}, there exists a prior $\pi_*$ supported on
\[
\left\{\delta:\norm{\delta}_0\le s,\ \norm{\delta}_2^2=V_\delta\right\}
\]
and a universal constant $c_0>0$ such that the Bayes risk in the centered Gaussian regression experiment \eqref{eq:centered-regression} satisfies
\begin{equation}\label{eq:sparse-bayes-lower}
R_\delta(\pi_*)
\ge
c_0\,
\psi_{n,p,s}(V_\delta,\sigma_\epsilon^2).
\end{equation}
\end{lemma}

\begin{proof}
Set
\[
k=s-1,
\qquad
q=p-1.
\]
Since $s\le p/2$, we have $q\ge 2k$. A constant-weight Varshamov--Gilbert construction yields a collection
\[
\mathcal V\subset\{v\in\{0,1\}^{q}:\norm v_0=k\}
\]
with
\begin{equation}\label{eq:VG-size}
\log|\mathcal V|
\ge c_1 k\log\left(\frac{eq}{k}\right)
\end{equation}
and pairwise Hamming distance
\begin{equation}\label{eq:VG-distance}
d_H(v,v')\ge c_2 k,
\qquad v\neq v'.
\end{equation}
Here and below $c_1,c_2>0$ are universal constants.

Choose a sufficiently small universal $c_3>0$ and set
\begin{equation}\label{eq:rstar}
r_*
\coloneqq 
c_3\min\left\{
V_\delta,
\frac{\sigma_\epsilon^2 k\log(eq/k)}{n-1}
\right\}.
\end{equation}
Let
\[
a^2=\frac{r_*}{k},
\qquad
b^2=V_\delta-r_*.
\]
For each $v\in\mathcal V$, embed $v$ in coordinates $2,\ldots,p$ and define
\begin{equation}\label{eq:packing-delta}
\delta_v
\coloneqq 
be_1+a(0,v^\top)^\top.
\end{equation}
Then $\norm{\delta_v}_0\le s$ and $\norm{\delta_v}_2^2=V_\delta$ for every $v$. Moreover, $d_H(v,v')\le 2k$ by definition and
\begin{equation}\label{eq:packing-separation}
\norm{\delta_v-\delta_{v'}}_2^2
=a^2d_H(v,v')
\ge c_2 r_*
\qquad(v\neq v').
\end{equation}

Let $P_v$ denote the law of the centered regression experiment under $\delta_v$. Since the Gaussian design distribution is identical under all parameters, the KL divergence is
\begin{align}
D_{\rm KL}(P_v\|P_{v'})
&=
\frac{n-1}{2\sigma_\epsilon^2}
\norm{\delta_v-\delta_{v'}}_2^2\nonumber\\
&\le
\frac{n-1}{\sigma_\epsilon^2}\,r_*
\le
c_3 k\log(eq/k).
\label{eq:KL-packing}
\end{align}
Now, for any estimator $\widehat\delta$, introduce a test $\tilde v:={\arg\min}_{v\in\cal V}\|\widehat\delta-\delta_v\|_2$.
Let $\pi_*$ be a uniform prior on $\delta_1,...\delta_V$.
By \eqref{eq:packing-separation} and triangle inequality, it holds that
\[
\|\widehat\delta-\delta_v\|_2^2\ge\frac{c_2r_*}{4}\mathbb{I}(\tilde v\ne v)\text{ hence }R_\delta(\pi_*)\ge c_2r_*/4\inf_{\widehat\delta}P(\tilde v\ne v).
\]
Choosing $c_3\le c_1/4$ makes the average KL divergence at most a fixed fraction of $\log|\mathcal V|$. i.e.,
\[
\frac{1}{\vert{}{\cal V}\vert{}^2} \sum_{v, v'} D_{\text{KL}}(P_v \Vert{} P_{v'}) \le \alpha \log \vert{}{\cal V}\vert{}\text{ for some }\alpha \in (0, 1/4].
\]
Fano's inequality gives
\[
P(\tilde v \ne v)\ge 1-\frac{\frac1{|{\cal V}|^2}\sum_{v,v'}D_{\rm KL}(P_v\Vert P_{v'})+\log2}{\log|\cal V|}\ge1-\alpha-\frac{\log2}{\log|{\cal V}|}>1/2
\]
hence
\[
R_\delta(\pi_*)\ge c_4 r_*
\]
for a universal $c_4>0$. Combining this with \eqref{eq:rstar} proves \eqref{eq:sparse-bayes-lower}.
\end{proof}

\begin{proof}[Proof of Theorem~\ref{thm:minimax-main}]
For any proper prior supported on the parameter class, minimax risk is at least Bayes risk. Use the prior $\pi_*$ from Lemma~\ref{lem:sparse-bayes} for $\delta$, independent $N(0,T^2)$ priors for $\theta$ and $\mu_s$, and an independent $N_p(0,T^2I_p)$ prior for $\mu$. Then
\[
\inf_{\widehat\theta\in\mathfrak T_{n,N}}
\sup_{P\in\gG_s}
\E_P(\widehat\theta-\theta)^2
\ge R_B(T,\pi_*)
\]
for every $T<\infty$. Letting $T\to\infty$ and applying Lemmas~\ref{lem:bayes-decomp} and \ref{lem:sparse-bayes} gives
\begin{align*}
\inf_{\widehat\theta\in\mathfrak T_{n,N}}
\sup_{P\in\gG_s}
\E_P(\widehat\theta-\theta)^2
&\ge
\frac{\sigma_\epsilon^2}{n}
+
\frac{V_s+V_\delta}{M}
+
\frac{N}{nM}c_0\psi_{n,p,s}(V_\delta,\sigma_\epsilon^2),
\end{align*}
which is \eqref{eq:minimax-main} because $V_t=V_s+V_\delta$ and $M=n+N$.
\end{proof}

\section{Additional theoretical details}
\label{app:theory}
\begin{assumption}\label{ass:model}
Conditionally on the source-trained predictor $f_s$, the variables $Y$, $f_t(X)$, and $\Delta^*(X)$ are square-integrable.
Further, $\E_t[\widehat f_t(X)^2\mid{\cal D}_s,{\cal D}_{t,1}^{\rm gold}]<\infty$ almost surely.
\end{assumption}

\subsection{Split TC-PPI}

We first describe our procedure using an inefficient sample splitting approach which we will later improve through cross-fitting. We split the target gold sample into two disjoint sets ${\cal I}_1,{\cal I}_2\subseteq[n]$ such that $|\mathcal I_1|=n_1$ and $|\mathcal I_2|=n_2$ with $n_1+n_2=n$. The first split is used for transfer calibration, and the second split is used only for PPI rectification. We denote the corresponding data splits by ${\cal D}_{t,j}^{\rm gold}\coloneqq \{(X_i,Y_i)_{i\in\mathcal I_j}\}$ for $j=1,2$.

On the calibration split ${\cal D}_{t,1}^{\rm gold}$, we apply a transfer-calibration algorithm appropriate for the type of data and discrepancy model, to the source predictor and target gold data points in $\mathcal I_1$, and obtain,
$\widehat\Delta
=
\mathcal A_{\rm cal}\!\left(
 f_s,{\cal D}_{t,1}^{\rm gold}
\right).$
Using this we form the target-adapted predictor
\begin{equation}
    \widehat f_t=f_s+\widehat\Delta
    \label{eq:fhat}
\end{equation}
% $\widehat f_t=f_s+\widehat\Delta$. as in \eqref{eq:fhat}.  

%\subsection{TC-PPI estimator for the target mean}

%We consider the mean estimation problem as an illustration of our approach. Generalizations to M estimators are given in the Appendix. For the mean estimation problem, the target parameter is $\theta_t^*\coloneqq \E_t[Y]$. 
Once $\widehat f_t$ is obtained on the calibration split, using only the independent rectification split $\mathcal I_2$ and the unlabeled target features, define
\begin{equation}\label{eq:tcppi}
\widehat\theta_{\mathrm{TC}}
=
\frac{1}{N}\sum_{j=1}^N \widehat f_t(\widetilde X_j)
+
\frac{1}{n_2}\sum_{i\in\mathcal I_2} \bigl\{Y_i-\widehat f_t(X_i)\bigr\}.
\end{equation}
This is the same as the PPI mean estimator, but with the calibrated target predictor $\widehat f_t$ in place of the original source predictor $f_s$. The following proposition shows the conditional unbiasedness of the estimator and an expression for conditional variance. All proofs are given in Appendix~\ref{app:proofs}.

\begin{proposition}\label{prop:unbiased}
Under Assumption~\ref{ass:model}, $
\E\bigl[\widehat\theta_{\mathrm{TC}}\mid \mathcal D_s,{\cal D}_{t,1}^{\rm gold}\bigr] = \theta_t^*,$
and consequently, unconditionally, $\E[\widehat\theta_{\mathrm{TC}}]=\theta_t^*$. Further, with the notation $V_t\coloneqq \Var_t(\cdot)$,
\begin{equation*}
%\label{eq:condvar}
\Var\bigl(\widehat\theta_{\mathrm{TC}}\mid \mathcal D_s,{\cal D}_{t,1}^{\rm gold}\bigr)
=
\frac{V_t\bigl(\widehat f_t(X)\mid \mathcal D_s,{\cal D}_{t,1}^{\rm gold}\bigr)}{N}
+
\frac{V_t(\varepsilon)
+V_t\bigl\{\widehat\Delta(X)-\Delta^*(X)\mid\mathcal D_s,{\cal D}_{t,1}^{\rm gold}\bigr\}}{n_2}.
\end{equation*}
%Moreover, $\Var_t\bigl(Y-\widehat f_t(X)\mid \mathcal D_s,{\cal D}_{t,1}^{\rm gold}\bigr)=$
\end{proposition}

\begin{comment}
\begin{remark}[Comparison with vanilla PPI]
The vanilla PPI estimator based on $f_s$ has target residual
\[
Y-f_s(X)=\varepsilon+\Delta^*(X),
\]
so, by $\E_t[\varepsilon\mid X]=0$,
\[
\Var_t\bigl(Y-f_s(X)\bigr)
=
\Var_t(\varepsilon)+\Var_t\{\Delta^*(X)\}.
\]
\end{remark}
\end{comment}

%\end{remark}

%\subsection{Asymptotic normality and efficiency expansion}

We state the asymptotic normality of the TC-PPI estimator and its limiting variance without specifying how the transfer correction is learned.  Define the conditional target prediction error
\begin{equation}\label{eq:generic-cal-error}
b_{n_1}
\coloneqq 
\E_t\!\left[
\{\widehat\Delta(X)-\Delta^*(X)\}^2
\mid \mathcal D_s,{\cal D}_{t,1}^{\rm gold}
\right].
\end{equation}
The next theorem states the limiting distribution of the TC-PPI estimator by imposing a condition on the accuracy of the target correction, $b_{n_1}$.

\begin{theorem}\label{thm:clt}
Assume Assumption~\ref{ass:model}. Suppose $Y$, $f_t(X)$, and $\widehat f_t(X)$ have conditional moments of order $2+\eta$ exist uniformly with probability tending to one, for some $\eta>0$, and
\begin{equation}\label{eq:ratecond}
b_{n_1}=o_p(1),
\qquad
\frac{n_2}{N}\to \rho_1\in[0,\infty).
\end{equation}
If the limiting variance below is positive, then, conditionally on $\mathcal D_s$ and the calibration split,
\begin{equation}\label{eq:condclt}
\frac{\sqrt{n_2}\,(\widehat\theta_{\mathrm{TC}}-\theta_t^*)}{\sqrt{V_{n_1}}}
\overset{D}{\to}N(0,1),
\end{equation}
where
\begin{equation}\label{eq:Vexpand}
V_{n_1}
=
\frac{n_2}{N}\Var_t\{f_t(X)\}+\Var_t(\varepsilon)
+O_p\!\left(\sqrt{b_{n_1}}+b_{n_1}\right).
\end{equation}
Consequently,
%\begin{equation}\label{eq:uncondclt}
$\sqrt{n_2}\,(\widehat\theta_{\mathrm{TC}}-\theta_t^*)
\overset{D}{\to}
N\!\left(0,\rho_1\Var_t\{f_t(X)\}+\Var_t(\varepsilon)\right).$
%\end{equation}
In the regime $N\gg n_2$, so that $\rho_1=0$, the asymptotic variance simplifies to $\Var_t(\varepsilon)$.
\end{theorem}

\paragraph{Sparse linear calibration: high-dimensional lasso rates}

The preceding construction is agnostic to the form of the transfer correction. To obtain an explicit rate, we now specialize to a sparse linear correction as follows,
\begin{equation}\label{eq:sparse-linear-calibration}
\Delta^*(x)=x^\top\delta^*,\quad\delta^*\in\R^p.
\end{equation}
We assume the features are high-dimensional, i.e., $p \gg n$, and the vector $\delta^*$ is $s$-sparse, i.e.,  $\|\delta^*\|_0\le s$ for some $s=s(n_1)>0$. 
%On the calibration split, define the pseudo-response $U_i=Y_i-f_s(X_i),\quad i\in\mathcal I_1.$ Then $\E_t[U_i\mid X_i]=X_i^\top\delta^*.$
We estimate $\delta^*$ by the lasso on the calibration split:
\begin{equation}\label{eq:lasso-rig}
\hat\delta
\in
\arg\min_{\delta\in\mathbb R^p}
\left\{
\frac{1}{2n_1}\sum_{i\in\mathcal I_1}(Y_i-f_s(X_i)-X_i^\top\delta)^2
+\lambda \|\delta\|_1
\right\}.
\end{equation}

The following corollary specializes Theorem \ref{thm:clt} to the case of sparse linear calibration estimated through lasso.
\begin{corollary}[Sparse linear calibration rate]\label{cor:split-lasso-clt}
Under the sparse linear calibration model \eqref{eq:sparse-linear-calibration} and the assumptions of Proposition~\ref{thm:lasso-calibration-rig} in the Appendix, we have, $
b_{n_1}
=
O_p\!\left(\frac{s\log p}{n_1}\right).$
Consequently, if $s\log p=o(n_1)$, Theorem~\ref{thm:clt} applies and
\[
V_{n_1}
=
\frac{n_2}{N}\Var_t\{f_t(X)\}+\Var_t(\varepsilon)
+O_p\!\left(\sqrt{\frac{s\log p}{n_1}}\right).
\]
\end{corollary}
Thus sparse linear calibration is one of the approaches that satisfies the generic prediction-consistency condition \eqref{eq:ratecond}. Our results therefore separate the role of the transfer correction learner from the validity mechanism of PPI. For a general fine-tuning procedure, it is enough for the learned correction to satisfy $b_{n_1}=o_p(1)$ in order for split TC-PPI to attain the desired asymptotic variance. In the sparse linear specialization, Proposition\ref{thm:lasso-calibration-rig} yields $b_{n_1}=O_p(s\log p/n_1)$, providing a necessary condition on the size of $n_1$. 
%The source model pays the nonvanishing drift variance $\Var_t\{\Delta^*(X)\}$ in the rectification term, whereas TC-PPI pays the decreasing calibration error. The joint estimator introduced later is even more robust: its oracle no-harm guarantee does not require $b_{n_1}\to0$ at all, because it can assign zero weight to an unhelpful transfer direction.

We note, in the target domain, $
Y-f_s(X)=\varepsilon+\Delta^*(X).$
Since $\Delta^*(X)$ is measurable in $X$ and $\E_t[\varepsilon\mid X]=0$, we have 
$\Cov_t\{\varepsilon,\Delta^*(X)\}=0$.
Hence $
\Var_t\{Y-f_s(X)\}
=
\Var_t(\varepsilon)+\Var_t\{\Delta^*(X)\}.$
Therefore, under the assumptions of Theorem~\ref{thm:clt}, if the \textit{same $n_2$ sized rectification sample} is used, the vanilla PPI estimator based on the unadapted source predictor $f_s$ has asymptotic variance
$
V_{\mathrm{PPI}}
=\rho_1\Var_t\{f_s(X)\}
+\Var_t(\varepsilon)
+\Var_t\{\Delta^*(X)\}.$
In comparison, the transfer-calibrated PPI estimator has asymptotic variance $
V_{\mathrm{TC}}
=\rho_1\Var_t\{f_t(X)\}+\Var_t(\varepsilon)+O_p\!\left(\sqrt{b_{n_1}}+b_{n_1}\right).$
Transfer calibration replaces the variance of the source-to-target discrepancy function $\Var_t\{\Delta^*(X)\}$ by the variance of the calibration error $\widehat\Delta(X)-\Delta^*(X)$. Therefore, the rectification component improves whenever the learned correction is sufficiently more accurate (and asymptotically vanishes) than leaving the source-to-target discrepancy uncorrected. Under sparse linear calibration and mean-zero $X$, $\Var_t\{\Delta^*(X)\}=(\delta^*)^\top\Sigma_t\delta^*$. Therefore as $n_1$ increases, since $b_{n_1}=O_p\!\left(\frac{s\log p}{n_1}\right)$, the TC-PPI has a variance that is lower by $(\delta^*)^\top\Sigma_t\delta^*$ than the vanilla PPI with high probability, provided the rectification sample is of the same size.

\begin{proposition}[Unbiasedness for fixed joint weights]
\label{prop:joint-cross-unbiased}
Suppose the folds are fixed independently of the observations and have equal
size $|I_k|=n/K$. Under Assumption~\ref{ass:model}, for every fixed
$\boldsymbol\lambda\in\mathbb R^2$, we have $
\E\!\left[
\widehat\mu_{\mathrm{JTC\text{-}Cross},\boldsymbol\lambda}
\mid\mathcal D_s
\right]
=
\theta_t^*.$
\end{proposition}

\subsection{Variance inflation with estimated tuning vector}

Theorem \ref{thm:joint-tc-cross} provides a no-harm guarantee with oracle $\boldsymbol{\lambda}$. There is an additional finite-sample cost from estimating the joint power parameter. We analyze this cost next.

Let $\Delta^\dagger$ be the limiting function that the transfer correction converges to in Theorem~\ref{thm:joint-tc-cross}, and recall
\[
H_\dagger(X)
=
\begin{pmatrix}
f_s(X)\\
\Delta^\dagger(X)
\end{pmatrix},
\qquad
\Sigma_\dagger
=
\Var_t\{H_\dagger(X)\},
\qquad
c_\dagger
=
\Cov_t\{H_\dagger(X),Y\}.
\]
For the finite sample ratio $\rho_n=n/N$, define
\begin{equation}\label{eq:joint-finite-pop-criterion}
V_{{\rm J},n}^\dagger(\boldsymbol\lambda)
\coloneqq 
\Var_t(Y)
-
2\boldsymbol\lambda^\top c_\dagger
+
(1+\rho_n)
\boldsymbol\lambda^\top
\Sigma_\dagger
\boldsymbol\lambda.
\end{equation}
For simplicity, we assume $\Sigma_\dagger$ is nonsingular. Then the minimizer of the above equation is
\begin{equation}\label{eq:joint-finite-pop-lambda}
\boldsymbol\lambda_{{\rm J},n}^\dagger
=
\frac1{1+n/N}
\Sigma_\dagger^{-1}c_\dagger.
\end{equation}

\begin{proposition}[Rate of feasible joint tuning]
\label{prop:joint-weight-rate}
Suppose $\lambda_{\min}(\Sigma_\dagger)\ge\kappa_H>0$, the relevant fourth
moments are finite, $K$ is fixed, $r_n=o(1)$, and the fold-specific transfer
errors satisfy
\[
\max_{1\le k\le K}
\E_t\!\left[
\{\widehat\Delta^{(-k)}(X)-\Delta^\dagger(X)\}^2
\mid
\mathcal D_s,\mathcal D_{-k}
\right]
=
O_p(r_n).
\]
Let $\widehat\Sigma_H$ and $\widehat c_H$ be the pooled out-of-fold estimators. Then
\[
\|\widehat\Sigma_H-\Sigma_\dagger\|_{\rm op}
+
\|\widehat c_H-c_\dagger\|_2
=
O_p\!\left(
n^{-1/2}+\sqrt{r_n}
\right).
\]
Consequently, for the ridge estimator
$
\widehat{\boldsymbol\lambda}
=
\frac1{1+n/N}
(\widehat\Sigma_H+\tau I_2)^{-1}
\widehat c_H,$ 
with $\tau=o(1)$,
\begin{equation}\label{eq:joint-weight-rate}
\left\|
\widehat{\boldsymbol\lambda}
-
\boldsymbol\lambda_{{\rm J},n}^\dagger
\right\|_2
=
O_p\!\left(
n^{-1/2}+\sqrt{r_n}+\tau
\right).
\end{equation}
Moreover, because \eqref{eq:joint-finite-pop-criterion} is quadratic,
\begin{align}
&
V_{{\rm J},n}^\dagger(\widehat{\boldsymbol\lambda})
-
V_{{\rm J},n}^\dagger(\boldsymbol\lambda_{{\rm J},n}^\dagger)
=
(1+n/N)
\left(
\widehat{\boldsymbol\lambda}
-
\boldsymbol\lambda_{{\rm J},n}^\dagger
\right)^\top
\Sigma_\dagger
\left(
\widehat{\boldsymbol\lambda}
-
\boldsymbol\lambda_{{\rm J},n}^\dagger
\right)
=
O_p\!\left(
n^{-1}+r_n+\tau^2
\right).
\label{eq:joint-excess-criterion}
\end{align}
\end{proposition}

\noindent\emph{Proof.} See Appendix~\ref{app:proofs}.

The above proposition shows that the excess risk of the Joint TC-Cross-PPI++ estimator with estimated power tuning parameter over the one with oracle power tuning parameter is given by
\[
O_p\!\left(
n^{-2}+\frac{r_n}{n}+\frac{\tau^2}{n}
\right).
\]
For the case of sparse linear calibration with fixed $K$, we have $
r_n
\asymp
\frac{s\log p}{n_{\rm tr}}.$
Therefore, when $\tau$ is negligible,
\[
\left\|
\widehat{\boldsymbol\lambda}
-
\boldsymbol\lambda_{{\rm J},n}^\dagger
\right\|_2
=
O_p\!\left(
n^{-1/2}
+
\sqrt{\frac{s\log p}{n_{\rm tr}}}
\right),
\]
and the excess limiting variance  is
\[
O_p\!\left(
\frac{1}{n}
+
\frac{s\log p}{nn_{\rm tr}}
\right).
\]

\subsection{Variance estimation and confidence intervals}
\label{app:variance-estimation}

We now describe the variance estimators used to construct confidence intervals
for the cross-fitted estimators in the experiments section. The formulas below estimate the 
asymptotic variances in Theorems~\ref{thm:tc-cross-mean},
\ref{thm:tc-cross-ppipp}, and \ref{thm:joint-tc-cross} from the available labeled and unlabeled samples. While our construction of the cross-fitted TC-Cross-PPI is similar to the Cross-PPI of \cite{zrnic2024cross}, we do not employ the same computationally inefficient bootstrap strategy to estimate the limiting variance. In \cite{zrnic2024cross}, the limiting variance is expressed in terms of an ``average'' model $\bar{f}$, which they propose to estimate through bootstrap. However, we show that under our assumptions we can obtain a consistent estimator without requiring bootstrap.

For a scalar sample $a_1,\ldots,a_m$, write the sample variance  as,
\[
\widehat{\Var}_m(a)
:=
\frac{1}{m-1}
\sum_{i=1}^m
(a_i-\overline a)^2,
\qquad
\overline a:=\frac1m\sum_{i=1}^m a_i.
\]
Likewise, for vectors $v_1,\ldots,v_m\in\mathbb R^d$, write the sample variance as
\[
\widehat{\Var}_m(v)
:=
\frac{1}{m-1}
\sum_{i=1}^m
(v_i-\overline v)(v_i-\overline v)^\top.
\]

\subsubsection{Mean estimation}

Recall that, for $i\in I_k$, the fitted transfer-calibrated predictor is
\[
\widehat f_i^{\rm oof}
=
\widehat f_t^{(-k)}(X_i).
\]
For each unlabeled observation, define the fold-averaged prediction
\[
\widetilde f_j^{\rm CF}
:=
\frac1K\sum_{k=1}^K
\widehat f_t^{(-k)}(\widetilde X_j),
\qquad j=1,\ldots,N.
\]
Thus
\[
\frac1N\sum_{j=1}^N\widetilde f_j^{\rm CF}
=
\frac1{KN}
\sum_{k=1}^K\sum_{j=1}^N
\widehat f_t^{(-k)}(\widetilde X_j).
\]

Then, replacing the components of the population variance with their sample versions, we can get an estimator for the variance of TC-Cross-PPI estimator  with a fixed $\lambda$ in the $\sqrt{n}$ scale as follows,
\begin{equation}
\label{eq:tc-cross-estimator-var}
\widehat V_{\rm TC}(\lambda)
=
\widehat{\Var}_n
\{Y_i-\lambda\widehat f_i^{\rm oof}\}
+
\frac{n}{N}\widehat{\Var}_N
\{\lambda\widetilde f_j^{\rm CF}\}
.
\end{equation}
For ordinary TC--Cross-PPI, one simply sets $\lambda=1$. For
TC--Cross-PPI++, we substitute with the estimates $\hat{\lambda}$.

The corresponding standard error is
\[
\widehat{\rm SE}_{\rm TC}
=
\sqrt{
\frac{\widehat V_{\rm TC}(\widehat\lambda)}{n}
},
\]
and the nominal $(1-\alpha)$ Wald interval is
\[
\widehat\mu_{{\rm TC\text{-}Cross},\widehat\lambda}
\pm
z_{1-\alpha/2}
\widehat{\rm se}_{\rm TC}.
\]

\paragraph{Joint TC--Cross-PPI++.}
Just as in the last part, it is useful to define the prediction-powered vector
separately for each unlabeled observation. Recall that, for $i\in I_k$,
\[
\widehat H_i^{\rm oof}
=
\begin{pmatrix}
f_s(X_i)\\
\widehat\Delta^{(-k)}(X_i)
\end{pmatrix}.
\]
For the unlabeled data points $j=1,\ldots,N$, define
\[
\widetilde H_j^{\rm CF}
:=
\begin{pmatrix}
f_s(\widetilde X_j)\\[0.3em]
\displaystyle
\frac1K\sum_{k=1}^K
\widehat\Delta^{(-k)}(\widetilde X_j)
\end{pmatrix}.
\]
Then the vector $\widetilde H_N^{\rm CF}$ used in the main text satisfies
\[
\widetilde H_N^{\rm CF}
=
\frac1N\sum_{j=1}^N
\widetilde H_j^{\rm CF}.
\]

For a fixed or estimated joint power vector
$\boldsymbol\lambda$, the estimated $\sqrt n$-scale variance is
\begin{equation}
\label{eq:joint-cross-var-est}
\widehat V_{\rm JTC}(\boldsymbol\lambda)
=
\widehat{\Var}_n
\left\{
Y_i-
\boldsymbol\lambda^\top
\widehat H_i^{\rm oof}
\right\}
+
\frac nN
\widehat{\Var}_N
\left\{
\boldsymbol\lambda^\top
\widetilde H_j^{\rm CF}
\right\}.
\end{equation}
We use the plug-in version $
\widehat V_{\rm JTC}
:=
\widehat V_{\rm JTC}
(\widehat{\boldsymbol\lambda}),$
where $\widehat{\boldsymbol\lambda}$ is the joint power estimate.
Then the standard error is obtained as
\[
\widehat{\rm SE}_{\rm JTC}
=
\sqrt{
\frac{\widehat V_{\rm JTC}}{n}
},
\]
and the corresponding Wald interval is
\[
\widehat\mu_{
{\rm JTC\text{-}Cross},
\widehat{\boldsymbol\lambda}}
\pm
z_{1-\alpha/2}
\widehat{\rm SE}_{\rm JTC}.
\]

Some intuition behind why the above estimators are reasonable (and indeed consistent) estimators of the limiting variances in the theorems are in order. Although the fold-specific predictors are dependent because their training
samples overlap, this dependence affects only the nuisance-estimation
remainder. Under foldwise consistency and fixed $K$ conditions of the theorems, the
cross-fitted estimator admits an asymptotic linear representation obtained by
replacing every fitted fold predictor by its limiting function $f^\dagger$,
with replacement error $o_p(n^{-1/2})$. The resulting leading terms are an
empirical average over the i.i.d. gold sample and an independent empirical
average over the i.i.d. unlabeled sample. Moreover, the sample variances computed using the out-of-fold predictions differ by $o_p(1)$ from
those hypothetically computed using $f^\dagger$. Hence the labeled and unlabeled sample
variances consistently estimate the two components of the first-order
asymptotic variance.

\begin{lemma}[Consistency of the variance estimators]
\label{lem:mean-var-consistency}
Suppose the conditions of Theorem~\ref{thm:tc-cross-ppipp} hold. If
$
\widehat\lambda-\lambda_\dagger^*=o_p(1),
$
then
\[
\widehat V_{\rm TC}(\widehat\lambda)
\overset{P}{\longrightarrow}
V_{\rm cross}^\dagger(\lambda_\dagger^*).
\]
Likewise, under the conditions of
Theorem~\ref{thm:joint-tc-cross}, if
$
\widehat{\boldsymbol\lambda}
-
\boldsymbol\lambda_{\rm Jcross}^*
=o_p(1),
$
then
\[
\widehat V_{\rm JTC}
(\widehat{\boldsymbol\lambda})
\overset{P}{\longrightarrow}
V_{\rm Jcross}^\dagger
(\boldsymbol\lambda_{\rm Jcross}^*).
\]
Consequently, if the corresponding limiting variance is positive, the
associated Wald intervals have asymptotic coverage $1-\alpha$.
\end{lemma}

\begin{proof}[Proof of Lemma \ref{lem:mean-var-consistency}]
We first consider scalar TC--Cross-PPI++. Recall that for $i\in I_k$,
\[
e_k(X_i)
=
\widehat f_t^{(-k)}(X_i)-f^\dagger(X_i).
\]
For each $k$, conditional on
$\mathcal D_s$ and $\mathcal D_{-k}$, the observations in $I_k$ are
independent of $e_k$. Hence
\[
\E\!\left[
\frac K n\sum_{i\in I_k}e_k(X_i)^2
\middle|
\mathcal D_s,\mathcal D_{-k}
\right]
=
a_{n,k}.
\]
Conditional Markov's inequality and $a_{n,k}=o_p(1)$ therefore give
\[
\frac K n\sum_{i\in I_k}e_k(X_i)^2
=
o_p(1).
\]
Since $K$ is fixed,
\begin{equation}
\label{eq:oof-L2-var-proof}
\frac1n\sum_{i=1}^n
\left\{
\widehat f_i^{\rm oof}
-
f^\dagger(X_i)
\right\}^2
=
o_p(1).
\end{equation}

For the unlabeled sample, define
\[
\overline e_n(X)
:=
\frac1K\sum_{k=1}^K
\left\{
\widehat f_t^{(-k)}(X)-f^\dagger(X)
\right\}.
\]
By Jensen's inequality,
\[
\E_t\{\overline e_n(X)^2\}
\le
\frac1K
\sum_{k=1}^K
\E_t\!\left[
\{\widehat f_t^{(-k)}(X)-f^\dagger(X)\}^2
\mid
\mathcal D_s,\mathcal D_{-k}
\right]
=o_p(1),
\]
by assumptions on foldwise convergence and fixed number of folds.
Conditional on the gold data, the unlabeled observations are i.i.d. and
independent of all fitted fold-specific predictors. The
conditional Markov's inequality therefore gives
\begin{equation}
\label{eq:unlabeled-L2-var-proof}
\frac1N\sum_{j=1}^N
\left\{
\widetilde f_j^{\rm CF}
-
f^\dagger(\widetilde X_j)
\right\}^2
=
o_p(1).
\end{equation}

Now define the oracle labeled and unlabeled influence contributions
\[
\xi_i^{L,*}
=
Y_i-\lambda^*f^\dagger(X_i),
\qquad
\xi_j^{U,*}
=
\lambda^*f^\dagger(\widetilde X_j).
\]
Their feasible counterparts that replaces $\lambda^*$ by $\hat{\lambda}$ and $f^{\dagger}$ by $\widehat f_i^{\rm oof}$ for the labeled samples and $\widetilde f_j^{\rm CF}$ for the unlabeled samples are,
\[
\widehat\xi_i^L
=
Y_i-\widehat\lambda\widehat f_i^{\rm oof},
\qquad
\widehat\xi_j^U
=
\widehat\lambda\widetilde f_j^{\rm CF}.
\]
Since $\widehat\lambda-\lambda^*=o_p(1)$,
$\widehat\lambda=O_p(1)$. Using
\eqref{eq:oof-L2-var-proof},
\begin{align*}
\frac1n\sum_{i=1}^n
(\widehat\xi_i^L-\xi_i^{L,*})^2
&\le
2(\widehat\lambda-\lambda^*)^2
\frac1n\sum_{i=1}^n
f^\dagger(X_i)^2+
2\widehat\lambda^2
\frac1n\sum_{i=1}^n
\{\widehat f_i^{\rm oof}-f^\dagger(X_i)\}^2=
o_p(1),
\end{align*}
because $f^\dagger$ is square-integrable and hence
$n^{-1}\sum_i f^\dagger(X_i)^2=O_p(1)$. Similarly, by
\eqref{eq:unlabeled-L2-var-proof},
\[
\frac1N\sum_{j=1}^N
(\widehat\xi_j^U-\xi_j^{U,*})^2
=
o_p(1).
\]

We use the elementary fact that if
\[
\frac1m\sum_{\ell=1}^m(A_\ell-B_\ell)^2=o_p(1)
\]
and
\[
\frac1m\sum_{\ell=1}^mB_\ell^2=O_p(1),
\]
then
\[
\widehat{\Var}_m(A_\ell)
-
\widehat{\Var}_m(B_\ell)
=
o_p(1).
\]

Applying this fact to the labeled and unlabeled contributions gives
\[
\widehat{\Var}_n
\left\{
Y_i-\widehat\lambda\widehat f_i^{\rm oof}
\right\}
=
\widehat{\Var}_n
\left\{
Y_i-\lambda^*f^\dagger(X_i)
\right\}
+
o_p(1),
\]
and
\[
\widehat{\Var}_N
\left\{
\widehat\lambda\widetilde f_j^{\rm CF}
\right\}
=
\widehat{\Var}_N
\left\{
\lambda^*f^\dagger(\widetilde X_j)
\right\}
+
o_p(1).
\]
The quantities on the right are ordinary empirical variances based on
the i.i.d.\ gold and unlabeled samples since a fixed function $f^{\dagger}$ is applied to them. Therefore, by the law of large
numbers,
\[
\widehat{\Var}_n
\left\{
Y_i-\lambda^*f^\dagger(X_i)
\right\}
\overset{P}{\longrightarrow}
\Var_t\{Y-\lambda^*f^\dagger(X)\},
\]
and
\[
\widehat{\Var}_N
\left\{
\lambda^*f^\dagger(\widetilde X_j)
\right\}
\overset{P}{\longrightarrow}
(\lambda^*)^2\Var_t\{f^\dagger(X)\}.
\]
Since $n/N\to\rho$, it follows that
\[
\widehat V_{\rm TC}(\widehat\lambda)
\overset{P}{\longrightarrow}
\Var_t\{Y-\lambda^*f^\dagger(X)\}
+
\rho(\lambda^*)^2\Var_t\{f^\dagger(X)\},
\]
which is $
V_{\rm cross}^\dagger(\lambda_\dagger^*).$
The joint result follows by the same argument.
Combining these results for consistency of the variance estimators with the asymptotic normality in
Theorems~\ref{thm:tc-cross-ppipp} and
\ref{thm:joint-tc-cross}, and applying Slutsky's theorem, gives the stated
asymptotic coverage of the corresponding Wald intervals.
\end{proof}
\section{Extension to generalized linear models and  M estimation}

\subsection{generalized linear models}
We first consider the problem of estimating the parameter
$\beta_t^* \in \mathbb{R}^d$ in a generalized linear model. The true target parameter can be defined by the solution of the population moment equation as follows,
\[
\E_t[g_{\beta_t^*}(X,Y)] = 0,
\qquad
g_\beta(x,y) = x\{\mu(x^\top \beta) - y\},
\]
where $\mu$ is a continuously differentiable inverse link function. Let $\hat f_t$ be the calibrated predictor constructed using the calibration split ${\cal D}_{t,1}^{\rm gold}$. On the independent rectification split $\mathcal I_2$, define
\[
\Psi_{n_2,N}(\beta;\hat f_t)
=
\frac{1}{n_2} \sum_{i\in\mathcal I_2} g_\beta(X_i,Y_i)
-
\lambda \frac{1}{n_2} \sum_{i\in\mathcal I_2} g_\beta(X_i,\hat f_t(X_i))
+
\lambda \frac{1}{N} \sum_{j=1}^N g_\beta(\tilde X_j,\hat f_t(\tilde X_j)).
\]
This is then a modified PPI sample moment equation.
Define $\hat\beta_{\mathrm{TC},\lambda}$ as any measurable solution of
$\Psi_{n_2,N}(\hat\beta_{\mathrm{TC},\lambda};\hat f_t)=0$.

%\subsection{Asymptotic linearity and asymptotic normality}

Let $f_t$ denote the oracle target-calibrated predictor and let $\rho=\lim n_2/N\in[0,\infty)$.

\begin{theorem}[Conditional asymptotic linearity for GLMs]
\label{thm:glm-linear}
Assume: (i) $d$ is fixed; (ii) $\mu$ is continuously differentiable in a neighborhood of $\beta_t^*$ and
$
\E_t\!\left[\sup_{\|\beta-\beta_t^*\|\le r}\|X\|^2 |\mu'(X^\top \beta)|\right] < \infty
$
for some $r>0$; (iii) $A_t\coloneqq \E_t[\nabla_\beta g_{\beta_t^*}(X,Y)]$ is nonsingular; (iv) the relevant score variables have finite $(2+\eta)$ moments for some $\eta>0$; and (v)
\[
\E_t\!\left[\|X\|^2\{\hat f_t(X)-f_t(X)\}^2\mid\mathcal D_s,{\cal D}_{t,1}^{\rm gold}\right]=o_p(1).
\]
Then, conditionally on $(\mathcal D_s,{\cal D}_{t,1}^{\rm gold})$,
\begin{align*}
  \sqrt{n_2}(\hat\beta_{\mathrm{TC},\lambda}-\beta_t^*)
=
-A_t^{-1}\!\bigg[&
\frac1{\sqrt{n_2}}\sum_{i\in\mathcal I_2}
\{g_{\beta_t^*}(X_i,Y_i)-\lambda g_{\beta_t^*}(X_i,f_t(X_i))\}
\\
& +
\frac{\lambda\sqrt{n_2}}{N}\sum_{j=1}^N g_{\beta_t^*}(\tilde X_j,f_t(\tilde X_j))
\bigg]
+o_p(1).  
\end{align*}

Further,
\[
\sqrt{n_2}(\hat\beta_{\mathrm{TC},\lambda}-\beta_t^*)
\overset{D}{\to}
N\!\left(0,A_t^{-1}V_{\lambda,t}A_t^{-T}\right),
\]
where
\[
V_{\lambda,t}
=
\Var_t\!\Bigl(g_{\beta_t^*}(X,Y)-\lambda g_{\beta_t^*}(X,f_t(X))\Bigr)
+
\lambda^2\rho\,\Var_t\!\Bigl(g_{\beta_t^*}(X,f_t(X))\Bigr).
\]
\end{theorem}

\subsection{General convex \texorpdfstring{$M$}-estimation}

We now formulate TC-PPI for a general low-dimensional convex $M$-estimation problem.
The structure parallels the mean-estimation section: a source-domain predictor is first
calibrated on one target split, and then the calibrated predictor is inserted into a
prediction-powered estimating equation on an independent target split and the unlabeled
target sample.

\subsubsection{Target estimand and rectified estimating equation}

Let $\Theta \subset \mathbb R^d$ be open, with fixed dimension $d$. Suppose the target
parameter $\theta_t^* \in \Theta$ is defined by the moment condition
\begin{equation}\label{eq:M-target}
\E_t[g_{\theta_t^*}(X,Y)] = 0,
\end{equation}
where $g_\theta(x,y) \in \mathbb R^d$ is the estimating function associated with a convex
loss $\ell_\theta(x,y)$. 

Let $\hat f_t$ be a calibrated predictor learned on the calibration split ${\cal D}_{t,1}^{\rm gold}$,
independent of the rectification split $\mathcal I_2$ and of the unlabeled target sample
$\{\tilde X_j\}_{j=1}^N$.

For a fixed scalar $\lambda \in \mathbb R$, define the TC-PPI estimating equation
\begin{equation}\label{eq:M-Psi}
\Psi_{n_2,N}(\theta;\hat f_t)
=
\frac{1}{n_2}\sum_{i\in \mathcal I_2} g_\theta(X_i,Y_i)
-
\lambda \frac{1}{n_2}\sum_{i\in \mathcal I_2} g_\theta(X_i,\hat f_t(X_i))
+
\lambda \frac{1}{N}\sum_{j=1}^N g_\theta(\tilde X_j,\hat f_t(\tilde X_j)).
\end{equation}
Our proposed TC-PPI estimator is then any measurable solution
\begin{equation}\label{eq:M-estimator}
\hat\theta_{\mathrm{TC},\lambda}
\in
\{\theta \in \Theta : \Psi_{n_2,N}(\theta;\hat f_t)=0\}.
\end{equation}

%\subsubsection{Oracle target predictor and notation}

We continue to use the same notations for the source and target predictors and the discrepancy function. Define the population rectified moment map for any predictor $f$ as
\begin{equation}\label{eq:M-pop-psi}
\Psi_t(\theta;f)
\coloneqq 
\E_t[g_\theta(X,Y)]
-\lambda \E_t[g_\theta(X,f(X))]
+\lambda \E_t[g_\theta(\tilde X,f(\tilde X))].
\end{equation}
Since $X$ and $\tilde X$ have the same target covariate law,
\[
\E_t[g_\theta(X,f(X))] = \E_t[g_\theta(\tilde X,f(\tilde X))],
\]
and therefore
\begin{equation}\label{eq:M-pop-unbiased}
\Psi_t(\theta;f)=\E_t[g_\theta(X,Y)].
\end{equation}
In particular, $\Psi_t(\theta_t^*;f)=0$ for every measurable predictor $f$ for which the
expectations exist.

\subsection{Assumptions}

\begin{assumption}[Identification and local smoothness]
\label{ass:M1}
The parameter $\theta_t^*$ is the unique solution of \eqref{eq:M-target} in a neighborhood
$\mathcal N$ of $\theta_t^*$. For almost every $(x,y)$, the map
$\theta \mapsto g_\theta(x,y)$ is continuously differentiable on $\mathcal N$.
\end{assumption}

\begin{assumption}[Nonsingular Jacobian]
\label{ass:M2}
The matrix
\[
A_t \coloneqq  \E_t[\nabla_\theta g_{\theta_t^*}(X,Y)]
\]
exists and is nonsingular.
\end{assumption}

\begin{assumption}[Moment envelope]
\label{ass:M3}
There exists $\eta>0$ and an integrable envelope $G(X,Y)$ such that
\[
\sup_{\theta \in \mathcal N}\|g_\theta(X,Y)\| \le G(X,Y),
\qquad
\E_t[G(X,Y)^{2+\eta}]<\infty.
\]
Also,
\[
\E_t\Big[\sup_{\theta\in \mathcal N}\|\nabla_\theta g_\theta(X,Y)\|\Big] < \infty.
\]
Analogous moment conditions hold with $Y$ replaced by $f_t(X)$ and by $\hat f_t(X)$.
\end{assumption}

\begin{assumption}[Calibration consistency in score norm]
\label{ass:M4}
Conditionally on $\mathcal D_s$ and the calibration split ${\cal D}_{t,1}^{\rm gold}$,
\[
\E_t\!\left[\|g_{\theta_t^*}(X,\hat f_t(X)) - g_{\theta_t^*}(X,f_t(X))\|^2\right]
=o_p(1),
\]
and
\[
\E_t\!\left[\sup_{\theta\in \mathcal N}
\|\nabla_\theta g_\theta(X,\hat f_t(X)) - \nabla_\theta g_\theta(X,f_t(X))\|\right]
=o_p(1).
\]
\end{assumption}

Assumption~\ref{ass:M4} is the natural high-level calibration condition.
For the mean-estimation problem with sparse linear discrepancy and lasso calibration, it
follows from the rate
\[
(\hat\delta-\delta^*)^\top \Sigma_t (\hat\delta-\delta^*)
=
O_p\!\left(\frac{s\log p}{n_1}\right)
\]
together with the condition $s\log p/n_1 \to 0$.
For smooth GLMs and other smooth estimating equations, the same assumptions can often be
verified by Lipschitz continuity of $g_\theta$ in its second argument.

\subsection{Results for split TC PPI M estimator}
%\subsection{Conditional unbiasedness}

\begin{proposition}[Exact conditional unbiasedness]
\label{prop:M-unbiased}
Under Assumption~\ref{ass:M3}, for every fixed $\theta \in \Theta$,
\[
\E\!\left[\Psi_{n_2,N}(\theta;\hat f_t)\mid \mathcal D_s,{\cal D}_{t,1}^{\rm gold}\right]
=
\E_t[g_\theta(X,Y)].
\]
In particular,
\[
\E\!\left[\Psi_{n_2,N}(\theta_t^*;\hat f_t)\mid \mathcal D_s,{\cal D}_{t,1}^{\rm gold}\right]=0.
\]
\end{proposition}

%\subsection{Consistency}

\begin{theorem}[Consistency]
\label{thm:M-consistency}
Suppose Assumptions~\ref{ass:M1}--\ref{ass:M4} hold and $n_2\to\infty$, $N\to\infty$.
Then
\[
\hat\theta_{\mathrm{TC},\lambda} \xrightarrow{P} \theta_t^*.
\]
\end{theorem}

\subsection{Asymptotic linearity and asymptotic normality}

Define the oracle influence contribution
\begin{equation}\label{eq:M-oracle-infl}
\zeta_{\lambda,t}(X,Y,\tilde X)
=
g_{\theta_t^*}(X,Y)
-\lambda g_{\theta_t^*}(X,f_t(X))
+\lambda g_{\theta_t^*}(\tilde X,f_t(\tilde X)).
\end{equation}
Let $\rho = \lim n_2/N \in [0,\infty)$ and define
\begin{equation}\label{eq:M-V}
V_{\lambda,t}
=
\Var_t\!\big(g_{\theta_t^*}(X,Y)-\lambda g_{\theta_t^*}(X,f_t(X))\big)
+
\lambda^2 \rho\,
\Var_t\!\big(g_{\theta_t^*}(X,f_t(X))\big).
\end{equation}

\begin{theorem}[Asymptotic linearity and normality]
\label{thm:M-linear}
Suppose Assumptions~\ref{ass:M1}--\ref{ass:M4} hold and $n_2/N\to \rho \in [0,\infty)$.
Then
\begin{align}\label{eq:M-linear}
\sqrt{n_2}\,(\hat\theta_{\mathrm{TC},\lambda}-\theta_t^*)
=
-A_t^{-1}
\bigg[ & 
\frac1{\sqrt{n_2}}\sum_{i\in \mathcal I_2}
\bigl\{
g_{\theta_t^*}(X_i,Y_i)-\lambda g_{\theta_t^*}(X_i,f_t(X_i))
\bigr\}
\nonumber \\ & +
\frac{\lambda\sqrt{n_2}}{N}\sum_{j=1}^N
g_{\theta_t^*}(\tilde X_j,f_t(\tilde X_j))
\bigg]
+o_p(1).
\end{align}
Further,
\[
\sqrt{n_2}\,(\hat\theta_{\mathrm{TC},\lambda}-\theta_t^*)
\;\xrightarrow{d}\;
N\!\left(0,\, A_t^{-1}V_{\lambda,t}A_t^{-T}\right).
\]
\end{theorem}

\subsection{Cross-fitted extension for general \texorpdfstring{$M$}-estimation}

Partition the gold sample into $K$ equal folds and let $\hat f_t^{(-k)}$ be trained without fold $I_k$. Define
\[
\Psi_{\rm CF}(\theta)
=
\frac1n\sum_{k=1}^K\sum_{i\in I_k}
\{g_\theta(X_i,Y_i)-\lambda g_\theta(X_i,\hat f_t^{(-k)}(X_i))\}
+
\frac{\lambda}{KN}\sum_{k=1}^K\sum_{j=1}^N
 g_\theta(\widetilde X_j,\hat f_t^{(-k)}(\widetilde X_j)).
\]
Let $\hat\theta_{\rm TC\text{-}Cross,\lambda}$ be a consistent root of $\Psi_{\rm CF}(\theta)=0$.

\begin{theorem}[Cross-fitted TC-PPI for general $M$-estimation]
\label{thm:M-cross}
Suppose $K$ is fixed and the identification, smoothness, Jacobian, and moment conditions in Assumptions~\ref{ass:M1}--\ref{ass:M3} hold. Suppose, uniformly over the fixed folds,
\[
\E_t\!\left[\|g_{\theta_t^*}(X,\hat f_t^{(-k)}(X))-g_{\theta_t^*}(X,f_t(X))\|^2\mid\mathcal D_s,\mathcal D_{-k}\right]=o_p(1),
\]
with the analogous Jacobian consistency condition, and let $n/N\to\rho\in[0,\infty)$. Then
\begin{align*}
\sqrt n\,(\hat\theta_{\rm TC\text{-}Cross,\lambda}-\theta_t^*)
=
-A_t^{-1}\bigg[&
\frac1{\sqrt n}\sum_{i=1}^n\{g_{\theta_t^*}(X_i,Y_i)-\lambda g_{\theta_t^*}(X_i,f_t(X_i))\}
\\ & +
\frac{\lambda\sqrt n}{N}\sum_{j=1}^N g_{\theta_t^*}(\widetilde X_j,f_t(\widetilde X_j))
\bigg]+o_p(1),
\end{align*}
and consequently
\[
\sqrt n\,(\hat\theta_{\rm TC\text{-}Cross,\lambda}-\theta_t^*)
\overset{d}{\longrightarrow}
N\!\left(0,A_t^{-1}V_{\lambda,t}^{\rm cross}A_t^{-T}\right),
\]
where
\[
V_{\lambda,t}^{\rm cross}
=
\Var_t\!\{g_{\theta_t^*}(X,Y)-\lambda g_{\theta_t^*}(X,f_t(X))\}
+
\lambda^2\rho\Var_t\!\{g_{\theta_t^*}(X,f_t(X))\}.
\]
Thus the split-sample result with scale $n_2$ is replaced by a full-sample result with scale $n$. This theorem covers TC--Cross-PPI++ by allowing $\lambda\ne1$; a consistent data-dependent $\widehat\lambda$ can be substituted without changing the first-order limit under the usual stochastic equicontinuity condition.
\end{theorem}

\subsection{Proofs of results from this section}

\begin{proof}[Proof of Theorem \ref{thm:glm-linear}]
A mean-value expansion of the estimating equation around $\beta_t^*$ gives
\[
\sqrt{n_2}(\hat\beta_{\mathrm{TC},\lambda}-\beta_t^*)
=-A_t^{-1}\sqrt{n_2}\Psi_{n_2,N}(\beta_t^*;\hat f_t)+o_p(1),
\]
by standard fixed-dimensional $Z$-estimation arguments and the Jacobian assumptions. It remains to replace $\hat f_t$ by $f_t$. Put
\[
h_{\hat f}(X)=X\{\hat f_t(X)-f_t(X)\}.
\]
The difference between the two scores is
\[
\lambda\sqrt{n_2}\left\{
\frac1{n_2}\sum_{i\in\mathcal I_2}h_{\hat f}(X_i)
-
\frac1N\sum_{j=1}^Nh_{\hat f}(\widetilde X_j)
\right\}.
\]
Conditional on the calibration data, the two empirical averages have the same population mean, so the mean cancels exactly. Their conditional variances are bounded by the weighted $L_2$ calibration error times $1/n_2$ and $1/N$, respectively. Since $n_2/N=O(1)$, Chebyshev's inequality makes the displayed term $o_p(1)$. Substituting $f_t$ therefore yields the claimed expansion.

Apply the multivariate Lindeberg--Feller central limit theorem to the two independent empirical averages in Theorem~\ref{thm:glm-linear}. The labeled contribution has covariance
$\Var_t\{g_{\beta_t^*}(X,Y)-\lambda g_{\beta_t^*}(X,f_t(X))\}$, while the unlabeled contribution has covariance $\lambda^2(n_2/N)\Var_t\{g_{\beta_t^*}(X,f_t(X))\}$. Slutsky's theorem and $n_2/N\to\rho$ give the result.
\end{proof}

\begin{proof}[Proof of Proposition \ref{prop:M-unbiased}]
Condition on $\mathcal D_s$ and ${\cal D}_{t,1}^{\rm gold}$, so that $\hat f_t$ is fixed.
Because the rectification split $\{(X_i,Y_i)\}_{i\in \mathcal I_2}$ and the unlabeled sample
$\{\tilde X_j\}_{j=1}^N$ are independent of the conditioning sigma-field and are sampled
from the target-domain law,
\[
\E\!\left[\frac1{n_2}\sum_{i\in\mathcal I_2} g_\theta(X_i,Y_i)
\middle| \mathcal D_s,{\cal D}_{t,1}^{\rm gold}\right]
=
\E_t[g_\theta(X,Y)],
\]
\[
\E\!\left[\frac1{n_2}\sum_{i\in\mathcal I_2} g_\theta(X_i,\hat f_t(X_i))
\middle| \mathcal D_s,{\cal D}_{t,1}^{\rm gold}\right]
=
\E_t[g_\theta(X,\hat f_t(X))],
\]
and
\[
\E\!\left[\frac1N\sum_{j=1}^N g_\theta(\tilde X_j,\hat f_t(\tilde X_j))
\middle| \mathcal D_s,{\cal D}_{t,1}^{\rm gold}\right]
=
\E_t[g_\theta(X,\hat f_t(X))].
\]
Substituting into \eqref{eq:M-Psi} proves the claim.
\end{proof}

\begin{proof}[Proof of Theorem \ref{thm:M-consistency}]
By Proposition~\ref{prop:M-unbiased} and the uniform law of large numbers implied by
Assumption~\ref{ass:M3},
\[
\sup_{\theta\in \mathcal N}
\left\|
\Psi_{n_2,N}(\theta;\hat f_t) - \E_t[g_\theta(X,Y)]
\right\|
=o_p(1).
\]
By Assumption~\ref{ass:M1}, the map $\theta \mapsto \E_t[g_\theta(X,Y)]$ has a unique zero
at $\theta_t^*$ in $\mathcal N$. Standard $Z$-estimation consistency then gives
\[
\hat\theta_{\mathrm{TC},\lambda}\xrightarrow{P}\theta_t^*.
\]
\end{proof}

\begin{proof}[Proof of Theorem \ref{thm:M-linear}]
Since $\hat\theta_{\mathrm{TC},\lambda}$ solves \eqref{eq:M-estimator},
\[
0
=
\Psi_{n_2,N}(\hat\theta_{\mathrm{TC},\lambda};\hat f_t).
\]
Expand around $\theta_t^*$:
\[
0
=
\Psi_{n_2,N}(\theta_t^*;\hat f_t)
+
\dot\Psi_{n_2,N}(\bar\theta;\hat f_t)
(\hat\theta_{\mathrm{TC},\lambda}-\theta_t^*)
\]
for some random $\bar\theta$ between $\hat\theta_{\mathrm{TC},\lambda}$ and $\theta_t^*$.
By Theorem~\ref{thm:M-consistency}, $\bar\theta \to \theta_t^*$ in probability. By
Assumptions~\ref{ass:M3}--\ref{ass:M4} and a uniform law of large numbers,
\[
\dot\Psi_{n_2,N}(\bar\theta;\hat f_t)
=
A_t + o_p(1).
\]
Hence,
\[
\sqrt{n_2}(\hat\theta_{\mathrm{TC},\lambda}-\theta_t^*)
=
-A_t^{-1}\sqrt{n_2}\,\Psi_{n_2,N}(\theta_t^*;\hat f_t)+o_p(1).
\]
It remains to replace $\hat f_t$ by $f_t$ in the score at $\theta_t^*$. Let
\[
 h_{\hat f}(X)=g_{\theta_t^*}(X,\hat f_t(X))-g_{\theta_t^*}(X,f_t(X)).
\]
The replacement error equals
\[
 \lambda\sqrt{n_2}\left\{
 \frac1N\sum_{j=1}^N h_{\hat f}(\widetilde X_j)
 -\frac1{n_2}\sum_{i\in\mathcal I_2}h_{\hat f}(X_i)
 \right\}.
\]
Conditional on the calibration data, the two empirical averages have the same population mean, so that mean cancels exactly. Assumption~\ref{ass:M4} implies that the conditional variance of the displayed difference is $o_p(1)$ after multiplication by $n_2$, because $n_2/N=O(1)$. Hence Chebyshev's inequality gives
\[
\sqrt{n_2}\,
\Big(
\Psi_{n_2,N}(\theta_t^*;\hat f_t)-\Psi_{n_2,N}(\theta_t^*;f_t)
\Big)
=o_p(1).
\]
Substituting this into the previous display proves \eqref{eq:M-linear}.

The bracketed term in \eqref{eq:M-linear} is a sum of two independent empirical averages,
one based on the rectification split and one based on the unlabeled sample. By
Assumption~\ref{ass:M3}, the Lindeberg--Feller central limit theorem applies, yielding
\[
\frac1{\sqrt{n_2}}\sum_{i\in \mathcal I_2}
\bigl\{
g_{\theta_t^*}(X_i,Y_i)-\lambda g_{\theta_t^*}(X_i,f_t(X_i))
\bigr\}
+
\frac{\lambda\sqrt{n_2}}{N}\sum_{j=1}^N
g_{\theta_t^*}(\tilde X_j,f_t(\tilde X_j))
\;\xrightarrow{d}\;
N(0,V_{\lambda,t}).
\]
Combine this with Theorem~\ref{thm:M-linear} and Slutsky's theorem.
\end{proof}

\begin{proof}[Proof of Theorem \ref{thm:M-cross}]
Expand the cross-fitted estimating equation around $\theta_t^*$. The Jacobian converges to $A_t$ by the assumed foldwise Jacobian consistency and fixed $K$. For the nuisance replacement, define
\[
h_k(X)=g_{\theta_t^*}(X,\hat f_t^{(-k)}(X))-g_{\theta_t^*}(X,f_t(X)).
\]
For each fold, conditional on its training complement, the held-out average of $h_k$ and the unlabeled average of $h_k$ have the same population mean; hence that mean cancels. The foldwise $L_2$ condition and Chebyshev's inequality imply that the sum of the resulting centered differences is $o_p(n^{-1/2})$ because $K$ is fixed and $n/N=O(1)$. Replacing every fold predictor by $f_t$ therefore yields the displayed full-sample asymptotic linear representation. The multivariate CLT for the labeled and unlabeled oracle empirical averages gives the stated covariance.
\end{proof}

\section{Proofs of main results}\label{app:proofs}

\subsection{Split-TC-PPI proofs}
\begin{proof}[Proof of Proposition \ref{prop:unbiased}]
    Condition on $\mathcal D_s$ and the calibration split ${\cal D}_{t,1}^{\rm gold}$, so that $\widehat f_t$ is fixed. Because the unlabeled sample $\tilde{X}$ and the rectification split $\{X_i,Y_i\}_{i\in{\cal I}_2}$ are independent of the conditioning sigma-field $\sigma(\mathcal{D}_s,{\cal D}_{t,1}^{\rm gold})$, and each has the target-domain covariate or joint law,
    \[
    \E\left[\frac{1}{N}\sum_{j=1}^N \widehat f_t(\widetilde X_j)\middle\vert \mathcal D_s,{\cal D}_{t,1}^{\rm gold}\right]=\E_t[\widehat f_t(X)]
    \]
    and
    \[
    \E\left[\frac{1}{n_2}\sum_{i\in\mathcal I_2}\{Y_i-\widehat f_t(X_i)\}\middle\vert \mathcal D_s,{\cal D}_{t,1}^{\rm gold}\right]=\E_t[Y-\widehat f_t(X)].
    \]
    Summing gives $
\E\bigl[\widehat\theta_{\mathrm{TC}}\mid \mathcal D_s,{\cal D}_{t,1}^{\rm gold}\bigr] = \theta_t^*$. The law of iterated expectations give the unconditional expectation.

Now, conditional on $\mathcal D_s$ and the calibration split ${\cal D}_{t,1}^{\rm gold}$, the two empirical means which are being added in \eqref{eq:tcppi} are each computed from independent samples. Therefore they have covariance 0 and consequently, the variance of the TC-PPI estimator is the sum of their variances.

Now let,
\[
e(X)\coloneqq \widehat\Delta(X)-\Delta^*(X),
\]
denote the estimation error of the discrepancy.
Combining \eqref{eq:general-delta}, \eqref{eq:fhat}, and definition of $\varepsilon$, we get
\[
Y-\widehat f_t(X)=\varepsilon-e(X).
\]
Conditional on $\mathcal D_s$ and ${\cal D}_{t,1}^{\rm gold}$, the function $e$ is fixed and measurable in $X$. Since $\E_t[\varepsilon\mid X]=0$,
\[
\Cov_t\{\varepsilon,e(X)\mid\mathcal D_s,{\cal D}_{t,1}^{\rm gold}\}
=
\E_t\!\left[e(X)\E_t(\varepsilon\mid X)\mid\mathcal D_s,{\cal D}_{t,1}^{\rm gold}\right]
=0.
\]
Therefore
\[
\Var_t\{Y-\widehat f_t(X)\mid\mathcal D_s,{\cal D}_{t,1}^{\rm gold}\}
=
\Var_t(\varepsilon)+\Var_t\{e(X)\mid\mathcal D_s,{\cal D}_{t,1}^{\rm gold}\},
\]
which is the expression for the conditional variance.
\end{proof}

\begin{proof}[Proof of Theorem \ref{thm:clt}]
We again condition on $\mathcal D_s$ and the calibration split ${\cal D}_{t,1}^{\rm gold}$. Then the quantities
\[
A_j=\widehat f_t(\widetilde X_j)-\E_t\{\widehat f_t(X)\},
\qquad j=1,\dots,N,
\]
and
\[
B_i=Y_i-\widehat f_t(X_i)-\E_t\{Y-\widehat f_t(X)\},
\qquad i\in\mathcal I_2,
\]
are independent, centered, and conditionally i.i.d. within each group. The stated conditional $(2+\eta)$ moment condition gives the Lindeberg condition, and therefore the triangular-array CLT yields the asymptotic normality \eqref{eq:condclt} with conditional variance $V_{n_1}$ as follows,
\begin{equation}\label{eq:Vn}
V_{n_1}
=
\frac{n_2}{N}\Var_t\bigl(\widehat f_t(X)\mid\mathcal D_s,{\cal D}_{t,1}^{\rm gold}\bigr)
+
\Var_t\bigl(Y-\widehat f_t(X)\mid\mathcal D_s,{\cal D}_{t,1}^{\rm gold}\bigr).
\end{equation}

By Proposition~\ref{prop:unbiased},
\[
\Var_t\{Y-\widehat f_t(X)\mid\mathcal D_s,{\cal D}_{t,1}^{\rm gold}\}
=
\Var_t(\varepsilon)+\Var_t\{e(X)\mid\mathcal D_s,{\cal D}_{t,1}^{\rm gold}\}
=
\Var_t(\varepsilon)+O_p(b_{n_1}),
\]
since $\Var_t\{e(X)\}\le\E_t[e(X)^2]=b_{n_1}$. Also since $\widehat f_t=f_t+e$, we have,
\begin{align*}
\Var_t\{\widehat f_t(X)\mid\mathcal D_s,{\cal D}_{t,1}^{\rm gold}\}
={}&\Var_t\{f_t(X)\}
+2\Cov_t\{f_t(X),e(X)\mid\mathcal D_s,{\cal D}_{t,1}^{\rm gold}\}\\
&+\Var_t\{e(X)\mid\mathcal D_s,{\cal D}_{t,1}^{\rm gold}\}.
\end{align*}
By Cauchy--Schwarz inequality, the covariance term is $O_p(\sqrt{b_{n_1}})$, while the last term is $O_p(b_{n_1})$. This proves \eqref{eq:Vexpand}. Since $b_{n_1}=o_p(1)$ and $n_2/N\to\rho_1$, we have $V_{n_1}$ converges in probability to $\rho_1\Var_t\{f_t(X)\}+\Var_t(\varepsilon)$.
\end{proof}

\subsection{Lasso results}

\begin{assumption}[Target design and noise]\label{ass:design}
The target covariate $X\in\R^p$ is mean-zero and sub-Gaussian with covariance matrix $\Sigma_t$. There exist constants $0<\underline\kappa\le \overline\kappa<\infty$ such that
\[
\underline\kappa \le \lambda_{\min}(\Sigma_t) \le \lambda_{\max}(\Sigma_t) \le \overline\kappa.
\]
Further, the conditional distribution of noise $\varepsilon|X$, in $\varepsilon=Y-f_t(X)$, is sub-Gaussian with variance proxy $\sigma_\varepsilon^2$.
\end{assumption}

For $S\coloneqq \mathrm{supp}(\delta^*)$, let $s=|S|$ and let
\[
\hat\Sigma_1 \coloneqq  \frac1{n_1}\sum_{i\in\mathcal I_1} X_iX_i^\top,
\quad\text{so that}\quad
\mathbb E_t[\hat\Sigma_1]=\Sigma_t.
\]

We will use the following standard compatibility condition on the sample Gram matrix.

\begin{assumption}[Sparsity]
\label{ass:sparsity}
The vector $\delta^*$ is $s$-sparse. That is, $\|\delta^*\|_0\le s$ for some $s=s(n_1)>0$.
\end{assumption}

\begin{assumption}[Compatibility]
\label{ass:compat}
Let $C_{n_1}(\phi)\coloneqq \left\{\underset{u\ne 0;\ \|u_{S^c}\|_1\le 3\|u_S\|_1}{\min}
\frac{\sqrt{s}\,\|u\|_{\hat\Sigma_1}}{\|u_S\|_1}
\ge \phi\right\}$ where $\|u\|_{\hat\Sigma_1}^2 \coloneqq  u^\top \hat\Sigma_1 u$. There exists a constant $\phi_0>0$ such that
\[
{\mathbb P}\Bigl[C_{n_1}(\phi_0)\Bigr]\overset{n_1\to\infty}{\to}1.
\]
\end{assumption}

\begin{assumption}[Score control]
\label{ass:score}
For some deterministic sequence $\lambda=\lambda_{n_1,p}>0$,
\[
\mathbb P\!\left(
\left\|
\frac1{n_1}\sum_{i\in\mathcal I_1} X_i\varepsilon_i
\right\|_\infty
\le \frac{\lambda}{2}
\ \middle|\ \mathcal D_s
\right)\to 1.
\]
\end{assumption}

\begin{remark}
Assumption~\ref{ass:score} is a high-level condition. It can be verified under a variety of
primitive assumptions. For example, it holds with
$\lambda \asymp \sqrt{(\log p)/n_1}$ if the coordinates of $X$ are sub-Gaussian and
$X_j\varepsilon$ are uniformly sub-exponential. We do not require marginal sub-Gaussianity of
$\varepsilon$, which would be too restrictive when $f_s$ is an arbitrary black-box
predictor.
\end{remark}

We first state the deterministic oracle inequality.

\begin{proposition}
[Deterministic lasso bound]
\label{lem:det-lasso}
Suppose Assumption~\ref{ass:sparsity} holds. Then, on the events $C_{n_1}(\phi_0)$ and
\[
\mathcal E_\lambda \coloneqq 
\left\{
\left\|
\frac1{n_1}\sum_{i\in\mathcal I_1} X_i\varepsilon_i
\right\|_\infty
\le \frac{\lambda}{2}
\right\},
\]
any solution $\hat\delta$ of \eqref{eq:lasso-rig} satisfies
\begin{align}
\|\hat\delta-\delta^*\|_1
&\le
\frac{12s\lambda}{\phi_0^2},
\label{eq:l1-bound}
\\
\|\hat\delta-\delta^*\|_{\hat\Sigma_1}^2
&\le
\frac{9s\lambda^2}{\phi_0^2}.
\label{eq:pred-bound}
\end{align}  
\end{proposition}

The next proposition gives the calibration rate used later in the TC-PPI mean analysis.

\begin{proposition}
    [Calibration lasso rate]
\label{thm:lasso-calibration-rig}
Suppose Assumptions~\ref{ass:sparsity}--\ref{ass:score} in the Appendix hold. In addition, assume that on the lasso cone $\mathcal C(S,3)=\{u:\|u_{S^c}\|_1\le3\|u_S\|_1\}$ there is a constant $C_{\rm pop}>0$ such that, with probability tending to one,
\begin{equation}\label{eq:sample-pop-transfer}
 u^\top\Sigma_tu\le C_{\rm pop}\,u^\top\widehat\Sigma_1u
 \qquad\text{for all }u\in\mathcal C(S,3).
\end{equation}
Then any lasso solution $\hat\delta$ of \eqref{eq:lasso-rig} satisfies
\begin{align}
\|\hat\delta-\delta^*\|_1
&=O_{\mathbb P}(s\lambda),\\
(\hat\delta-\delta^*)^\top\Sigma_t(\hat\delta-\delta^*)
&=O_{\mathbb P}(s\lambda^2).
\end{align}
In particular, if $\lambda\asymp\sqrt{(\log p)/n_1}$, then
\[
\|\hat\delta-\delta^*\|_1
=O_{\mathbb P}\!\left(s\sqrt{\frac{\log p}{n_1}}\right),
\quad
(\hat\delta-\delta^*)^\top\Sigma_t(\hat\delta-\delta^*)
=O_{\mathbb P}\!\left(\frac{s\log p}{n_1}\right).
\]
\end{proposition}

\begin{lemma}[Lasso calibration risk in expectation]\label{lem:exp-lasso}
Suppose Assumptions~\ref{ass:design} and \ref{ass:sparsity} hold and $\|\delta^*\|_2\le B$. For fold $k$, let $\widehat\delta^{(-k)}$ be any solution of the lasso on $\mathcal D_{-k}$,
\[
\widehat\delta^{(-k)}\in\argmin_{\delta\in\R^p}\Bigl\{\frac1{2n_{\rm tr}}\sum_{i\notin I_k}\bigl(Y_i-f_s(X_i)-X_i^\top\delta\bigr)^2+\lambda\|\delta\|_1\Bigr\},
\qquad
\lambda=A\sigma_\varepsilon\sqrt{\frac{\log p}{n_{\rm tr}}} .
\]
Assume $p\ge \max\{3,n_{\rm tr}^{\upsilon}\}$ for some $\upsilon>0$. Then there exist constants $A_0,c_0,C>0$, depending only on $\underline\kappa,\overline\kappa,\upsilon$ and $B/\sigma_\varepsilon$, such that if $A\ge A_0$ and $s\log p/n_{\rm tr}\le c_0$ (with $c_0$ allowed to depend on $A$), then for every $k$,
\[
\E\!\left[(\widehat\delta^{(-k)}-\delta^*)^\top\Sigma_t(\widehat\delta^{(-k)}-\delta^*)\,\middle|\,\mathcal D_s\right]
\le C\,\frac{\sigma_\varepsilon^2\, s\log p}{n_{\rm tr}} .
\]
\end{lemma}
In particular, this lemma implies that $r_n\le C\sigma_\varepsilon^2 s\log p/n_{\rm tr}$ under the lasso-calibration.

\subsubsection{Proofs for this section}
\begin{proof}[Proof of Propostion \ref{lem:det-lasso}]
Let $\Delta=\hat\delta-\delta^*$. By optimality of $\hat\delta$,
\[
\frac1{2n_1}\|X\Delta-\varepsilon\|_2^2+\lambda\|\hat\delta\|_1
\le
\frac1{2n_1}\|\varepsilon\|_2^2+\lambda\|\delta^*\|_1,
\]
where $X$ is the $n_1\times p$ design matrix on the calibration split and
$\varepsilon=(\varepsilon_i)_{i\in{\cal D}_{t,1}^{\rm gold}}$. Expanding and rearranging gives
\[
\frac1{2n_1}\|X\Delta\|_2^2
\le
\frac1{n_1}\varepsilon^\top X\Delta
+\lambda(\|\delta^*\|_1-\|\hat\delta\|_1).
\]
On $\mathcal E_\lambda$,
\[
\frac1{n_1}\varepsilon^\top X\Delta
\le
\left\|\frac1{n_1}X^\top\varepsilon\right\|_\infty\|\Delta\|_1
\le \frac{\lambda}{2}\|\Delta\|_1.
\]
Since $\delta^*_{S^c}=0$,
\[
\|\delta^*\|_1-\|\hat\delta\|_1
\le
\|\Delta_S\|_1-\|\Delta_{S^c}\|_1.
\]
Hence
\[
\frac1{2}\|\Delta\|_{\hat\Sigma_1}^2
\le
\frac{\lambda}{2}(\|\Delta_S\|_1+\|\Delta_{S^c}\|_1)
+\lambda(\|\Delta_S\|_1-\|\Delta_{S^c}\|_1)
=
\frac{3\lambda}{2}\|\Delta_S\|_1-\frac{\lambda}{2}\|\Delta_{S^c}\|_1.
\]
Therefore
\[
\|\Delta_{S^c}\|_1\le 3\|\Delta_S\|_1,
\]
so $\Delta$ belongs to the compatibility cone. On the compatibility event in Assumption~\ref{ass:compat},
\[
\|\Delta_S\|_1 \le \frac{\sqrt s}{\phi_0}\|\Delta\|_{\hat\Sigma_1}.
\]
Combining with the previous display yields
\[
\frac1{2}\|\Delta\|_{\hat\Sigma_1}^2
\le
\frac{3\lambda\sqrt s}{2\phi_0}\|\Delta\|_{\hat\Sigma_1},
\]
which implies
\[
\|\Delta\|_{\hat\Sigma_1}
\le
\frac{3\lambda\sqrt s}{\phi_0}.
\]
This gives \eqref{eq:pred-bound}. Then
\[
\|\Delta\|_1
\le
\|\Delta_S\|_1+\|\Delta_{S^c}\|_1
\le
4\|\Delta_S\|_1
\le
\frac{4\sqrt s}{\phi_0}\|\Delta\|_{\hat\Sigma_1},
\]
which gives \eqref{eq:l1-bound}.
\end{proof}

\begin{proof}[Proof of Proposition \ref{thm:lasso-calibration-rig}]
By Assumption~\ref{ass:score}, the score event $\mathcal E_\lambda$ holds with probability tending to one. Proposition~\ref{lem:det-lasso} then gives the cone condition and
\[
 \|\hat\delta-\delta^*\|_{\widehat\Sigma_1}^2\le C s\lambda^2,
 \qquad
 \|\hat\delta-\delta^*\|_1\le C s\lambda.
\]
Since $\Delta=\hat\delta-\delta^*$ belongs to $\mathcal C(S,3)$, \eqref{eq:sample-pop-transfer} implies
\[
 \Delta^\top\Sigma_t\Delta
 \le C_{\rm pop}\Delta^\top\widehat\Sigma_1\Delta
 =O_p(s\lambda^2),
\]
hence the claimed rates follow. Under a sub-Gaussian random design, \eqref{eq:sample-pop-transfer} is a standard restricted covariance comparison and holds under the usual scaling $n_1\gtrsim s\log p$ with appropriate constants.
\end{proof}

\begin{proof}[Proof of Corollary \ref{cor:split-lasso-clt}]
Under sparse linear calibration,
\[
\widehat\Delta(X)-\Delta^*(X)
=X^\top(\widehat\delta-\delta^*).
\]
Because $X$ is mean-zero under Assumption~\ref{ass:design},
\[
b_{n_1}
=(\widehat\delta-\delta^*)^\top\Sigma_t(\widehat\delta-\delta^*).
\]
Therefore, Proposition~\ref{thm:lasso-calibration-rig} therefore gives
\[
b_{n_1}=O_p\!\left(\frac{s\log p}{n_1}\right).
\]
Substitution into Theorem~\ref{thm:clt} proves the result.
\end{proof}

\begin{proof}[Proof of Lemma~\ref{lem:exp-lasso}]
Fix $k$ and condition on $\mathcal D_s$ throughout, so that $f_s$ is fixed; the target sample $\mathcal D_{-k}$ is independent of $\mathcal D_s$.
Write $n\equiv n_{\rm tr}$, let $X\in\R^{n\times p}$ be the design on $\mathcal D_{-k}$, $\widehat\Sigma=n^{-1}X^\top X$,
$\tilde y_i=Y_i-f_s(X_i)=X_i^\top\delta^*+\varepsilon_i$, $\widehat\delta\equiv\widehat\delta^{(-k)}$ and $\Delta=\widehat\delta-\delta^*$.
Let $S=\mathrm{supp}(\delta^*)$ and $\mathcal C(S,3)=\{u:\|u_{S^c}\|_1\le3\|u_S\|_1\}$.
Throughout, $c,c',c_0,c_1,c_2,\dots$ and $C_1,C_2,\dots$ denote generic positive constants.

By \citet[Theorem~1]{raskutti2010restricted}, if $s\log p/n\le c_0$,
\[
\mathcal R:=\Bigl\{u^\top\widehat\Sigma u\ \ge\ c\,u^\top\Sigma_tu\ \ \text{for all }u\in\mathcal C(S,3)\Bigr\}
\]
holds with $\mathbb P(\mathcal R^c)\le c_1e^{-c_2n}$.
Also, taking $\lambda=A\sigma_\varepsilon\sqrt{\log p/n}$, we have
\[
\mathbb P(\mathcal S^c)\le 2p^{\,1-c_4A^2},
\qquad
\mathcal S:=\Bigl\{\bigl\|n^{-1}X^\top\varepsilon\bigr\|_\infty\le\lambda/2\Bigr\}.
\]
Let $\mathcal E:=\mathcal R\cap\mathcal S$, so that
\begin{equation}\label{eq:exp-lasso-bad}
\mathbb P(\mathcal E^c)\le c_1e^{-c_2n}+2p^{\,1-c_4A^2}.
\end{equation}
The proof of Proposition~\ref{lem:det-lasso} shows that on $\mathcal S$, $\Delta\in\mathcal C(S,3)$ and
$\tfrac12\|\Delta\|_{\widehat\Sigma}^2\le\tfrac{3\lambda}{2}\|\Delta_S\|_1$.
On $\mathcal R$,
$\|\Delta_S\|_1\le\sqrt s\|\Delta\|_2\le\sqrt{s/\underline\kappa}\,\|\Delta\|_{\Sigma_t}\le\sqrt{s/(\underline\kappa c)}\,\|\Delta\|_{\widehat\Sigma}$,
so $\|\Delta\|_{\widehat\Sigma}\le c'\lambda\sqrt{s/(\underline\kappa c)}$ and
\begin{equation}\label{eq:exp-lasso-good}
\Delta^\top\Sigma_t\Delta\ \le\ \frac{\|\Delta\|_{\widehat\Sigma}^2}{c}
\ \le\ \frac{9A^2}{\underline\kappa\,c^2}\,\frac{\sigma_\varepsilon^2 s\log p}{n}\le c'\frac{A^2s\sigma_\varepsilon^2\log p}{n}
\qquad\text{on }\mathcal E.
\end{equation}
Since \eqref{eq:exp-lasso-good} holds pointwise on $\mathcal E$ and its right-hand side is deterministic given $\mathcal D_s$,
\begin{equation}\label{eq:exp-lasso-goodexp}
\E\bigl[\Delta^\top\Sigma_t\Delta\,\mathbf 1_{\mathcal E}\mid\mathcal D_s\bigr]
\le c'\frac{A^2s\sigma_\varepsilon^2\log p}{n}.
\end{equation}

Next, comparing the lasso objective at $\widehat\delta$ and at $0$,
\[
\lambda\|\widehat\delta\|_1\le\frac1{2n}\|\tilde y\|_2^2=:\frac{T}{2},
\qquad\text{so}\qquad \|\widehat\delta\|_1\le\frac{T}{2\lambda}.
\]
Since $\|\delta^*\|_1\le\sqrt s\,\|\delta^*\|_2\le\sqrt s\,B$,
\[
\Delta^\top\Sigma_t\Delta\le\overline\kappa\|\Delta\|_2^2\le\overline\kappa\|\Delta\|_1^2
\le 2\overline\kappa\Bigl(\frac{T^2}{4\lambda^2}+sB^2\Bigr).
\]
Since $x\mapsto x^4$ is convex, applying it to the empirical average gives
$T^4=\bigl(n^{-1}\sum_{i}\tilde y_i^2\bigr)^4\le n^{-1}\sum_{i}\tilde y_i^8$.
As the $\tilde y_i$ are i.i.d.\ given $\mathcal D_s$, $\E[T^4\mid\mathcal D_s]\le\E[\tilde y_1^8\mid\mathcal D_s]$.
Since $\tilde y=X^\top\delta^*+\varepsilon$ is sub-Gaussian with $\|\tilde y\|_{\psi_2}^2\le C_2M$, where $M:=\sigma_\varepsilon^2+\overline\kappa B^2$, we have $\E\tilde y^8\le C_3M^4$. Therefore
\[
\bigl(\E[(\Delta^\top\Sigma_t\Delta)^2]\bigr)^{1/2}
\le C_4\Bigl(\frac{M^2}{\lambda^2}+sB^2\Bigr)
= C_4\Bigl(\frac{M^2 n}{A^2\sigma_\varepsilon^2\log p}+sB^2\Bigr)
\le \gamma\, n,
\]
where $\gamma:=C_5\{M^2/\sigma_\varepsilon^2+B^2\}$, using $\log p\ge1$ and $s\le n$.
By Cauchy--Schwarz and \eqref{eq:exp-lasso-bad},
\[
\E\bigl[\Delta^\top\Sigma_t\Delta\,\mathbf 1_{\mathcal E^c}\bigr]
\le\gamma n\,\mathbb P(\mathcal E^c)^{1/2}
\le c''\gamma n\Bigl(e^{-c_2n/2}+p^{(1-c_4A^2)/2}\Bigr).
\]
Choose $A_0\ge1$ with $c_4A_0^2\ge1+6/\upsilon$. Then for $A\ge A_0$ and $p\ge n^\upsilon$,
$p^{(1-c_4A^2)/2}\le n^{-\upsilon(c_4A^2-1)/2}\le n^{-3}$. Also $n e^{-c_2n/2}\le C_6n^{-2}$. Hence
\begin{equation}\label{eq:exp-lasso-badbound}
\E\bigl[\Delta^\top\Sigma_t\Delta\,\mathbf 1_{\mathcal E^c}\bigr]\le C_6\gamma\,n^{-2}
\le \frac{C_6\gamma}{\sigma_\varepsilon^2}\,\frac{\sigma_\varepsilon^2 s\log p}{n},
\end{equation}
where the last step uses $s\ge1$ and $\log p\ge1$.

Combining the bounds on $\cal E$ and ${\cal E}^c$ completes the proof.
\end{proof}

\subsection{TC-Cross-PPI proofs}

\begin{proof}[Proof of Proposition~\ref{prop:tc-cross-unbiased}]
For each fold $k$, let
\[
\mathcal D_{-k}
=
\{(X_i,Y_i):i\notin I_k\}.
\]
Condition on $\mathcal D_s$ and $\mathcal D_{-k}$. Then
$\widehat f_t^{(-k)}$ is fixed, while the held-out observations
$\{(X_i,Y_i):i\in I_k\}$ are independent of $\mathcal D_{-k}$ and have
law $P_t$. The unlabeled sample is also independent of the gold sample and
has covariate law $P_{t,X}$. Therefore,
\[
\E\!\left[
\frac1N\sum_{j=1}^N
\widehat f_t^{(-k)}(\widetilde X_j)
\middle|
\mathcal D_s,\mathcal D_{-k}
\right]
=
\E_t\!\left[
\widehat f_t^{(-k)}(X)
\middle|
\mathcal D_s,\mathcal D_{-k}
\right],
\]
whereas, since $|I_k|=n/K$,
\begin{align*}
&\E\!\left[
\frac K n\sum_{i\in I_k}
\{Y_i-\widehat f_t^{(-k)}(X_i)\}
\middle|
\mathcal D_s,\mathcal D_{-k}
\right]\\
&\qquad=
\E_t[Y]
-
\E_t\!\left[
\widehat f_t^{(-k)}(X)
\middle|
\mathcal D_s,\mathcal D_{-k}
\right].
\end{align*}
Hence the conditional expectation of the $k$th bracket in
\eqref{eq:tc-cross-fold-decomp} is $\theta_t^*$. By linearity and iterated
expectation,
\begin{align*}
\E\!\left[
\widehat\mu_{\rm TC\text{-}Cross}
\mid\mathcal D_s
\right]
&=
\frac1K\sum_{k=1}^K
\E\!\left[
\E\!\left[
T_k
\mid\mathcal D_s,\mathcal D_{-k}
\right]
\middle|\mathcal D_s
\right]\\
&=
\frac1K\sum_{k=1}^K\theta_t^*
=
\theta_t^*,
\end{align*}
where $T_k$ denotes the $k$th bracket in
\eqref{eq:tc-cross-fold-decomp}.
\end{proof}

\begin{proof}[Proof of Theorem~\ref{thm:tc-cross-mean}]
Recall $e_k(x)$ is the error of the predictor trained on all folds except for the $k$th fold in approximating the target function
\[
e_k(x)=\widehat f_t^{(-k)}(x)-f_t(x).
\]
Now define for the $k$th fold,
\[
Z_{n,k}
=
\frac1N\sum_{j=1}^N e_k(\widetilde X_j)
-
\frac K n\sum_{i\in I_k}e_k(X_i),
\]
and let
\[
R_n=\frac1K\sum_{k=1}^K Z_{n,k}.
\]
Substituting $Y=f_t(x)+\varepsilon$ into the estimator gives
\[
\widehat\mu_{\rm TC\text{-}Cross}-\theta_t^*
=
\frac1n\sum_{i=1}^n\varepsilon_i
+
\frac1N\sum_{j=1}^N
\{f_t(\widetilde X_j)-\E_t f_t(X)\}
+
R_n.
\]

For each fixed $k$, let us condition on $\mathcal D_s$ and the training complement
$\mathcal D_{-k}$. Then $e_k$ is fixed, the observations in $I_k$ are
independent of $e_k$, and the unlabeled sample $\widetilde X$ is independent of both the
training complement and the held-out fold. Therefore the two empirical
averages defining $Z_{n,k}$ are independent, and it is easy to see that they have the same conditional population mean.
Consequently,
\[
\E[Z_{n,k}\mid\mathcal D_s,\mathcal D_{-k}]=0,
\]
and
\begin{align*}
\E[Z_{n,k}^2\mid\mathcal D_s,\mathcal D_{-k}]
&=
\frac1N
\Var_t\{e_k(X)\mid\mathcal D_s,\mathcal D_{-k}\}+
\frac K n
\Var_t\{e_k(X)\mid\mathcal D_s,\mathcal D_{-k}\}\\
&\le
\left(\frac1N+\frac K n\right)
\E_t[e_k(X)^2\mid\mathcal D_s,\mathcal D_{-k}].
\end{align*}

The quantities $Z_{n,1},\ldots,Z_{n,K}$ need not be independent because
the training complements overlap and the same unlabeled observations are used
for every fold. Since $K$ is fixed,
\[
R_n^2
=
\left(\frac1K\sum_{k=1}^KZ_{n,k}\right)^2
\le
\frac1K\sum_{k=1}^KZ_{n,k}^2.
\]
Hence, if
\[
\max_{1\le k\le K}
\E_t[e_k(X)^2\mid\mathcal D_s,\mathcal D_{-k}]
=o_p(1),
\]
conditional Chebyshev's inequality applied fold by fold, followed by the fixed
$K$ bound above, gives
\[
\sqrt n\,R_n=o_p(1).
\]
If the maximum foldwise prediction error is instead $O_p(r_n)$ for a
deterministic sequence $r_n\to0$, the same argument yields
\[
\sqrt n\,R_n=O_p(\sqrt{r_n}).
\]

Therefore,
\[
\sqrt n(\widehat\mu_{\rm TC\text{-}Cross}-\theta_t^*)
=
\frac1{\sqrt n}\sum_{i=1}^n\varepsilon_i
+
\frac{\sqrt n}{N}\sum_{j=1}^N
\{f_t(\widetilde X_j)-\E_t f_t(X)\}
+
o_p(1).
\]
The two leading empirical averages are based on independent gold and unlabeled
samples. The ordinary central limit theorem therefore gives
\[
\sqrt n(\widehat\mu_{\rm TC\text{-}Cross}-\theta_t^*)
\overset{D}{\longrightarrow}
N\!\left(
0,
\Var_t(\varepsilon)+\rho\Var_t\{f_t(X)\}
\right).
\]
\end{proof}

\begin{proof}[Proof of Corollary \ref{cor:tc-cross-lasso}]
Each training complement has size $(1-1/K)n$. Applying Proposition~\ref{thm:lasso-calibration-rig} with that training-sample size gives, for fixed $K$,
\[
(\widehat\delta^{(-k)}-\delta^*)^\top\Sigma_t(\widehat\delta^{(-k)}-\delta^*)
=O_p\!\left(\frac{s\log p}{n}\right)
\]
for each fold. Because $K$ is fixed, the maximum over folds has the same order. Thus $a_n=o_p(1)$ when $s\log p/n\to0$, and Theorem~\ref{thm:tc-cross-mean} applies.
\end{proof}

\subsection{Proofs for TC-Cross-PPI++ and Joint TC-Cross-PPI++}
\begin{proof}[Proof of Theorem~\ref{thm:tc-cross-ppipp}]
Under the assumed foldwise mean squared error convergence to $f_\dagger$, the same
cross-fitting remainder argument as in
Theorem~\ref{thm:tc-cross-mean} gives
\[
\widehat\mu_{{\rm TC\text{-}Cross},\lambda}-\theta_t^*
=
\frac1n\sum_{i=1}^n
\left[
Y_i-\lambda f_\dagger(X_i)
-\E_t\{Y-\lambda f_\dagger(X)\}
\right]
\]
\[
\qquad
+
\frac{\lambda}{N}\sum_{j=1}^N
\left[
f_\dagger(\widetilde X_j)-\E_t f_\dagger(X)
\right]
+o_p(n^{-1/2}).
\]
The first two terms again are centered empirical averages and are independent. Hence
\[
\sqrt n\bigl(
\widehat\mu_{{\rm TC\text{-}Cross},\lambda}
-\theta_t^*
\bigr)
\overset{D}{\longrightarrow}
N\!\left(0,V_{\rm cross}^\dagger(\lambda)\right),
\]
where
\[
V_{\rm cross}^\dagger(\lambda)
=
\Var_t\{Y-\lambda f_\dagger(X)\}
+
\rho\lambda^2\Var_t\{f_\dagger(X)\}.
\]
Expanding gives
\[
V_{\rm cross}^\dagger(\lambda)
=
\Var_t(Y)
-
2\lambda\Cov_t\{Y,f_\dagger(X)\}
+
(1+\rho)\lambda^2\Var_t\{f_\dagger(X)\}.
\]
Differentiating with respect to $\lambda$ therefore yields
\[
\lambda_\dagger^*
=
\frac{\Cov_t\{Y,f_\dagger(X)\}}
{(1+\rho)\Var_t\{f_\dagger(X)\}}.
\]
Substitution gives the minimized variance in the statement of the theorem.

If $f_\dagger=f_t=\E_t[Y\mid X]$, then
\[
\Cov_t\{Y,f_t(X)\}
=
\Cov_t\{\E_t(Y\mid X),f_t(X)\}
=
\Var_t\{f_t(X)\},
\]
and consequently
\[
\lambda_t^*=\frac1{1+\rho}.
\]
The result with a consistent substitution of $\widehat\lambda$ follows from Slutsky's theorem
and the fixed-dimensional foldwise law of large numbers.
\end{proof}

\begin{proof}[Proof of Proposition~\ref{prop:joint-cross-unbiased}]
For each fold $k$, define
\[
\widehat H^{(-k)}(x)
=
\begin{pmatrix}
f_s(x)\\
\widehat\Delta^{(-k)}(x)
\end{pmatrix}.
\]
Since the first coordinate does not depend on $k$,
\eqref{eq:joint-tc-cross} can be written as
\[
\widehat\mu_{\mathrm{JTC\text{-}Cross},\boldsymbol\lambda}
=
\frac1K\sum_{k=1}^K
\left[
\frac K n\sum_{i\in I_k}Y_i
+
\boldsymbol\lambda^\top
\left\{
\frac1N\sum_{j=1}^N
\widehat H^{(-k)}(\widetilde X_j)
-
\frac K n\sum_{i\in I_k}
\widehat H^{(-k)}(X_i)
\right\}
\right].
\]
For each fixed $k$, condition on $\mathcal D_s$ and the training complement
$\mathcal D_{-k}$. Then $\widehat H^{(-k)}$ is fixed, while the held-out
observations in $I_k$ have law $P_t$ and are independent of
$\mathcal D_{-k}$. The unlabeled observations are also independent and have
covariate law $P_{t,X}$. Therefore,
\[
\E\!\left[
\frac1N\sum_{j=1}^N
\widehat H^{(-k)}(\widetilde X_j)
\middle|
\mathcal D_s,\mathcal D_{-k}
\right]
=
\E_t\!\left[
\widehat H^{(-k)}(X)
\middle|
\mathcal D_s,\mathcal D_{-k}
\right],
\]
and
\[
\E\!\left[
\frac K n\sum_{i\in I_k}
\widehat H^{(-k)}(X_i)
\middle|
\mathcal D_s,\mathcal D_{-k}
\right]
=
\E_t\!\left[
\widehat H^{(-k)}(X)
\middle|
\mathcal D_s,\mathcal D_{-k}
\right].
\]
Also,
\[
\E\!\left[
\frac K n\sum_{i\in I_k}Y_i
\middle|
\mathcal D_s,\mathcal D_{-k}
\right]
=
\theta_t^*.
\]
Thus the conditional expectation of each fold-specific bracket is
$\theta_t^*$. Averaging over folds and applying iterated expectation yields
\[
\E\!\left[
\widehat\mu_{\mathrm{JTC\text{-}Cross},\boldsymbol\lambda}
\mid\mathcal D_s
\right]
=
\theta_t^*.
\]
\end{proof}

\begin{proof}[Proof of Theorem~\ref{thm:joint-tc-cross}]
For $i\in I_k$, write
\[
e_k(x)
\coloneqq 
\widehat\Delta^{(-k)}(x)-\Delta^\dagger(x).
\]
Then
\[
\widehat H_i^{\rm oof}
=
H_\dagger(X_i)
+
\begin{pmatrix}
0\\
e_k(X_i)
\end{pmatrix}.
\]
Similarly, the averaged unlabeled vector can be written as
\[
\widetilde H_N^{\rm CF}
=
\frac1N\sum_{j=1}^N H_\dagger(\widetilde X_j)
+
\begin{pmatrix}
0\\[0.3em]
\displaystyle
\frac1K\sum_{k=1}^K
\frac1N\sum_{j=1}^N e_k(\widetilde X_j)
\end{pmatrix}.
\]
Substituting these expressions into \eqref{eq:joint-tc-cross} gives
\begin{align}
\widehat\mu_{\mathrm{JTC\text{-}Cross},\boldsymbol\lambda}
-\theta_t^*
={}&
\frac1n\sum_{i=1}^n
\left[
Y_i-\boldsymbol\lambda^\top H_\dagger(X_i)
-\E_t\{Y-\boldsymbol\lambda^\top H_\dagger(X)\}
\right]
\nonumber\\
&+
\frac1N\sum_{j=1}^N
\left[
\boldsymbol\lambda^\top H_\dagger(\widetilde X_j)
-\E_t\{\boldsymbol\lambda^\top H_\dagger(X)\}
\right]
+
R_n(\boldsymbol\lambda),
\label{eq:joint-cross-proof-decomp}
\end{align}
where, writing
$\boldsymbol\lambda=(\lambda_s,\lambda_\Delta)^\top$,
\begin{equation}\label{eq:joint-cross-remainder}
R_n(\boldsymbol\lambda)
=
\frac{\lambda_\Delta}{K}
\sum_{k=1}^K
\left[
\frac1N\sum_{j=1}^N e_k(\widetilde X_j)
-
\frac K n\sum_{i\in I_k}e_k(X_i)
\right].
\end{equation}

We first show that the cross-fitting remainder is negligible. Define
\[
Z_{n,k}
\coloneqq 
\frac1N\sum_{j=1}^N e_k(\widetilde X_j)
-
\frac K n\sum_{i\in I_k}e_k(X_i).
\]
For each fixed $k$, condition on $\mathcal D_s$ and the training complement
$\mathcal D_{-k}$. Then $e_k$ is fixed, the observations in $I_k$ are
independent of $e_k$, and the unlabeled observations are independent of both
the training complement and the held-out observations. Moreover, the two
empirical averages defining $Z_{n,k}$ have the same conditional population
mean. Hence
\[
\E\!\left[
Z_{n,k}
\mid
\mathcal D_s,\mathcal D_{-k}
\right]
=0.
\]
Let
\[
a_{n,k}
\coloneqq 
\E_t\!\left[
e_k(X)^2
\mid
\mathcal D_s,\mathcal D_{-k}
\right].
\]
Using independence of the held-out and unlabeled samples conditional on
$\mathcal D_s,\mathcal D_{-k}$,
\begin{align*}
\Var\!\left(
Z_{n,k}
\mid
\mathcal D_s,\mathcal D_{-k}
\right)
&=
\frac1N
\Var_t\!\left(
e_k(X)
\mid
\mathcal D_s,\mathcal D_{-k}
\right)\\
&\quad+
\frac K n
\Var_t\!\left(
e_k(X)
\mid
\mathcal D_s,\mathcal D_{-k}
\right)\\
&\le
\left(
\frac1N+\frac K n
\right)a_{n,k}.
\end{align*}
Since $n/N\to\rho<\infty$ and $K$ is fixed,
\[
n\left(\frac1N+\frac K n\right)=O(1).
\]
Together with the assumption that fold-specific transfer corrections converge uniformly, conditional Chebyshev's
inequality therefore gives
\[
\sqrt n\,Z_{n,k}=o_p(1)
\]
for every fixed $k$. Because $K$ is fixed,
\[
\sqrt n\,R_n(\boldsymbol\lambda)
=
\frac{\lambda_\Delta}{K}
\sum_{k=1}^K
\sqrt n\,Z_{n,k}
=
o_p(1)
\]
for every fixed $\boldsymbol\lambda$. Notice that no independence across
different folds is required for this argument; the training complements may
overlap and the same unlabeled sample may be used in all folds.

It follows from \eqref{eq:joint-cross-proof-decomp} that
\begin{align*}
\sqrt n\left(
\widehat\mu_{\mathrm{JTC\text{-}Cross},\boldsymbol\lambda}
-\theta_t^*
\right)
={}&
\frac1{\sqrt n}\sum_{i=1}^n
\left[
Y_i-\boldsymbol\lambda^\top H_\dagger(X_i)
-\E_t\{Y-\boldsymbol\lambda^\top H_\dagger(X)\}
\right]\\
&+
\frac{\sqrt n}{N}\sum_{j=1}^N
\left[
\boldsymbol\lambda^\top H_\dagger(\widetilde X_j)
-\E_t\{\boldsymbol\lambda^\top H_\dagger(X)\}
\right]
+o_p(1).
\end{align*}
The two leading empirical averages are based on independent gold and
unlabeled samples. Under the stated $(2+\eta)$ moment conditions, the
Lindeberg--Feller central limit theorem therefore gives
\[
\sqrt n\left(
\widehat\mu_{\mathrm{JTC\text{-}Cross},\boldsymbol\lambda}
-\theta_t^*
\right)
\overset{D}{\longrightarrow}
N\!\left(
0,
V_{\rm Jcross}^\dagger(\boldsymbol\lambda)
\right),
\]
where
\[
V_{\rm Jcross}^\dagger(\boldsymbol\lambda)
=
\Var_t\{Y-\boldsymbol\lambda^\top H_\dagger(X)\}
+
\rho
\Var_t\{\boldsymbol\lambda^\top H_\dagger(X)\}.
\]

Now write
\[
\Sigma_\dagger
=
\Var_t\{H_\dagger(X)\},
\qquad
c_\dagger
=
\Cov_t\{H_\dagger(X),Y\}.
\]
Expanding the preceding variance yields
\[
V_{\rm Jcross}^\dagger(\boldsymbol\lambda)
=
\Var_t(Y)
-
2\boldsymbol\lambda^\top c_\dagger
+
(1+\rho)
\boldsymbol\lambda^\top
\Sigma_\dagger
\boldsymbol\lambda,
\]

If $\Sigma_\dagger$ is nonsingular, differentiation with respect to
$\boldsymbol\lambda$ gives
\[
(1+\rho)
\Sigma_\dagger
\boldsymbol\lambda_{\rm Jcross}^*
=
c_\dagger,
\]
and hence
\[
\boldsymbol\lambda_{\rm Jcross}^*
=
\frac1{1+\rho}
\Sigma_\dagger^{-1}c_\dagger.
\]
Substitution into the variance criterion gives
\[
V_{\rm Jcross}^{\dagger,*}
=
\Var_t(Y)
-
\frac{
c_\dagger^\top
\Sigma_\dagger^{-1}
c_\dagger
}{
1+\rho
}.
\]

If $\Sigma_\dagger$ is singular, let $v$ belong to its null space. Then
\[
\Var_t\{v^\top H_\dagger(X)\}
=
v^\top\Sigma_\dagger v
=
0,
\]
so $v^\top H_\dagger(X)$ is almost surely constant. Therefore
\[
v^\top c_\dagger
=
\Cov_t\{v^\top H_\dagger(X),Y\}
=
0.
\]
Hence $c_\dagger$ belongs to the column space of $\Sigma_\dagger$, and the
quadratic criterion has a minimizer. Its minimum-norm minimizer is
\[
\boldsymbol\lambda_{\rm Jcross}^*
=
\frac1{1+\rho}
\Sigma_\dagger^\dagger c_\dagger,
\]
and the minimized variance is
\begin{equation}
V_{\rm Jcross}^{\dagger,*}
=
\Var_t(Y)
-
\frac{
c_\dagger^\top
\Sigma_\dagger^\dagger
c_\dagger
}{
1+\rho
},
\label{eq:joint-cross-min-var}
\end{equation}

For the no-harm comparisons, the joint parameter space
$\mathbb R^2$ contains the source-PPI++ line
\[
\{(\lambda,0)^\top:\lambda\in\mathbb R\},
\]
for which
\[
\boldsymbol\lambda^\top H_\dagger(X)
=
\lambda f_s(X),
\]
and the scalar TC--Cross-PPI++ line
\[
\{(\lambda,\lambda)^\top:\lambda\in\mathbb R\},
\]
for which
\[
\boldsymbol\lambda^\top H_\dagger(X)
=
\lambda\{f_s(X)+\Delta^\dagger(X)\}
=
\lambda f_\dagger(X).
\]
The joint class also contains $(0,0)^\top$, which gives classical inference.
Therefore minimizing over the full joint class cannot yield a larger
first-order variance than minimizing over either scalar subfamily or choosing
classical inference. This proves
\eqref{eq:joint-cross-dominance} without requiring
$\Delta^\dagger=\Delta^*$.

Finally, suppose that
\[
\Delta^\dagger=\Delta^*,
\qquad
f_\dagger=f_t=\E_t[Y\mid X].
\]
Let
\[
a=(1,1)^\top,
\]
so that
\[
f_t(X)=a^\top H_\dagger(X).
\]
Because $H_\dagger(X)$ is measurable with respect to $X$ and
$f_t(X)=\E_t[Y\mid X]$,
\[
c_\dagger
=
\Cov_t\{H_\dagger(X),Y\}
=
\Cov_t\{H_\dagger(X),f_t(X)\}
=
\Sigma_\dagger a.
\]
Consequently,
\begin{align*}
c_\dagger^\top
\Sigma_\dagger^\dagger
c_\dagger
&=
a^\top
\Sigma_\dagger
\Sigma_\dagger^\dagger
\Sigma_\dagger
a\\
&=
a^\top\Sigma_\dagger a\\
&=
\Var_t\{f_t(X)\}.
\end{align*}
Substitution into \eqref{eq:joint-cross-min-var} gives
\[
V_{\rm Jcross}^{\dagger,*}
=
\Var_t(Y)
-
\frac{\Var_t\{f_t(X)\}}{1+\rho},
\]

It remains to justify feasible substitution. Define
\[
D_n
\coloneqq 
\widetilde H_N^{\rm CF}
-
\frac1n\sum_{i=1}^n\widehat H_i^{\rm oof}.
\]
The same decomposition used above gives
\[
D_n
=
\left\{
\frac1N\sum_{j=1}^N H_\dagger(\widetilde X_j)
-
\frac1n\sum_{i=1}^nH_\dagger(X_i)
\right\}
+
o_p(n^{-1/2}),
\]
componentwise, and therefore
\[
\sqrt n\,D_n=O_p(1).
\]
Since the estimator in \eqref{eq:joint-tc-cross} is affine in
$\boldsymbol\lambda$,
\begin{align*}
\sqrt n\Bigl\{
\widehat\mu_{\mathrm{JTC\text{-}Cross},
\widehat{\boldsymbol\lambda}}
-
\widehat\mu_{\mathrm{JTC\text{-}Cross},
\boldsymbol\lambda_{\rm Jcross}^*}
\Bigr\}
&=
\sqrt n
\left(
\widehat{\boldsymbol\lambda}
-
\boldsymbol\lambda_{\rm Jcross}^*
\right)^\top
D_n\\
&=
o_p(1)
\end{align*}
whenever
\[
\widehat{\boldsymbol\lambda}
-
\boldsymbol\lambda_{\rm Jcross}^*
=o_p(1).
\]
Thus, substituting a consistent feasible weight estimator leaves the same
first-order limiting distribution and variance.
\end{proof}
%\begin{comment}
\begin{proof}[Proof of Proposition~\ref{prop:joint-weight-rate}]
Write
\[
H_i^\dagger\coloneqq H_\dagger(X_i),
\qquad
q_i\coloneqq 
\widehat H_i^{\rm oof}-H_i^\dagger.
\]
If $i\in I_k$, then
\[
q_i
=
\begin{pmatrix}
0\\
\widehat\Delta^{(-k)}(X_i)-\Delta^\dagger(X_i)
\end{pmatrix}
=
\begin{pmatrix}
0\\
e_k(X_i)
\end{pmatrix}.
\]
Let
\[
\overline H_\dagger
=
\frac1n\sum_{i=1}^n H_i^\dagger,
\qquad
\overline q
=
\frac1n\sum_{i=1}^n q_i,
\qquad
\overline H_{\rm oof}
=
\overline H_\dagger+\overline q.
\]

We first control the empirical size of the foldwise transfer errors. For each
fold $k$, conditional on $\mathcal D_s$ and $\mathcal D_{-k}$,
\[
\E\!\left[
\frac K n\sum_{i\in I_k}e_k(X_i)^2
\middle|
\mathcal D_s,\mathcal D_{-k}
\right]
=
\E_t[e_k(X)^2\mid\mathcal D_s,\mathcal D_{-k}]
=
O_p(r_n).
\]
Conditional Markov's inequality and fixed $K$ therefore imply
\[
\frac1n\sum_{i=1}^n\|q_i\|_2^2
=
O_p(r_n).
\]
Consequently,
\begin{equation}\label{eq:proof-qbar}
\|\overline q\|_2
\le
\left(
\frac1n\sum_{i=1}^n\|q_i\|_2^2
\right)^{1/2}
=
O_p(\sqrt{r_n}).
\end{equation}

Let
\[
\widetilde\Sigma_H
=
\frac1{n-1}\sum_{i=1}^n
(H_i^\dagger-\overline H_\dagger)
(H_i^\dagger-\overline H_\dagger)^\top
\]
denote the empirical covariance matrix computed from the limiting prediction
vector $H_\dagger$. Since $H_\dagger(X)$ has fixed dimension and finite fourth
moments,
\begin{equation}\label{eq:proof-oracle-Sigma}
\|\widetilde\Sigma_H-\Sigma_\dagger\|_{\rm op}
=
O_p(n^{-1/2}).
\end{equation}
Similarly, if
\[
\widetilde c_H
=
\frac1{n-1}\sum_{i=1}^n
(H_i^\dagger-\overline H_\dagger)
(Y_i-\overline Y_n),
\]
then
\begin{equation}\label{eq:proof-oracle-c}
\|\widetilde c_H-c_\dagger\|_2
=
O_p(n^{-1/2}).
\end{equation}

We next compare the out-of-fold covariance estimators with these oracle
empirical quantities. Since
\[
\widehat H_i^{\rm oof}-\overline H_{\rm oof}
=
(H_i^\dagger-\overline H_\dagger)
+
(q_i-\overline q),
\]
expansion gives
\begin{align*}
\widehat\Sigma_H-\widetilde\Sigma_H
={}&
\frac1{n-1}\sum_{i=1}^n
(H_i^\dagger-\overline H_\dagger)
(q_i-\overline q)^\top\\
&+
\frac1{n-1}\sum_{i=1}^n
(q_i-\overline q)
(H_i^\dagger-\overline H_\dagger)^\top\\
&+
\frac1{n-1}\sum_{i=1}^n
(q_i-\overline q)(q_i-\overline q)^\top.
\end{align*}
By Cauchy--Schwarz,
\[
\left\|
\frac1{n-1}\sum_{i=1}^n
(H_i^\dagger-\overline H_\dagger)
(q_i-\overline q)^\top
\right\|_{\rm op}
\le
\left[
\frac1{n-1}\sum_{i=1}^n
\|H_i^\dagger-\overline H_\dagger\|_2^2
\right]^{1/2}
\left[
\frac1{n-1}\sum_{i=1}^n
\|q_i-\overline q\|_2^2
\right]^{1/2}.
\]
The first factor is $O_p(1)$ by the moment assumptions, while the second is
$O_p(\sqrt{r_n})$ by the preceding bound. The transpose term has the same
order, and
\[
\left\|
\frac1{n-1}\sum_{i=1}^n
(q_i-\overline q)(q_i-\overline q)^\top
\right\|_{\rm op}
=
O_p(r_n).
\]
Hence
\begin{equation}\label{eq:proof-Sigma-perturb}
\|\widehat\Sigma_H-\widetilde\Sigma_H\|_{\rm op}
=
O_p(\sqrt{r_n}+r_n)
=
O_p(\sqrt{r_n}),
\end{equation}
where the last equality uses $r_n=o(1)$.

For the covariance with the outcome,
\[
\widehat c_H-\widetilde c_H
=
\frac1{n-1}\sum_{i=1}^n
(q_i-\overline q)(Y_i-\overline Y_n).
\]
Again by Cauchy--Schwarz,
\begin{align*}
\|\widehat c_H-\widetilde c_H\|_2
&\le
\left[
\frac1{n-1}\sum_{i=1}^n
\|q_i-\overline q\|_2^2
\right]^{1/2}
\left[
\frac1{n-1}\sum_{i=1}^n
(Y_i-\overline Y_n)^2
\right]^{1/2}\\
&=
O_p(\sqrt{r_n}).
\end{align*}
Combining this with \eqref{eq:proof-oracle-Sigma}--\eqref{eq:proof-oracle-c}
gives
\[
\|\widehat\Sigma_H-\Sigma_\dagger\|_{\rm op}
+
\|\widehat c_H-c_\dagger\|_2
=
O_p\!\left(
n^{-1/2}+\sqrt{r_n}
\right).
\]

Now put
\[
\alpha_n
\coloneqq 
n^{-1/2}+\sqrt{r_n}+\tau.
\]
Since $r_n=o(1)$ and $\tau=o(1)$,
\[
\|\widehat\Sigma_H+\tau I_2-\Sigma_\dagger\|_{\rm op}
=
O_p(\alpha_n)
=
o_p(1).
\]
The assumption
$\lambda_{\min}(\Sigma_\dagger)\ge\kappa_H>0$
therefore implies
\[
\|(\widehat\Sigma_H+\tau I_2)^{-1}\|_{\rm op}
=
O_p(1).
\]
Using the resolvent identity,
\[
(\widehat\Sigma_H+\tau I_2)^{-1}
-
\Sigma_\dagger^{-1}
=
(\widehat\Sigma_H+\tau I_2)^{-1}
\{\Sigma_\dagger-\widehat\Sigma_H-\tau I_2\}
\Sigma_\dagger^{-1},
\]
and hence
\[
\left\|
(\widehat\Sigma_H+\tau I_2)^{-1}
-
\Sigma_\dagger^{-1}
\right\|_{\rm op}
=
O_p(\alpha_n).
\]
Therefore
\begin{align*}
\left\|
\widehat{\boldsymbol\lambda}
-
\boldsymbol\lambda_{{\rm J},n}^\dagger
\right\|_2
&\le
\frac1{1+n/N}
\Big[
\|(\widehat\Sigma_H+\tau I_2)^{-1}\|_{\rm op}
\|\widehat c_H-c_\dagger\|_2\\
&\qquad\qquad+
\left\|
(\widehat\Sigma_H+\tau I_2)^{-1}
-
\Sigma_\dagger^{-1}
\right\|_{\rm op}
\|c_\dagger\|_2
\Big]\\
&=
O_p\!\left(
n^{-1/2}+\sqrt{r_n}+\tau
\right),
\end{align*}
which proves \eqref{eq:joint-weight-rate}.

It remains to evaluate the excess population variance criterion. By
\eqref{eq:joint-finite-pop-lambda},
\[
c_\dagger
=
(1+n/N)\Sigma_\dagger
\boldsymbol\lambda_{{\rm J},n}^\dagger.
\]
Completing the square in \eqref{eq:joint-finite-pop-criterion} therefore gives
the exact identity
\begin{align*}
&
V_{{\rm J},n}^\dagger(\boldsymbol\lambda)
-
V_{{\rm J},n}^\dagger(
\boldsymbol\lambda_{{\rm J},n}^\dagger)\\
&\qquad=
(1+n/N)
\left(
\boldsymbol\lambda
-
\boldsymbol\lambda_{{\rm J},n}^\dagger
\right)^\top
\Sigma_\dagger
\left(
\boldsymbol\lambda
-
\boldsymbol\lambda_{{\rm J},n}^\dagger
\right).
\end{align*}
Taking
$\boldsymbol\lambda=\widehat{\boldsymbol\lambda}$
and using the bounded eigenvalues of the fixed-dimensional
$\Sigma_\dagger$ yields
\[
V_{{\rm J},n}^\dagger(\widehat{\boldsymbol\lambda})
-
V_{{\rm J},n}^\dagger(
\boldsymbol\lambda_{{\rm J},n}^\dagger)
=
O_p\!\left(
n^{-1}+r_n+\tau^2
\right),
\]
which proves \eqref{eq:joint-excess-criterion}.
\end{proof}
%\end{comment}
\section{Experimental Details}\label{app:experimental_setup}

For each dataset, Table~\ref{tab:dataset-stats} gives the population size $N$ from which gold labels are sampled, the swept range of gold-label counts $n$, and the repetition/cross-fitting protocol. Each repetition reveals $n$ rows' labels uniformly at random and treats the remaining $N-n$ rows as the prediction-rich unlabeled pool. The unlabeled share $(N-n)/N$ at the largest gold-label budget swept ranges from about $78\%$ (AutoEval MT-Bench Arena, $n=750$ of $N=3355$) to over $99\%$ (Census, AutoEval-ProteinGym); it is smallest for the two applications where $n$ is swept up to a sizable fraction of a comparatively small population (AutoEval MT-Bench Arena, ProteinGym).

Each application below gives its dataset, the source predictor that supplies the score, and the sparse calibration features appended to that score. Common setup details are in Appendix~\ref{app:implementation}.

\begin{table}[ht]
  \centering
  \caption{Population size ($N + n$), gold-label range ($n$), repetition count, and cross-fitting fold count ($K$). $N$ rows are unlabeled in each repetition. Parenthesized values give the LoRA variant's count where it differs from the sparse-calibration default shown (if applicable).}
  \label{tab:dataset-stats}
  \vspace{1mm}
  \begin{tabular}{lrccc}
    \toprule
    Application & Population ($N + n$) & Gold labels ($n$) & Repetitions & $K$ \\
    \midrule
    Census mean / OLS & 380,091 & 100--2000 & 200 & 5 \\
    Census healthcare & 318,215 & 100--2000 & 100 & 5 \\
    Capital Bikeshare & 8,734~\tablefootnote{Target-year (2012) population. The 8,645 source-year (2011) rows are used only to fit the source predictor.} & 50--500 & 100 & 5 \\
    ASR accent transfer & 15,378 & 50--1000 & 250 & 5 \\
    Camelyon17 (WILDS) & 119{,}958~\tablefootnote{Pooled held-out-hospital population ($85{,}054$ test plus $34{,}904$ validation patches). The $335{,}996$ training-hospital patches, including the in-distribution validation split, are used only to fit and select the source predictor.} & 25--500 & 200 & 5 \\
    PovertyMap (WILDS) & 7,872~\tablefootnote{Fold~A's pooled OOD-country population (3,963 test plus 3,909 validation clusters). The 10,796 training- and in-distribution-validation-country clusters are used only to fit and select the source predictor.} & 25--500 & 200 & 5 \\
    % Capital Bikeshare & 8,734$^{a}$ & 50--500 & 100 & 5 \\
    % ASR accent transfer & 15,378$^{b}$ & 50--1000 & 250 & 5 \\
    % Camelyon17 (WILDS) & 119{,}958$^{c}$ & 25--500 & 200 & 5 \\
    % PovertyMap (WILDS) & 7,872$^{d}$ & 25--500 & 200 & 5 \\
    
    ProteinGym (BLAT) & 4,783 & 100--1000 & 100 (20) & 5 \\
    AutoEval-ProteinGym (SPG1) & 536,962 & 50--1000 & 250 (20) & 5 (2) \\

    Chatbot Arena & 14,379 & 200--2000 & 50 (20) & 5 (2) \\
    Chatbot Arena (judge quality) & 14,379 & 500 & 250 & 5 \\
    AutoEval MT-Bench Arena & 3,355 & 25--750 & 250 (25) & 5 (3) \\
    \bottomrule
\end{tabular}
\par\smallskip
% \begin{minipage}{\linewidth}\footnotesize
% $^{a}$Target-year (2012) population. An additional 8,645 source-year (2011) rows are used only to fit the source predictor and are never sampled as gold or unlabeled.\\

% $^{b}$Gold labels are drawn only from the train split (14,369 rows); the dev and test splits are always unlabeled.\\

% $^{c}$Pooled held-out-hospital population (85,054 test plus 34,904 validation patches). The 335,996 training-hospital patches, including the in-distribution validation split, are used only to fit and select the source predictor and are never sampled as gold or unlabeled.\\

% $^{d}$Fold~A's pooled OOD-country population (3,963 test plus 3,909 validation clusters). The 10,796 training- and in-distribution-validation-country clusters are used only to fit and select the source predictor and are never sampled as gold or unlabeled.
% \end{minipage}
\end{table}

\subsection{Tabular}

\noindent\textbf{Census.} 2019 American Community Survey microdata for California via the \texttt{folktables} package~\citep{ding2021retiring}, covering three targets: mean income, an income-regression coefficient, and private health-insurance coverage.

\noindent\textit{Source predictor:} XGBoost~\citep{chen2016xgboost} predictions, precomputed and shipped by the upstream \texttt{ppi\_py} benchmark rather than trained in this codebase.

\noindent\textit{Calibration features:} raw tabular covariates plus the source prediction, together with degree-two interactions and squares of its two covariates ($6$ features total, shared across all three Census applications).

\noindent\textbf{Capital Bikeshare.} UCI's hourly Capital Bikeshare rental counts for Washington, D.C.~\citep{fanaeet2014event}; a temporal-transfer benchmark predicting 2012 rentals from a source trained only on 2011 data.

\noindent\textit{Source predictor:} a gradient-boosted regressor trained on the 2011 data: learning rate $0.06$, $300$ boosting rounds, at most $31$ leaf nodes per tree, $\ell_2$ regularization $1.0$, all other settings at defaults).

\noindent\textit{Calibration features:} raw tabular covariates plus the source prediction, one-hot-encoding the categorical covariates alongside the standardized continuous ones with no polynomial expansion ($60$ features total).

\subsection{Speech}

\noindent\textbf{AfriSpeech-200.} African-accented English speech~\citep{olatunji2023afrispeech}; we use the Yoruba-accented portion for a word-error-rate transfer task, pooling its train, dev and test splits into a single target population (Table~\ref{tab:dataset-stats}).

\noindent\textit{Source predictor:} wav2vec2-base-960h~\citep{baevski2020wav2vec20frameworkselfsupervised} with zero-shot greedy CTC decoding (no beam search), trained on American-English speech and applied directly to the Yoruba-accented audio with no fine-tuning.

\noindent\textit{Calibration features:} the source WER prediction, audio duration, hypothesis word count, speaking rate, and one-hot speaker gender ($7$ features total, three gender categories observed).

\subsection{WILDS benchmarks}

\subsubsection{Camelyon17}

The WILDS~\citep{koh2021wilds} version of the CAMELYON17 lymph-node histopathology challenge~\citep{bandi2018detection}: $455{,}954$ $96\times96$ RGB patches from $50$ slides across five hospitals, labeled $1$ if the patch's central $32\times32$ region contains tumor tissue. Three hospitals are used for training and the two held-out ones, the OOD validation and OOD test hospitals, are pooled into a single target population. Pooling is sound because we select the source checkpoint on in-distribution data, so neither held-out hospital informs training or model selection; it also makes the target a two-hospital mixture rather than a single unseen hospital. WILDS's own CodaLab download host was unreachable, so we load the \texttt{wltjr1007/Camelyon17-WILDS} Hugging Face mirror and assign hospitals from its per-patch \texttt{center} field, which reproduces WILDS's official split sizes exactly ($302{,}436$ training, $33{,}560$ in-distribution validation, $34{,}904$ OOD validation and $85{,}054$ test patches).

\noindent\textit{Source predictor:} DenseNet-121~\citep{huang2017densely} trained from scratch on the three training hospitals, following WILDS's ERM baseline (binary cross-entropy; SGD with momentum $0.9$, learning rate $10^{-3}$, weight decay $10^{-2}$; batch size $32$; $10$ epochs; random flips). Its sigmoid probability is the source score. The selected checkpoint reaches $0.968$ and $0.966$ accuracy on the training hospitals and in-distribution validation, against $0.787$ and $0.673$ on the two held-out hospitals.

\noindent\textit{Calibration features:} the source probability, the patch's $(x,y)$ coordinates within its slide, and one-hot indicators for the $20$ slides of the two target hospitals ($23$ features total). The slide indicators let the correction absorb slide-to-slide staining variation within and between the target hospitals.

\subsubsection{PovertyMap}

The WILDS~\citep{koh2021wilds} version of the PovertyMap benchmark~\citep{yeh2020using}: $224\times224$ multispectral satellite images (seven Landsat bands plus a nightlights band) for $19{,}669$ DHS survey clusters in 23 African countries, labeled with a real-valued asset-wealth index. We use WILDS's country fold~A, whose ten out-of-distribution countries we pool into a single target population: the OOD test group (Angola, C\^ote d'Ivoire, Ethiopia, Mali, Rwanda) and the OOD validation group (Benin, Burkina Faso, Guinea, Sierra Leone, Tanzania). Doing so doubles the target population and the number of countries the correction can key on. WILDS's own CodaLab download host was unreachable, so we load the \texttt{1aurent/PovertyMap} Hugging Face mirror, which also ships the survey metadata (urban/rural, nightlight summaries, country, year) as ordinary columns. The loader identifies the fold by matching the mirror's OOD country sets against WILDS's five fold definitions.

\noindent\textit{Source predictor:} ResNet18~\citep{he2016deep} with its first convolution widened to the eight multispectral bands, trained from scratch on fold~A's training countries, following WILDS's ERM baseline (mean squared error; Adam with learning rate $10^{-3}$; step decay of $0.96$ per epoch; batch size $64$; up to $200$ epochs; random flips). It reaches correlations of $0.902$ and $0.868$ on the training countries and in-distribution validation, against $0.830$ on the pooled held-out countries.

\noindent\textit{Calibration features:} the source prediction, the urban/rural indicator, the survey's two nightlight summaries (\texttt{nl\_mean} and \texttt{nl\_center}), and one-hot indicators for the ten target countries ($14$ features total). The mirror's binarized wealth column is never used, as it is derived from the target.

\subsubsection{Model selection}

Both WILDS source predictors keep the checkpoint that scores best on the \emph{in-distribution} validation split rather than the out-of-distribution one WILDS's own protocol uses, because our setting assumes no target-domain labels at training time; this is also what makes pooling both held-out groups into a single target leak-free (\texttt{--selection-split val} recovers WILDS's protocol). One consequence is that our Camelyon17 source reaches $0.673$ on the OOD test hospital against WILDS's reported ERM baseline of $70.3\%$, a benchmark \citet{koh2021wilds} single out for its $6.4\%$ seed-to-seed standard deviation there.

\subsection{Protein language models}
\label{app:models-features}

\noindent\textbf{ProteinGym.} The BLAT\_ECOLX\_Firnberg\_2014 deep-mutational-scanning assay from the ProteinGym benchmark~\citep{notin2023proteingym}, measuring the fitness effect of amino-acid substitutions in a beta-lactamase enzyme. It also includes a protein-model-ranking target and an annotator-quality sweep.

\noindent\textit{Source predictor:} ESM-2 8M~\citep{lin2023evolutionary} zero-shot masked-marginal scores, with no training.

\noindent\textit{Candidate models:} three ESM-2 checkpoints~\citep{lin2023evolutionary} (8M, 35M, and 150M parameters), which are ranked in the protein-model-ranking target and swapped in as the annotator in the annotator-quality sweep.

\noindent\textit{Calibration features:} one-hot mutation identity (position and substituted amino acid) with the ESM-2 score ($47$ features total, for the $286$-residue BLAT assay with all $20$ canonical amino acids observed).

\noindent\textit{LoRA calibration:} fine-tunes the same ESM-2 8M backbone that produces the source score, on the raw mutated sequence, with the source score as a second linear term.

\noindent\textbf{AutoEval-ProteinGym.} A replication of AutoEval's~\citep{boyeau2024autoeval} released SPG1 experiment (the STRSG assay). This replication follows AutoEval's own convention of $90\%$ nominal coverage, rather than the $95\%$ used elsewhere in this paper.

\noindent\textit{Source predictor:} a precomputed VESPA~\citep{marquet2021embeddings} annotator score.

\noindent\textit{Candidate models:} the seven zero-shot models AutoEval compares, using AutoEval's released scores: UniRep~\citep{alley2019unified}, CARP (640M)~\citep{yang2024convolutions}, ESM-1b~\citep{rives2021biological}, ESM-1v (five-model ensemble)~\citep{meier2021language}, ESM-2 (15B)~\citep{lin2023evolutionary}, RITA (XL)~\citep{hesslow2022rita}, and ProGen2 (XLarge)~\citep{nijkamp2023progen2}. The annotator-quality sweep uses each of these as the annotator in turn, plus VESPA.

\noindent\textit{Calibration features:} per-site biochemical-property deltas (hydrophobicity, volume, polarity, charge, aromaticity) and epistasis proxies with the VESPA score ($37$ features total).

\noindent\textit{LoRA calibration:} fine-tunes a separately loaded ESM-2 8M encoder on the raw mutant sequence, with the VESPA score as an auxiliary input -- source and LoRA backbone are different models here. %Unlabeled scoring is capped at 200{,}000 of the assay's 537{,}000 rows.

\subsection{LLM arenas}

\noindent\textbf{Chatbot Arena.} Human pairwise preferences from Chatbot Arena~\citep{chiang2024chatbot}, paired with precomputed public LLM-judge scores, restricted to the most frequently compared models.

\noindent\textit{Source predictor:} the Meta-Llama-3-8B-Instruct judge's precomputed pairwise-preference probability (human agreement $62.9\%$).

\noindent\textit{Calibration features:} a small metadata layer only -- source logit, prompt/response character and word counts, response-length contrast, and one-hot model-identity dummies for both compared models ($47$ features total among the $20$ most-compared models); a frozen-sentence-embedding feature mode also exists in code but was not used for the reported run.

\noindent\textbf{AutoEval MT-Bench Arena.} A replication of AutoEval's~\citep{boyeau2024autoeval} released MT-Bench~\citep{zheng2023judging} experiment, matching human ratings to GPT-4's pairwise judgments across six chat systems. This replication follows AutoEval's own convention of $90\%$ nominal coverage, rather than the $95\%$ used elsewhere in this paper.

\noindent\textit{Source predictor:} GPT-4's own precomputed pairwise judgment, as released.

\noindent\textit{Calibration features:} its Bradley--Terry design matrix and a small metadata layer (source logit, response-length contrast, and turn; $8$ features total among the six chat systems) combined with frozen MiniLM-L6-v2 sentence embeddings of the prompt and a response-embedding contrast ($2\times384=768$ dimensions), for $776$ features total. Question-level fixed effects are omitted in both arena applications since the smallest gold-sample sizes leave too few observations per calibration fold.

\noindent\textit{LoRA calibration (both arenas):} both replace their (unfine-tunable) source entirely with \texttt{distilgpt2}~\citep{sanh2019distilbert}, which supplies its own zero-shot pairwise-preference score and is LoRA-fine-tuned as the correction model on the same formatted prompt/response text. In both cases this also makes the sparse-vs-LoRA comparison an apples-to-apples test of calibration mechanism: each LoRA-comparison's sparse arm is given \texttt{distilgpt2}'s own frozen, pooled hidden state over the identical text, so the only difference between the sparse and LoRA arms is the calibration mechanism itself, not which encoder supplies the features.
% ---rather than the metadata/MiniLM-embedding features used in the corresponding main-body result above---

\subsection{Hyperparameters and Implementation}
\label{app:implementation}
To promote reproducibility, the source code is provided in the supplementary material, along with the SLURM script files used to launch the jobs.

\textbf{Sparse-linear calibration.} Continuous-outcome applications (Census mean/OLS, Capital Bikeshare, PovertyMap, ProteinGym), and Camelyon17, whose binary label enters a mean-estimation residual correction, select their calibrator's $\ell_1$ penalty by $5$-fold cross-validation over scikit-learn's automatic log-spaced grid (\texttt{LassoCV}). AutoEval-ProteinGym instead fixes the penalty ($\alpha=0.01$, no cross-validation) given the cost of repeated re-fits over its much larger unlabeled pool. Binary/preference-outcome applications (Census healthcare, Chatbot Arena, AutoEval MT-Bench Arena) calibrate with a fixed, smoothed-$\ell_1$-penalized offset-logistic correction ($\alpha=10^{-3}$ for Census healthcare and Chatbot Arena, $\alpha=0.01$ for AutoEval MT-Bench Arena) rather than cross-validation, which we found unnecessary for stability at this scale and considerably slower under the offset-logistic objective.

\textbf{Joint power-tuning ridge.} The joint estimator's multi-source weight solve is stabilized with a small, fixed ridge term added to its Hessian before inversion: a multiplier of $10^{-8}$ (relative to $\mathrm{trace}(\Sigma_H)/M$) for the mean estimator, and $10^{-6}$ for the logistic and OLS joint estimators. This constant is fixed across every experiment purely for numerical conditioning of the weight solve -- it is never tuned from data or exposed as a command-line option.

\noindent\textbf{Confidence intervals.} All cross-fitted estimators use Wald intervals $\hat\mu\pm z_{1-\alpha/2}\widehat{\mathrm{SE}}$. For TC--Cross-PPI(++), $\widehat{\mathrm{SE}}^2=\widehat{\operatorname{Var}}(Y-\hat\lambda\hat f_t^{\,\mathrm{oof}})/n+\hat\lambda^2\,\widehat{\operatorname{Var}}(\bar f_t)/N$, where $\hat f_t^{\,\mathrm{oof}}$ is the out-of-fold calibrated prediction on the gold sample and $\bar f_t$ the fold-averaged prediction on the unlabeled sample. Joint TC--Cross-PPI++ is the same with $\hat\lambda\hat f_t$ replaced by the weighted combination $\hat\lambda_s f_s+\hat\lambda_\Delta\hat\Delta$ of the source prediction and the correction. The weights are estimated from the same out-of-fold quantities, with no additional split, and are treated as fixed in $\widehat{\mathrm{SE}}$.

\textbf{LoRA hyperparameters.} Table~\ref{tab:lora-hparams} gives the adapter configuration for each of the four LoRA experiments; all four additionally include the source score as an auxiliary linear input to the correction and pool the backbone's final hidden states by mean pooling.

\begin{table}[ht]
  \centering
  \small
  \caption{LoRA adapter hyperparameters by application. Target modules follow each backbone's attention-layer naming (\texttt{query}/\texttt{value} for ESM-2, \texttt{c\_attn} for GPT-2-family models). Both arena applications share \texttt{distilgpt2} as backbone.}
  \label{tab:lora-hparams}
  \resizebox{\textwidth}{!}{%
  \begin{tabular}{lcccccccc}
    \toprule
    Application & $r$ & $\alpha$ & Dropout & Target modules & Epochs & LR & Batch (train/predict) & Max length \\
    \midrule
    ProteinGym (BLAT) & 1 & 16 & 0.05 & query, value & 100 & $10^{-3}$ & 256 / 256 & --- \\
    AutoEval-ProteinGym & 2 & 16 & 0.05 & query, value & 100 & $10^{-3}$ & 512 / 512 & --- \\
    Chatbot Arena & 2 & 16 & 0.05 & c\_attn & 100 & $10^{-3}$ & 16 / 64 & 1024 \\
    AutoEval MT-Bench Arena & 2 & 16 & 0.05 & c\_attn & 100 & $10^{-3}$ & 16 / 64 & 1024 \\
    \bottomrule
  \end{tabular}
  }
\end{table}

\textbf{Compute Resources:} Non-LoRA experiments, and those that don't require training a source model, are run on up to 8 Intel Xeon 8268s. Experiments that score or fine-tune a neural network (i.e., Camelyon17's and PovertyMap's source training, ProteinGym, AutoEval-ProteinGym LoRA, the Chatbot Arena and AutoEval MT-Bench Arena LoRA variants, and ASR accent transfer), additionally use a single V100, or A100 for LoRA experiments.

\section{Additional Results}\label{app:additional_results}
% Figures~\ref{fig:empirical-ess-full} and \ref{fig:empirical-width-full} repeat the main-text panel figures (Section~\ref{sec:empirical}) across all datasets, including ProteinGym and Census Income (OLS), which are discussed only in this appendix. Figure~\ref{fig:empirical-coverage} reports the empirical coverages across all datasets. 

Figures~\ref{fig:empirical-ess-full}-~\ref{fig:empirical-coverage} report the empirical ESS, interval widths, and coverage across all datasets, including ProteinGym and Census Income (OLS), which are discussed only in this appendix. 

\begin{figure}[H]
  \centering
  \includegraphics[width=0.9\textwidth]{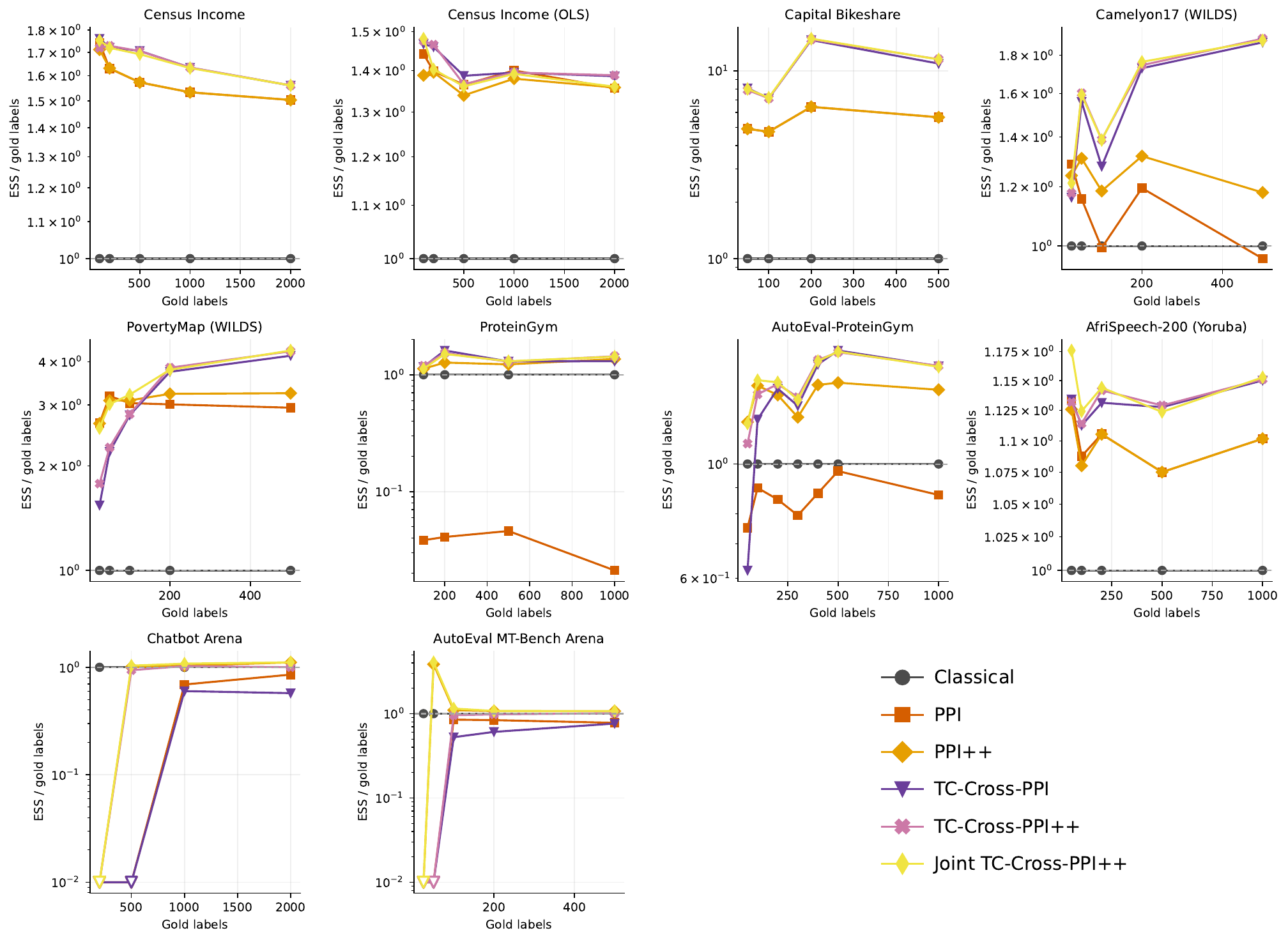}
  \caption{Empirical ESS divided by the number of gold labels ($n$), on a logarithmic $y$-axis. The dashed line is classical inference. Values above one indicate MSE improvement over the gold-only estimator. Values below $0.01$ are drawn at $0.01$ as open downward triangles; they occur only for the two Bradley--Terry arena applications at their smallest label budgets, where quasi-separation makes the non-classical estimators unstable.}
  \label{fig:empirical-ess-full}
\end{figure}

\begin{figure}[H]
  \centering
  \includegraphics[width=0.9\textwidth]{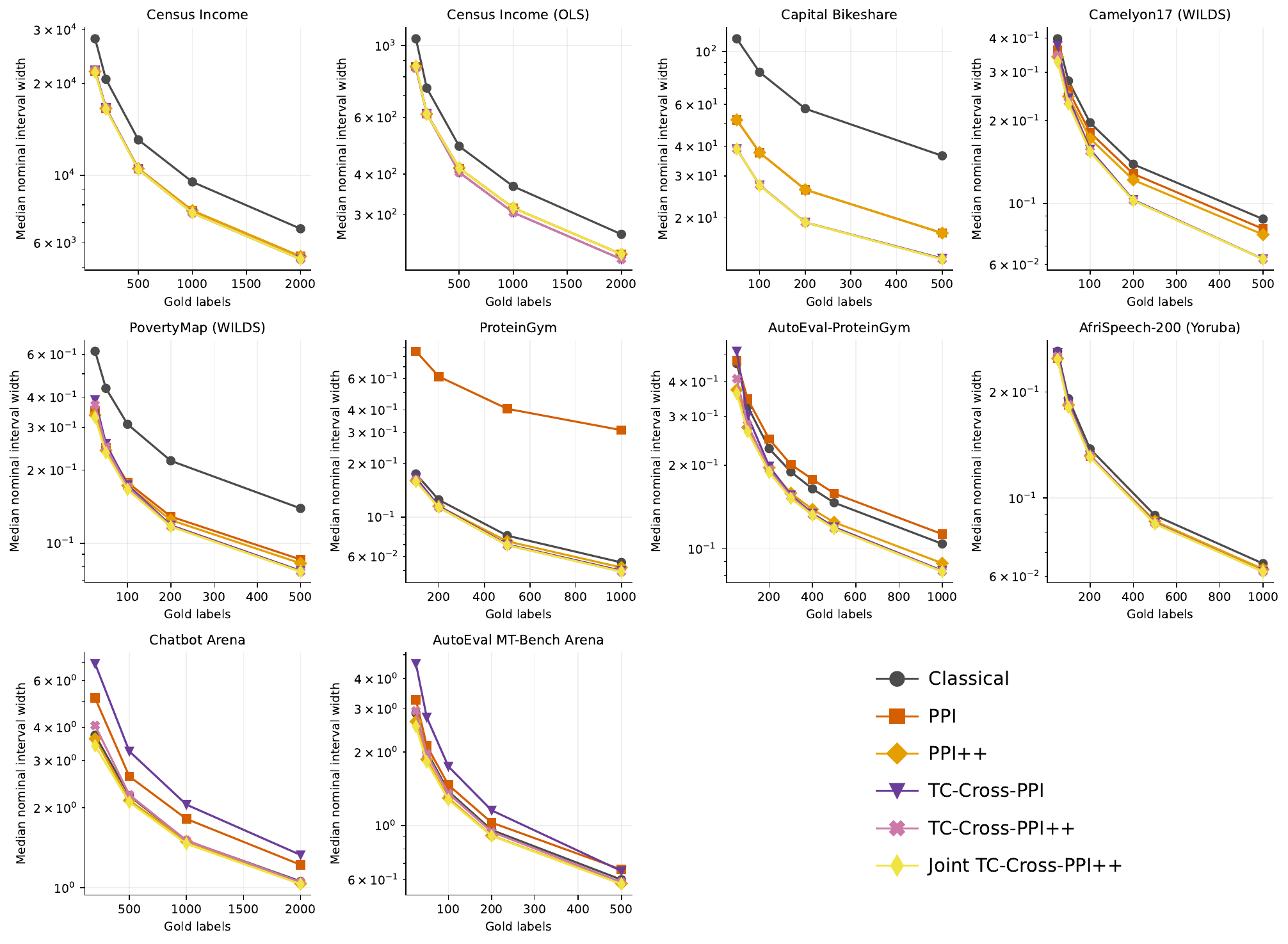}
      \caption{Median (over trials) width of nominal $95\%$ confidence intervals ($90\%$ for the two AutoEval replications), on a logarithmic $y$-axis. Widths are in each estimand's natural units (e.g., dollars for Census Income, hourly rentals for Capital Bikeshare), so they should be compared across methods within a panel rather than across panels.}
  \label{fig:empirical-width-full}
\end{figure}

\begin{figure}[H]
  \centering
  \includegraphics[width=0.9\textwidth]{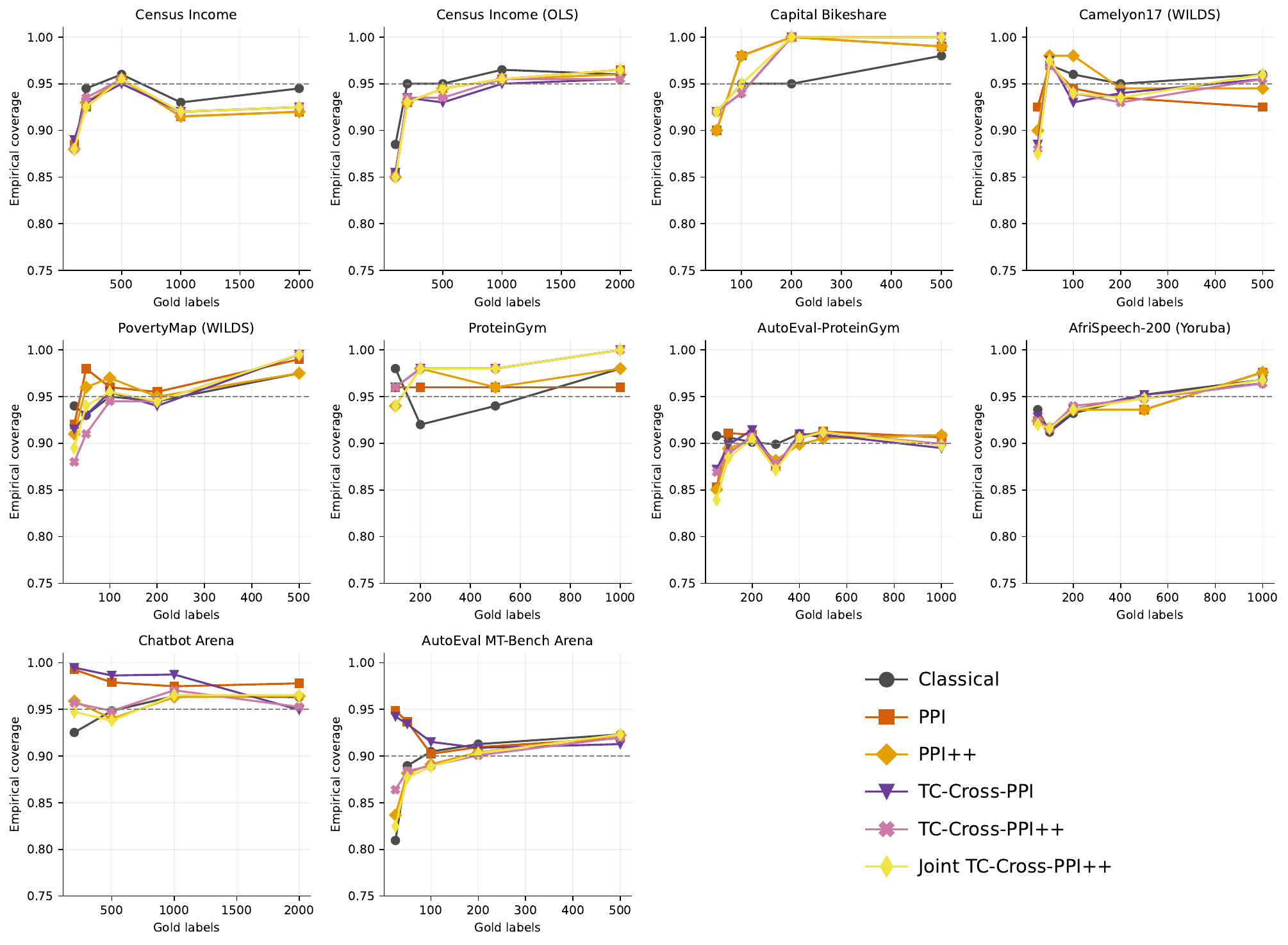}
  \caption{Empirical coverage of nominal confidence intervals. The dashed line marks the nominal level: $95\%$ for most panels, and $90\%$ for the two AutoEval replications (AutoEval-ProteinGym and AutoEval MT-Bench Arena), following AutoEval's convention (Appendix~\ref{app:experimental_setup}). Coverage substantially below the line indicates anti-conservative inference; coverage near one can instead reflect conservative or separation-driven intervals and should be interpreted jointly with interval width.}
  \label{fig:empirical-coverage}
\end{figure}

\subsection{Results with Additional Datasets}
\vspace{-1mm}
\paragraph{Census OLS.}
On the Census dataset, we additionally estimate the OLS coefficient using the same covariates (Appendix~\ref{app:experimental_setup}; Figures~\ref{fig:empirical-ess-full} and~\ref{fig:empirical-width-full}). PPI++ has relative MSE of $0.717$ at $n=200$ and $0.747$ at $n=500$, and TC--Cross-PPI++ improves on it at both, with $0.682$ and $0.732$. Joint TC--Cross-PPI++ is more conservative here, at $0.713$ and $0.735$, tracking PPI++ more closely than the stronger cross-fitted correction. Its power weight can only shrink toward whichever direction it trusts least.
% \begin{figure}[H]
%   \centering
%   \includegraphics[width=0.7\textwidth]{generated/census_income_ols_ess_std.pdf}
%   \caption{Census OLS: empirical effective sample size relative to the number of gold labels, on a logarithmic $y$-axis. The dotted line is classical inference, for which the ratio is one; values above one indicate MSE improvement over it.}
%   \label{fig:census-ols}
% \end{figure}
\vspace{-2mm}
\paragraph{ProteinGym.}
We analyze the \texttt{BLAT\_ECOLX\_Firnberg\_2014} assay with ESM-2 8M (\texttt{esm2\_t6\_8M\_UR50D}) as the source predictor (Figures~\ref{fig:empirical-ess-full} and~\ref{fig:empirical-width-full}). The raw source score is poorly aligned with experimental fitness, and vanilla PPI pins $\lambda$ to one, so it is highly inefficient here, with relative MSE above $20\times$ classical at every $n$. Power tuning fixes this. PPI++ reaches relative MSE of $0.79$--$0.82$ across $n=200$--$500$, and the cross-fitted methods improve further still, to $0.65$--$0.77$, or ESS gains of up to ${\sim}53\%$. For model ranking we target $\mu_m=\mathbb E[Ys_m(X)]$, where $s_m$ is model $m$'s standardized zero-shot score. With three candidate ESM models, TC--Cross-PPI++ reaches mean rank correlation $1.000$ at $n=100$, against $0.917$ for PPI++ and $0.250$ for vanilla PPI.

\subsection{Results looking at Annotator Quality}
\vspace{-1mm}
\paragraph{Protein annotator quality.}
To study sensitivity to synthetic-label quality, we repeat the ProteinGym analysis with ESM-2 8M, 35M, and 150M as the annotator, at a fixed $n=500$. Their correlations with experimental fitness are $0.396$, $0.522$, and $0.626$. As annotator quality increases, PPI++ MSE decreases from $2.89\times10^{-4}$ to $2.74\times10^{-4}$ to $2.18\times10^{-4}$. TC--Cross-PPI++ reduces MSE further, to $2.33\times10^{-4}$, $2.18\times10^{-4}$, and $1.76\times10^{-4}$, and Joint TC--Cross-PPI++ tracks it closely (Figure~\ref{fig:annotator-quality}, left). Transfer calibration again has its largest proportional benefit for the weakest annotator. Coverage for TC--Cross-PPI++ is $0.976$, $0.972$, and $0.980$ across the three annotators.
\vspace{-2mm}
\paragraph{AutoEval-ProteinGym annotator quality.}
We repeat the SPG1 replication using all eight candidate/annotator scores (UniRep, CARP, ESM-1b, ESM-1v, ESM-2, RITA, ProGen2, and VESPA) as the source, at a fixed $n=500$. Annotator-fitness correlation ranges from $-0.05$ for UniRep to $0.49$ for VESPA, and both MSE and ESS gain improve monotonically with it (Figure~\ref{fig:annotator-quality}, right). PPI++'s ESS gain rises from $+0.3\%$ to $+36.8\%$ over that range, and TC--Cross-PPI++'s rises from $+36.8\%$ to $+70.2\%$, staying ahead of PPI++ at every annotator. As with Census OLS, Joint TC--Cross-PPI++ again tracks closer to PPI++ than to the stronger TC--Cross-PPI++ correction. Coverage stays within $0.88$--$0.92$ across every method and annotator, against a nominal $90\%$.

\begin{figure}[H]
  \centering
  \includegraphics[width=0.9\textwidth]{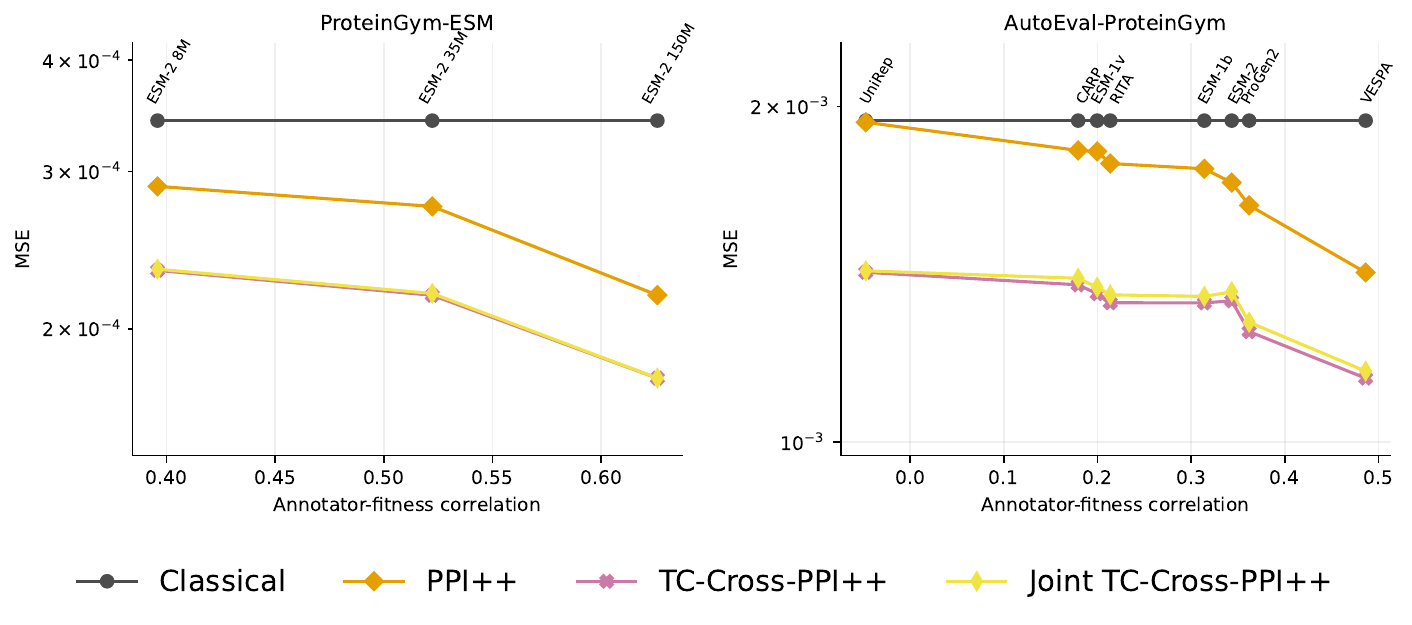}
  \caption{Sensitivity to annotator quality at $n=500$ (MSE on a logarithmic $y$-axis in both panels). Left: ProteinGym-ESM MSE by method versus annotator--fitness correlation (3 annotators). Right: AutoEval-ProteinGym MSE by method versus annotator--fitness correlation (8 annotators).}
  \label{fig:annotator-quality}
\end{figure}

\subsection{Variance Results}
For mean estimation, adding a constant to the predictor doesn't change PPI's variance, since the rectifier already corrects for such a shift. What transfer calibration adds is a lower residual variance, $\operatorname{Var}(Y-f)$, which Table~\ref{tab:residual-variance} compares for the source and calibrated predictors. Calibration helps most when the source's errors depend on the covariates. On Capital Bikeshare, the 2011-to-2012 shift varies with the calendar and weather covariates, and calibration roughly halves the residual variance. The reduction also grows with $n$, since more gold labels give a better-fit correction. When little of the source's error is predictable from the calibration features, as on Census Income, calibration changes little and the cross-fitted methods perform about as well as PPI++. With very few labels, the correction can instead increase the residual variance (e.g., PovertyMap at $n=25$). Joint power tuning matters most here, since it can down-weight the correction while keeping the source predictor, and it essentially matches PPI++ where TC--Cross-PPI falls behind (Figure~\ref{fig:empirical-ess}).

\begin{table}[H]
  \centering
  \caption{Residual variance of the transfer-calibrated predictor relative to the source predictor, $\widehat{\operatorname{Var}}(Y-\hat f_t)/\widehat{\operatorname{Var}}(Y-f_s)$, on the gold sample and averaged over trials, at the smallest and largest label budgets. Values below one indicate calibration reduced the residual variance.}
  \label{tab:residual-variance}
  \begin{tabular}{lrr}
\toprule
Application & Smallest $n$ & Largest $n$ \\
\midrule
Census Income & 1.03 ($n=100$) & 0.96 ($n=2000$) \\
Capital Bikeshare & 0.61 ($n=50$) & 0.41 ($n=500$) \\
Camelyon17 (WILDS) & 1.07 ($n=25$) & 0.60 ($n=500$) \\
PovertyMap (WILDS) & 1.97 ($n=25$) & 0.79 ($n=500$) \\
ProteinGym & 0.04 ($n=100$) & 0.03 ($n=1000$) \\
AfriSpeech-200 (Yoruba) & 1.07 ($n=50$) & 0.97 ($n=1000$) \\
\bottomrule
\end{tabular}

\end{table}

\subsection{Results with LoRA}
As an alternative to sparse-linear calibration, we also fine-tune a LoRA adapter as the correction model. On BLAT this adapts the ESM-2 8M backbone that produces the source score; on AutoEval-ProteinGym the source (VESPA) is not a fine-tunable network, so the adapter is trained on a separate ESM-2 8M encoder with the VESPA score as an input; on AutoEval MT-Bench Arena, distilgpt2 replaces the GPT-4 judge as both the source and the LoRA backbone (Appendix~\ref{app:experimental_setup}). The results are mixed and assay-dependent (Figure~\ref{fig:proteingym-lora}, Table~\ref{tab:lora-results}). On BLAT, sparse TC--Cross-PPI++ is slightly preferable, with relative MSE of $0.854$ against $0.904$ at $n=1000$. On SPG1, LoRA calibration is noisier at small $n$ but eventually overtakes sparse calibration, giving the strongest result in this comparison at $n=500$ ($0.223$ against $0.423$ relative MSE, an ESS gain of $348\%$). On AutoEval MT-Bench Arena, distilgpt2 supplies both the source score and the LoRA backbone, and neither calibration style meaningfully improves on classical inference across the sweep from $n=100$ to $750$. Sparse TC--Cross-PPI++ stays within $1$--$4\%$ of classical throughout, while its LoRA counterpart is noisier at small $n$ ($1.300$ at $n=100$) before converging to the same near-parity result ($1.010$ at $n=750$). We read this as evidence that LoRA calibration can need more gold labels to stabilize than sparse-linear calibration, not that either approach dominates.

\begin{table}[H]
  \centering
  \small
  \caption{Sparse-linear versus LoRA calibration at one representative $n$ per application (Figure~\ref{fig:proteingym-lora} shows the full ProteinGym/AutoEval-ProteinGym sweep over $n$; the AutoEval MT-Bench Arena sweep over $n=100$--$750$ is summarized in the text). Each cell reports (MSE ratio, CI width ratio), both normalized by classical inference. Bold marks the better of the two calibrated methods, independently for each of the two ratios.}
  \label{tab:lora-results}
  \begin{tabular}{lllr}
\toprule
Application & $n$ & Method & (MSE, CI width) ratios \\
\midrule
ProteinGym (BLAT) & 1000 & Classical & (1.000, 1.000) \\
 &  & TC-Cross-PPI++ & (\textbf{0.854}, \textbf{0.886}) \\
 &  & TC-Cross-PPI++-LoRA & (0.904, 0.937) \\
AutoEval-ProteinGym & 500 & Classical & (1.000, 1.000) \\
 &  & TC-Cross-PPI++ & (0.423, 0.822) \\
 &  & TC-Cross-PPI++-LoRA & (\textbf{0.223}, \textbf{0.746}) \\
AutoEval MT-Bench Arena & 750 & Classical & (1.000, 1.000) \\
 &  & TC-Cross-PPI++ & (1.010, 0.998) \\
 &  & TC-Cross-PPI++-LoRA & (\textbf{1.010}, \textbf{0.994}) \\
\bottomrule
\end{tabular}

\end{table}

\begin{figure}[H]
  \centering
  \includegraphics[width=\textwidth]{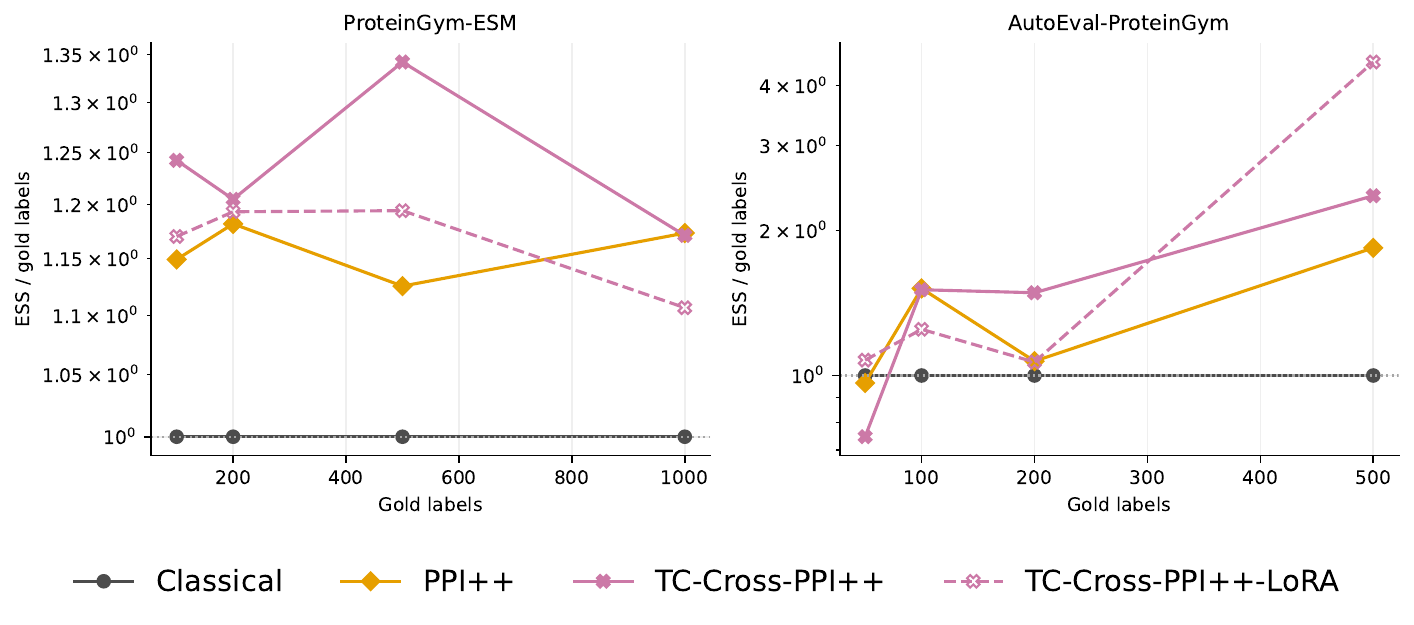}
  \caption{Sparse-linear versus LoRA calibration for ProteinGym-ESM (left) and AutoEval-ProteinGym (right), both using ESM-2 8M as the LoRA backbone. ESS/gold-labels is shown on a logarithmic $y$-axis; the dotted line at one is classical inference.}
  \label{fig:proteingym-lora}
\end{figure}

\subsection{Results with Other PPI Variants}\label{app:other_ppi_results}

\paragraph{Comparison with Cross-PPI.} Cross-prediction-powered inference (Cross-PPI)~\citep{zrnic2024cross} also cross-fits a predictor on the gold labels, but it trains that predictor from scratch rather than adapting a pretrained source. This comparison tells us how much of our gain comes from the source predictor and how much from cross-fitting alone. For Census Income and Capital Bikeshare, whose sources are tree ensembles, we train Cross-PPI with the same model class on the target covariates (Figure~\ref{fig:cross-ppi}). On Capital Bikeshare, Cross-PPI reaches a relative MSE of only $0.75$--$0.91$, compared with $0.16$--$0.21$ for PPI and $0.07$--$0.14$ for TC--Cross-PPI, with coverage between $0.91$ and $0.98$. On Census Income it does worse than classical inference up to $n=500$ (relative MSE $1.85$ at $n=100$) and reaches only $0.87$ at $n=2000$, while PPI, PPI++ and the transfer-calibrated methods stay between $0.57$ and $0.67$. A model fit to up to a few thousand gold labels from the target domain could not compete with the source model that was trained on far more data. Our approach uses both the well-trained source predictor and the scarce gold labels on learning a correction.

\begin{figure}[H]
  \centering
  \includegraphics[width=\textwidth]{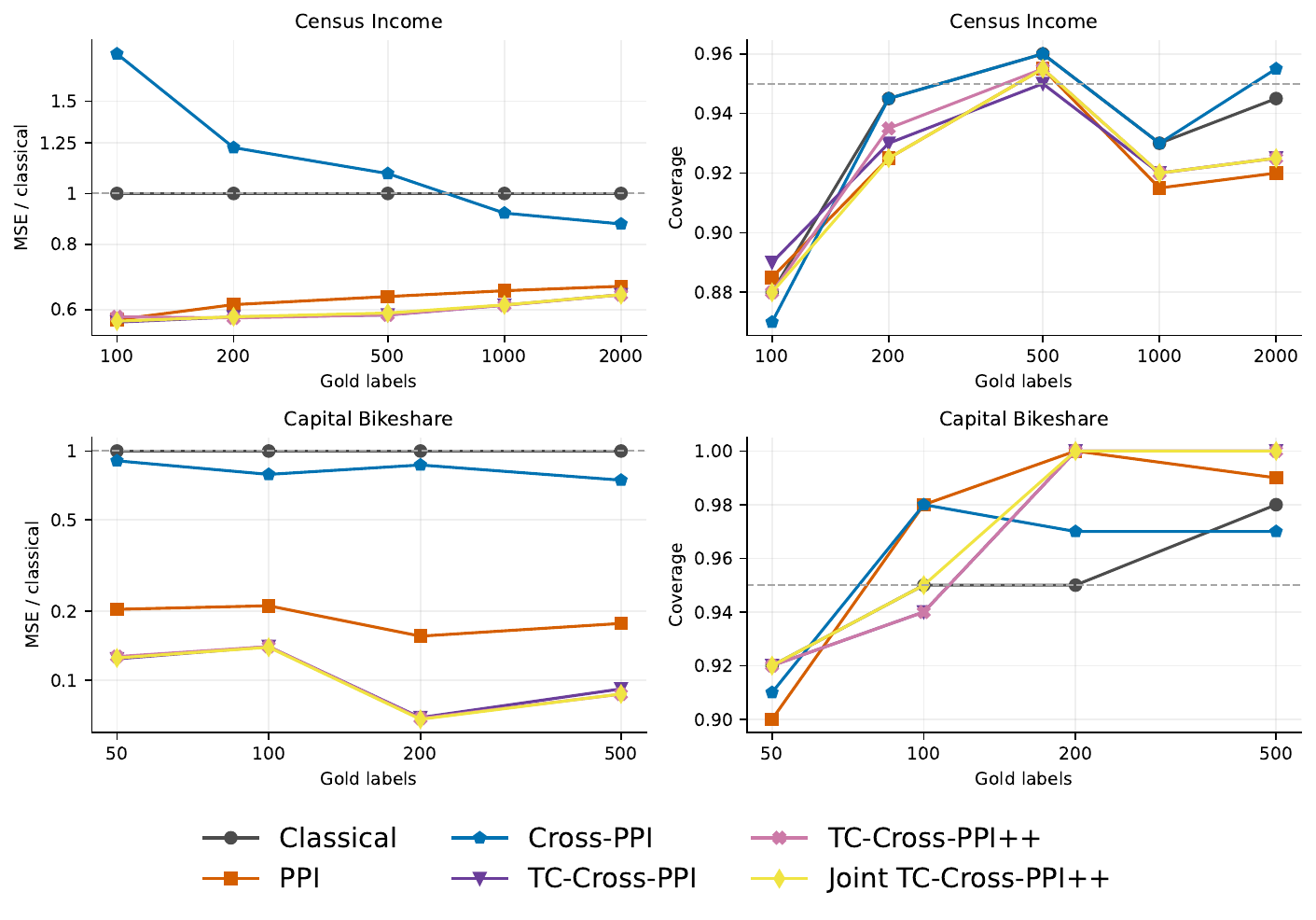}
  \caption{Cross-PPI, trained from scratch with the source's model class, compared with PPI and the transfer-calibrated methods on Census Income (top, 200 repetitions) and Capital Bikeshare (bottom, 100 repetitions). Left: MSE relative to classical inference on a logarithmic $y$-axis; the dashed line at one is classical inference. Right: coverage, with the dashed line at the nominal $95\%$. Both panels use a logarithmic $x$-axis.}
  \label{fig:cross-ppi}
\end{figure}

\paragraph{Comparison with Stratified PPI.}
Stratified prediction-powered inference (StratPPI)~\citep{fisch2024stratified} partitions the input space into strata $\mathcal{A}_k$ with known weights $w_k$ and draws $n_k$ gold labels from each stratum. It then minimizes the stratified rectified loss $$\sum_k w_k\Big[\tfrac{1}{n_k}\sum_{i\in\mathcal{L}_k}\ell_\theta(X_i,Y_i)+\lambda_k\big(\tfrac{1}{N_k}\sum_{j\in\mathcal{U}_k}\ell_\theta(X_j,f(X_j))-\tfrac{1}{n_k}\sum_{i\in\mathcal{L}_k}\ell_\theta(X_i,f(X_i))\big)\Big],$$
with one power-tuning weight $\lambda_k\in[0,1]$ per stratum. Each $\lambda_k$ is tuned by the PPI++ criterion restricted to its stratum, so the synthetic labels are trusted only where they track the gold labels. The strata are formed from the autorater's own output. This makes StratPPI a useful baseline for us, since it adapts to judge heterogeneity without learning any correction, while transfer calibration corrects the judge as a function of the covariates.

We evaluate StratPPI on the Arena datasets with proportional allocation, $n_k\propto w_k$, using one stratum per value of GPT-4's five-valued preference on AutoEval MT-Bench Arena and $10$ equal-mass bins of the judge probability on Chatbot Arena. Figure~\ref{fig:stratppi} reports the median MSE across trials, relative to classical inference, since at small $n$ the Bradley--Terry fit occasionally diverges for every method except classical. StratPPI's coverage is $0.83$--$0.93$ on AutoEval, against a nominal $0.90$, and $0.92$--$0.97$ on Chatbot Arena, against a nominal $0.95$. Its median-MSE ratio is $0.82$--$1.01$ on AutoEval, against $0.70$--$0.96$ for PPI++, and $0.89$--$1.30$ on Chatbot Arena at $n\geq500$, against $0.72$--$0.99$. The per-stratum $\lambda_k$ average $0.25$--$0.39$, close to PPI++'s single weight.

\begin{figure}[ht!]
  \centering
  \includegraphics[width=0.95\textwidth]{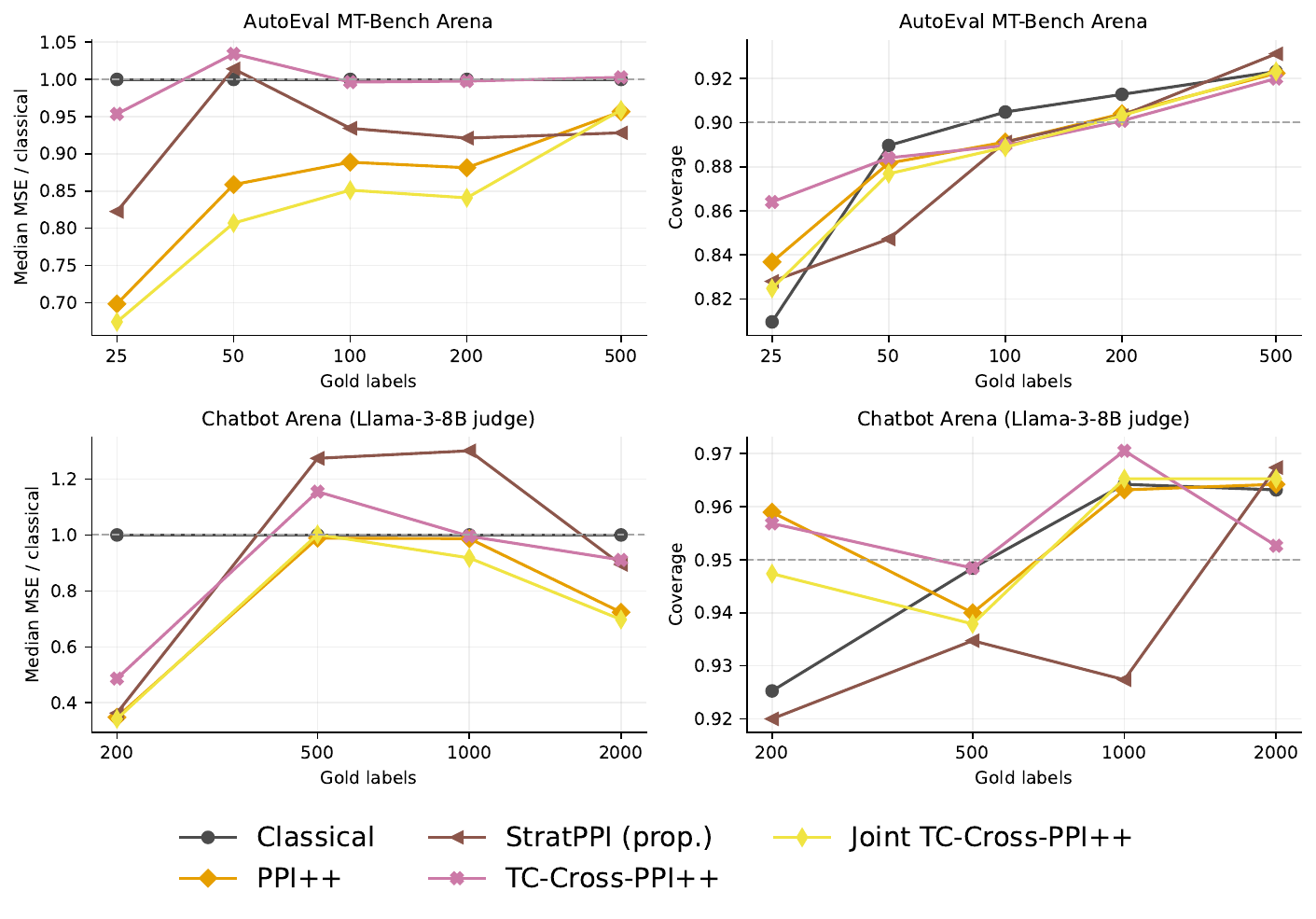}
  \caption{StratPPI (proportional allocation, strata from the judge's own output) compared with PPI++ and the transfer-calibrated methods on AutoEval MT-Bench Arena (top, 250 repetitions) and Chatbot Arena with the Llama-3-8B judge (bottom, 50 repetitions). Left: median BT-coefficient MSE relative to the classical median; the dashed line at one is classical inference. Right: average coordinate-wise coverage, with the dashed line at nominal coverage ($90\%$ for AutoEval, following its convention, and $95\%$ for Chatbot Arena). Both panels use a logarithmic $x$-axis.}
  \label{fig:stratppi}
\end{figure}

\end{document}